%% file: main.tex
\documentclass{article}

\usepackage[preprint]{neurips_2025}

\usepackage{graphicx}
\usepackage{booktabs} 

\usepackage{caption}
\usepackage{subcaption}

\usepackage{hyperref}
\usepackage[utf8]{inputenc} 
\usepackage[T1]{fontenc}    
\usepackage{url}   
\usepackage{tabularx} 
\usepackage{amsfonts}       
\usepackage{nicefrac}       
\usepackage[table,xcdraw]{xcolor}
\usepackage{microtype}      
\usepackage{xcolor}         

\usepackage{amsmath}
\usepackage{amssymb}
\usepackage{mathtools}
\usepackage{amsthm}
\usepackage{dsfont}
\usepackage{float}
\usepackage{amsthm}
\usepackage{bbm}
\input{math_commands.tex}

\usepackage[capitalize,noabbrev]{cleveref}

\usepackage{algorithm}
\usepackage{algorithmic}
\usepackage{enumitem}

\usepackage{multirow}
\usepackage{siunitx}
\usepackage[normalem]{ulem}
\useunder{\uline}{\ul}{}

\usepackage{wrapfig}
\usepackage{lipsum}  
\theoremstyle{plain}
\newtheorem{theorem}{Theorem}[section]

\theoremstyle{definition}
\newtheorem{definition}[theorem]{Definition}

\theoremstyle{remark}
\newtheorem{remark}[theorem]{Remark}

\title{LOBSTUR: A Local Bootstrap Framework for \\Tuning Unsupervised Representations \\in Graph Neural Networks}

\author{%
  So Won Jeong  \\
  Booth School of Business\\
  The University of Chicago\\
  Chicago, IL 60637 \\
  \texttt{sowonjeong@chicagobooth.edu} \\
  \And
  Claire Donnat \\
  Department of Statistics \\
  The University of Chicago\\
  Chicago, IL 60637 \\
  \texttt{cdonnat@uchicago.edu} \\
}

\begin{document}

\maketitle

\begin{abstract}

Graph Neural Networks (GNNs) are  increasingly used in conjunction with unsupervised learning techniques to learn powerful node representations, but their deployment is hindered by their high sensitivity to hyperparameter tuning and the absence of established methodologies for selecting the optimal models.
To address these challenges, we propose LOBSTUR-GNN ({\bf Lo}cal {\bf B}oot{\bf s}trap for {\bf T}uning {\bf U}nsupervised {\bf R}epresentations in GNNs) i), a novel framework designed to adapt bootstrapping techniques for unsupervised graph representation learning. LOBSTUR-GNN  tackles two main challenges: (a) adapting the bootstrap edge and feature resampling process to account for local graph dependencies in creating alternative versions of the same graph, and (b) establishing robust metrics for evaluating learned representations without ground-truth labels. Using locally bootstrapped resampling and leveraging Canonical Correlation Analysis (CCA) to assess embedding consistency, LOBSTUR  provides a principled approach for hyperparameter tuning in unsupervised GNNs. 
We validate the effectiveness and efficiency of our proposed method through extensive experiments on established academic datasets, showing an 65.9\% improvement in the classification accuracy compared to an uninformed selection of hyperparameters. Finally, we deploy our framework on a real-world application, thereby demonstrating its validity and practical utility in various settings. 
\footnote{The code is available at \href{https://github.com/sowonjeong/lobstur-graph-bootstrap}{github.com/sowonjeong/lobstur-graph-bootstrap}.}

\end{abstract}

\section{Introduction}
With the expanding availability of network and spatial data in the sciences, Graph Neural Networks (GNNs) have emerged as a compelling approach to identify interaction patterns within complex systems. 
Examples include spatial transcriptomics \citep{zhu2018identification, dong2022deciphering}, where graph-based neural networks are used to learn representations of cell neighborhoods that can be correlated with cancer outcomes, and microbiome studies \citep{lamurias2022graph, le2020deep} where they are used for genome assembly or to predict metabolite information from microbes.
In many of these applications, scientists are increasingly interested in using GNNs in conjunction with unsupervised learning techniques for learning informative representations, due to the paucity of available labeled data, or as a way of automatically detecting structure or patterns \citep{zhu2018identification, dong2022deciphering, lamurias2022graph, le2020deep}. Consequently, the last few years have seen the development of a number of unsupervised GNN methods for the sciences \citep{hu2021spagcn,partel2021spage2vec,zhang2020dango,li2021scgslc,ishiai2024graph}, usually adapting known methods for Euclidean data such as tSNE \citep{tsne}, k-Means clustering, or UMAP \citep{mcinnes2018umap} to accommodate graph data and GNNs. These methods typically involve few hyperparameters, but their scope of application is typically confined to a specific data type or use case.

Within the methods community, on the other hand, recent advances in unsupervised node representation learning  seem to have primarily been driven by contrastive learning  \citep{stokes2020deep, zhang2021canonical,you2021graphcontrastivelearningaugmentations}. This popular self-supervised learning framework has indeed demonstrated impressive performance for learning rich and versatile data representations across various domains. 
However, in the graph-setting, despite their impressive performance and promising results on academic benchmarks, these methods are not tuning-free, making them difficult to deploy in real-world applications. In fact, they rely heavily on selecting appropriate values for several of their hyperparameters, but
incorrect hyperparameter values can lead to severely distorted data representations. We illustrate this effect in Figure~\ref{fig:perfromance-varies-by-hyperparam} by showing how  different choices of hyperparameters can decrease the accuracy of a linear classifier (trained on the learned node embeddings of a toy benchmark) from 73\% to 30\%, thus indicating a severe loss in embedding quality. 

Despite the empirical importance of hyperparameter tuning, there is currently no valid hyperparameter selection procedure for unsupervised GNN representation learning. 
In the methods community, new unsupervised learning approaches are commonly tested on established benchmark datasets, with hyperparameters selected based on performance in a downstream node classification task. However, this procedure essentially converts the problem into a supervised learning setting, making it unsuitable for genuinely unsupervised, real-world use cases.

Hyperparameter tuning in unsupervised settings is made difficult by two main challenges: (a) the absence of a clear ground truth or statistical framework for unsupervised learning,
and (b) the lack of an established metric to evaluate the learned embeddings. To our knowledge, the only study that attempts to measure the quality of latent representations is that of \citeauthor{tsitsulin2023unsupervised}, which empirically evaluates various metrics. Yet, without a proper inference framework, pinpointing a suitable metric remains a significant challenge. 

\noindent\textbf{Contributions.} In this paper, we propose the first bootstrapped-based method for selecting hyperparameters for unsupervised GNN representation learning. Hyperparameters can be broadly categorized into two types: (1) those that require tuning within a specific model (e.g. parameters in the loss function) and (2) those that involve tuning across a family of models (e.g. different GNN architectures or sizes). In this paper, we aim to tackle both, as both are essential to self-supervised learning.  
More specifically, 
\begin{enumerate}[noitemsep, topsep=0em]
\item We cast the learning of representations as an estimation problem: we posit that the learned representations correspond to a learned low-dimensional manifold, which must therefore be consistent under a noise model, as explicited in Section~\ref{sec:prob-setup}.
    \item To generate independent copies of the same graph, we propose a bootstrap procedure based on nonparametric modeling of the graph as a graphon \cite{su2020network} (Section~\ref{section:cv-split}).
    \item To evaluate the quality of the embeddings learned on independent copies of the same graph in the absence of labels, we suggest using Canonical Correlation Analysis \citep{hotelling1936relations} as a translation- and rotation-invariant tool to quantify the stability of the learned embedding spaces (Section~\ref{sec:cv-eval}). 
\end{enumerate}

\begin{figure}
    \centering
    \includegraphics[width = \linewidth]{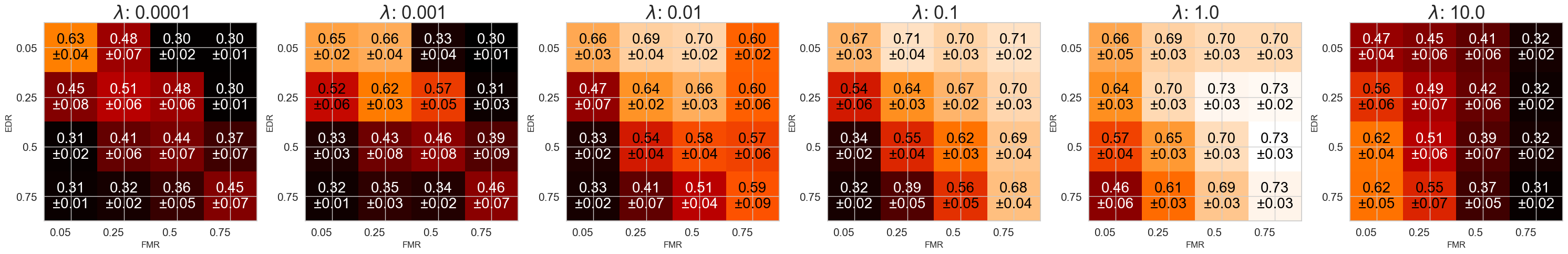}
    \caption{Cora. Evaluation of the CCA-SSG embeddings \cite{zhang2021canonical}, an unsupervised learning method, for each combination of the hyperparameters (loss parameter $\lambda$, edge drop rate (EDR), feature mask rate (FMR)). Each entry denotes the mean and standard deviation of the node classification accuracy of a linear classifier trained on the learned representations (averaged over 20 experiments).}
    \label{fig:perfromance-varies-by-hyperparam}
\end{figure}

\section{Problem Formalization}\label{sec:prob-setup}

Establishing a framework for hyperparameter tuning in unsupervised learning requires us to address two fundamental questions: what are we aiming to estimate, and where does the randomness come from?

In the graph setting, the data is presented in two modalities: a feature matrix, and an adjacency matrix. Unsupervised learning can be framed as learning what information is shared across modalities, and what information is specific to each one in a condensed format. This approach is typically described in the data-integration literature using a latent variable space model \cite{bishop1998latent,hoff2002latent}, which we adapt here for the graph domain.

 \paragraph{Inference setting.} We consider a graph $G$ on $n$ nodes with features $X\in \R^{n \times p},$ and denote by $A \in \{0,1\}^{n \times n}$ its corresponding (binary) adjacency matrix. We assume the graph is sampled from a graphon $W$  (see for instance \citeauthor{gao2015rate}) ---  a non-parametric random graph model--, and that node features are a noisy transformation of the latent variable  $U_i$: 
 
\begin{equation}\label{eq:graphex}
    \begin{split}
    \forall i \in [n], &\ U_i \sim \text{Unif}([0,1]), \\
    \forall j \in [n], &\ A_{ij} \sim \text{Bernoulli}(W(U_i, U_j)) , \\
    &\ X_i \sim g(U_i) + \epsilon_i,
    \end{split}
\end{equation}

where $W(U_i, U_j)$ denotes the graphon function evaluated at the latent positions $U_i$ and $U_j$, and $\epsilon_i$ denotes some independent, mean-zero noise.
This model allows us to reason on the randomness of the generation procedure without making assumptions on the specifics of the graph generation process.
While graphons are known to generate dense graphs,  their output can be sparsified by scaling $W$ by a sparsity factor that tends to 0 as $n\to \infty$, e.g. $\rho_n = \frac{\log(n)}{n}$ \citep{davison2023asymptotics, gaucher2021maximum}.

\paragraph{Learning Conditional or Marginal Representations.} 
In the unsupervised context, we differentiate between two main scenarios based on whether the goal is to learn representations conditioned on the realized $U_i$ or not. \\
\textit{(a) Learning \textbf{marginal data representations} that do not depend on the realized $U_i$s} : In some applications, embeddings are assumed to capture clusters or predict specific outcomes \citep{hu2021spagcn, wu2022space}. In these cases, the learned embeddings serve as a way to extract inherent structure from the data --- such as clusters, an underlying manifold, or more generally, functions of the graphon $W$ ---   that should be consistent across datasets of the same type. For example, in single-cell transcriptomics studies (e.g. \citet{wu2022space} and \citet{hu2021spagcn}), GNNs are employed to learn embeddings that capture patterns in cell-neighborhood interactions. Here, the assumption is that the clusters derived from one dataset should be reproducible on another dataset drawn from the same distribution. An oracle data generation procedure would therefore generate unseen data using the full model in (\ref{eq:graphex}), resampling latent variables $U_i$ to generate new nodes and edges.

\textit{(b) Learning \textbf{conditional data representations}, i.e. conditional on the realized $U_i$s}:  However, in other applications, these assumptions do not hold: there might not necessarily be an obvious notion of ``another similar graph''. In citation networks, for example, where nodes represent users, the focus shifts to learning properties specific to individual nodes. In this case, data generation procedures must accommodate another type of randomness, this time conditioned on $U_i$: randomness arises solely in sampling the edges and features, as per the last two lines of (\ref{eq:graphex}).

In both settings, however, the quality of the learned embeddings might be evaluated based on their reproducibility, or the alignment between the latent structure stemming from representations learned on one dataset to those learned on another. Devising a criterion leveraging this notion would require two main components: (a) a data generation procedure, to create independent draws of the same datasets (Section~\ref{section:cv-split}), and (b) a metric to measure the alignment between representations (Section~\ref{sec:cv-eval}).

\begin{figure*}
    \centering
    \includegraphics[width=\textwidth]{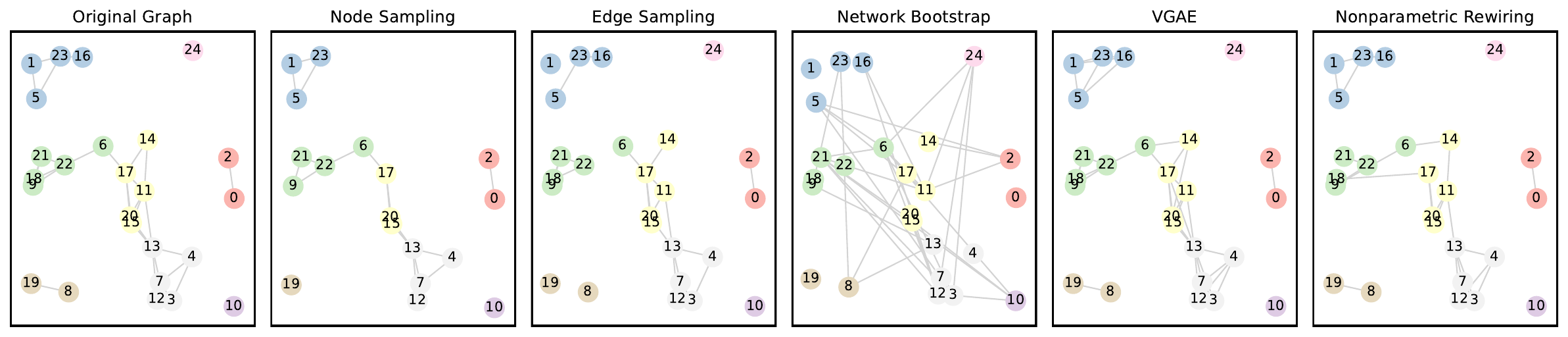}
    \caption{Illustration of different techniques for generating new copies of a simple graph (left-most image). The original graph has a distinctive community structure. Note that node sampling or edge sampling randomly removes either nodes or edges, disrupting the original graph structure.}    
    \label{fig:split-illustration}
\end{figure*}

\section{A Local Graph Bootstrapping Procedure}\label{section:cv-split}


In the GNN literature, data splitting and resampling are usually done in one of two ways: by resampling the nodes or by resampling the edges. 
However, in the unsupervised setting, these two sampling procedures are not necessarily suitable for two reasons. First, this type of sampling can considerably disrupt the structure of the graph (by thinning nodes or edges, respectively), as reflected in Figure~\ref{fig:split-illustration}, Table~\ref{tab:graphon-compare-all} and Figure~\ref{fig:graphon_comparison} in the Appendix --- an indication that the stochastic process underpinning the graph creation is not correctly simulated. Second, these procedures require the specification of the node (respectively edge) drop rate. However, without any theoretical underpinning, it is difficult to set these correctly, as the distance between the original graph and the thinned ones increases monotonously with the drop rate.

\paragraph{Relationship with existing literature.} Relatively few studies have been conducted on extending bootstrapping to network data \citep{li2020network, hoff2007modeling, Chen2018Network, levin2021bootstrapping}. Most of these works assume a specific distribution for the graph, also typically assuming latent variables for each node (e.g. a random product graph \citep{levin2021bootstrapping}). The edges are assumed to be drawn from a Bernoulli distribution with probability $f(U_i, U_j)$, where $f$ is a parametric function. For instance, in the random dot product graph \citep{levin2021bootstrapping}), edges are assumed to be sampled as $A_{ij} \sim \text{Bernoulli}( U_i^\top U_j)$. 
However, these works (a) rely on strong parametric assumptions about the graph’s nature, which, if unverified, can lead to resampled graphs that differ substantially from the original; (b) usually do not address the challenge of simultaneously resampling both edges and node features; and (c) typically generate new graphs conditioned on the imputed latent variables, but are not easily extendable to the marginal case, which would require resampling the node representations $U_i$'s themselves.

By contrast, we propose a nonparametric technique for resampling graphs based on the model detailed in (\ref{eq:graphex}), thus requiring minimal assumptions about the underlying graph structure. 

\subsection{The Oracle Case}
We begin by assuming that, for each latent variable $U_i$, we have oracle knowledge of its 
$k$-nearest neighbors. We denote the resulting directed $k$-nearest neighbor graph as $\mathcal{G}_{knn}$. Under sufficiently smooth functions \(W\) and \(g\) (as defined in the next paragraphs), for a given node $i$, its neighbors in \(\mathcal{G}_{knn}\) have similar distributions, and can thus be viewed as alternative realizations of the same underlying stochastic process (conditioned on $U$).

We leverage this observation to  propose a bootstrap procedure conditioned on the realized \(U_i\):
\begin{enumerate}[noitemsep, topsep=0em]
    \item {\it Feature resampling:} we resample the features of each node by drawing at random a feature vector from one of $k$-nearest neighbors. This ensures preserving the covariance between features by sampling full vectors.
    \item  {\it Edge rewiring:}  
 let $\mathcal{N}_{m}(i)$ denote the $m^{th}$ closest neighbor of node $i$ according to the oracle graph $\mathcal{G}_{knn}$. For each pair of node $(i,j)$, sample an edge with probability $ \hat{p}_{ij}   = \frac{1}{k} \sum_{m=1}^k A_{\mathcal{N}_m(i), j},$ effectively estimating the underlying probability $\mathbb{P}[A_{ij} = 1| U_i, U_j]$.
 An efficient procedure for resampling is presented in Algorithm~\ref{alg:graph}.
\end{enumerate}
 To extend this procedure to generate {\it marginally-resampled} graphs, we propose simply sampling with replacement nodes (which effectively implies resampling the $U_i$), and applying the same procedure as above. The whole procedure is described in more detail in Algorithm~\ref{alg:full}.  In either case (marginally or conditionally on $U$), this framework preserves local latent-space similarities while generating plausible bootstrap replicates of the graph. 

\begin{algorithm}[H]
    \caption{Non-parametric Resampling of Node Features}\label{alg:features}
    \begin{algorithmic}[1]
        \STATE \textbf{Input:} Graph $G$ with node features $\{X_i\}_{i=1}^n$, $\mathcal{G}_{knn}$ $k$-nearest neighbor graph on $U$.
        \FOR{each node $i \in [n]$}
            \STATE Identify the set of neighboring nodes $N(i) = \{j : j \sim i\}$ in the graph $\mathcal{G}_{knn}$,
            \STATE Construct the candidate set for resampling: $C_i = \{X_i\} \cup \{X_j\}_{j \in N(i)}$.
            \STATE Resample the feature vector for node $i$ by selecting a vector uniformly from $C_i$: 
            $X_i^{\text{new}} \sim \text{Unif}(C_i)$.
        \ENDFOR
        \STATE \textbf{Output:} Resampled node features $\{X_i^{\text{new}}\}_{i=1}^n$.
\end{algorithmic}
\end{algorithm}

 The following theorem characterizes the consistency of the procedure in deriving nodes with similar features.
\begin{theorem}
    Assume that $\mathcal{G}_{knn}$, the directed $k$- nearest neighbor graph induced by the latent variable $\{U_i\}_{i=1}^n$ is known, with $k$ such that $\lim_{n \to \infty} \frac{k}{n} = 0$. Suppose $g$ is an  $\alpha$-H{\"o}lder-continuous function on the interval $[0,1]$, so that there exists a constant $C$ such that:
     $|g(U_i)- g(U_j)| \leq C | U_i- U_j|^{\alpha}$ for $\alpha>0$.
     
     Let $\mathcal{X}$ denote the domain of the features (so that for each node $i$, its features are denoted $X_i \in \mathcal{X}$). Then, the procedure described in Algorithm~\ref{alg:features} is asymptotically consistent in that for any set $\mathcal{A} \in \mathcal{X}$:
    $$ \forall j \in \mathcal{N}_{knn}(i), \lim_{n \to \infty} \left | \mathbb{P}(X_{j} \in \mathcal{A}|U_j) -\mathbb{P}(X_{i} \in \mathcal{A}|U_i)\right|= 0,$$
where $\mathcal{N}_{knn}(i)$ denotes any of the $k$-nearest neighbors of $i.$
\end{theorem}
\begin{proof}
Under (\ref{eq:graphex}),
$\forall i, \quad X_i = g(U_i) + \epsilon_i,$
where $\epsilon_i$ is independent, identically distributed centered noise, and $g(U_i)$ is the expectation of $X$ given the latent $U_i$. Since the $\epsilon$ are assumed to be i.i.d,  we can write for any nodes \( i \) and \( j \):
$$X_i \overset{d}{=} g(U_i) + \epsilon_j = g(U_i) + X_j - g(U_j).$$
The quantity $g(U_i)- g(U_j)$ represents the bias in using the expectation $X_j$ to approximate the distribution of $X_i$, and since $g$ is assumed to be H{\"o}lder-continuous:  $\|g(U_i)- g(U_j)\| \leq c | U_i- U_j|^{\alpha}$.

Consider now  $j$ to be chosen to be one of the $k$-nearest neighbors of node $i$. $| U_i- U_j|^{\alpha}$ is a monotonously decreasing function of $n$, and
with high probability (over the distribution of $U_1, \cdots , U_n)$, 
we have $|U_i - U_j|  \leq {c_0 }\frac{k}{n}$, for all $j \in \mathcal{N}(U_i)$, and a constant $c_0$.  Therefore, $ \|g(U_i)- g(U_j)\|$ tends to 0 (in probability) as $n$ goes to $\infty$. Therefore, by Slutsky's lemma, as $n$ goes to $\infty$, $\{ X_{j} |U_j\}_{j \in \mathcal{N}_{knn}(i)}\overset{d}{\to} X_i | U_i. $
\end{proof}

\begin{algorithm}
    \caption{Non-parametric Resampling of Edges}\label{alg:graph}
\begin{algorithmic}[1]
    \STATE \textbf{Input:} Graph $G = (\mathcal{V}, \mathcal{E})$ with $n = |\mathcal{V}|$ nodes; flattened list of edge stems
    $$L = \{ u \mid u \in E[:,0] \} \cup \{ v \mid v \in E[:,1] \},$$
    where $E \in \mathbb{R}^{|\mathcal{E}| \times 2}$, and $k$-nearest neighbor graph $\mathcal{G}_{\text{knn}}$ on $U$.
    \STATE Initialize an empty graph $G'$ with $n$ nodes.
    \WHILE{$\texttt{len}(L) > 0$}
        \STATE Sample a source node $u$ uniformly at random from $L$ and remove it: $u \leftarrow \text{pop}(L).$
        \STATE Sample a target node $v$ from
        $$v \sim L \cap \left( \bigcup_{m=1}^k \mathcal{N}_A\left( \mathcal{N}^{\text{knn}}_m(u) \right) \right),$$
        where $\mathcal{N}_A(i)$ denotes the set of neighbors of node $i$ in $G$, and $\mathcal{N}^{\text{knn}}_m(u)$ denotes the $m$-th nearest neighbor of node $u$ in $\mathcal{G}_{\text{knn}}$.
        \STATE Remove the selected node $v$ from $L$.
        \STATE Add an undirected edge between $u$ and $v$ in $G'$.
    \ENDWHILE
    \STATE \textbf{Output:} Resampled edge structure $\{ A_{ij}^{\text{new}} \}_{i,j=1}^n$.
\end{algorithmic}
\end{algorithm}

The following theorem highlights the consistency of the edge rewiring procedure.

\begin{theorem}\label{theorem:edge}
    Suppose that the $k_n$-nearest neighbor graph $\mathcal{G}_{knn}$ induced by the latent variables $\{U_i\}_{i=1}^n$ is known, where $k_n$ is such that $ \lim_{n\to \infty}\frac{k_n}{n} = 0$, and $\lim_{n \to \infty} k_n =0$. Suppose that $W$ is an $\alpha$-H{\"o}lder graphon function \citep{gao2015rate} (see definition~\ref{def:graphon_norm} in the Appendix) with $\alpha \in (0, 1]$.
    
    Then, the quantity $\hat{p}_{ij}  = \frac{1}{k_n} \sum_{m=1}^{k_n} A_{\mathcal{N}_m(i), j}$ is a consistent estimator of $p_{ij}$ in the sense that: 
    $$ \lim_{n\to\infty}  \hat{p}_{ij} = \mathbb{P}[A_{ij} | U_i, U_j].$$
\end{theorem}
\begin{proof}
    The proof follows a similar argument to the previous theorem and is deferred to Appendix~\ref{app:proof}.
\end{proof}


\begin{remark} We note that the noise $\epsilon_i$ on the features does not need to be globally identically distributed for the previous construction to hold. Instead, since the procedure only relies on the $k$-nearest neighborhood of each node, it suffices to assume that these properties hold locally.
\end{remark}

\subsection{The Noisy Setting}\label{sec:split-noisy}

The resampling procedure highlighted in the previous paragraph requires oracle knowledge of the $k$-NN graph on the latent $U$. In practice, the graph $\mathcal{G}_{knn}$ has to be estimated from the data. To this end, we suggest two strategies:

\textit{Solution 1: Constructing 2 kNN graphs based on features and topology, respectively.} The theorems established earlier ensure the consistency of our resampling procedure, provided that the kNN graph is constructed independently of the data. Therefore, one approach is to build a kNN graph based on a notion of graph distance (e.g., shortest-path, shared -neighbors, or Jaccard similarity, depending on what metric is adapted for the graph), which can then be used for resampling the node features, while using a kNN graph based on feature similarity for resampling the edges. This approach is expected to perform well as long as the graph distance chosen is reflective of the underlying distance (e.g. shortest-path distance for homophilic graphs, or shared neighbors for more general classes of graphs), and the features are sufficiently informative.\\

{\it Solution 2: Using solely the kNN Graph from the adjacency matrix.} In practice, while solution 1 may work well for feature-based resampling, the kNN induced by the features (and used to resample edges) might not be as reliable as the one induced by the edges (see Table~\ref{tab:cora-exp-noisty-setting} in the Appendix). This is because, in high-dimensional feature spaces, kNN suffers from the curse of dimensionality, making it difficult to ensure the consistency of the kNN graph. As an alternative, one can define the kNN graph solely based on the graph structure for all components of the algorithm. While this approach does not guarantee theoretical consistency in estimating the relevant quantities, it exhibits promising empirical performance, as shown in the experiments.

\subsection{Validation of Bootstrap Samples}

To evaluate the quality of bootstrapped samples, we propose bootstrapping different graphs (synthetic and real), and to compare key graph statistics---including node and edge counts, average degree, and degree distribution---, against those of the original graph.

Table~\ref{tab:graph-stat-simple} summarizes these comparisons for a synthetic graphon function and three well-established graph benchmarks. Additional results for more datasets and graphon settings, including the effect of the choice of $k$, are provided in Appendix~\ref{app:bootstrap-validation}. While not exhaustive, these comparisons help assess whether the structural properties of the original graph are preserved in the bootstrapped samples. In particular, we note that our approach typically produces graphs with a closer average degree and edge count than other methods (see for instance Table~\ref{tab:graphon-compare-all} and Figure~\ref{fig:graphon_comparison} in the Appendix). When the underlying graph is a graphon, our model is in fact very good at reproducing graphs with the similar statistics (see Table~\ref{tab:synthetic-graph-stat-s1}, \ref{tab:synthetic-graph-stat-s2}, \ref{tab:synthetic-graph-stat-s3}, \ref{tab:synthetic-graph-stat-s4}). On real datasets, our method seems to produce reasonable copies of the same graph as well, as reflected by similar average degrees and number of connected components (Table~\ref{tab:graph-stat-real-data} and \ref{tab:graph-stat-real-data-compare-all}). However, the graphon assumption upon which our method relies seems to hit a limit in the ability of the method to reproduce a graph with as many triangles (see results Cora in Table~\ref{tab:graph-stat-simple}).

\begin{table*}[t]
\centering
\resizebox{\textwidth}{!}{%
\begin{tabular}{@{}lcccccccc@{}}
\toprule
\textbf{Statistic} &
  \multicolumn{2}{c}{\textbf{Graphon (n = 500)}} &
  \multicolumn{2}{c}{\textbf{Cora}} &
  \multicolumn{2}{c}{\textbf{Citeseer}} &
  \multicolumn{2}{c}{\textbf{Twitch}} \\
\cmidrule(lr){2-3} \cmidrule(lr){4-5} \cmidrule(lr){6-7} \cmidrule(lr){8-9}
 &
  \textbf{True} &
  \textbf{Ours} &
  \textbf{True} &
  \textbf{Ours} &
  \textbf{True} &
  \textbf{Ours} &
  \textbf{True} &
  \textbf{Ours} \\ \midrule
$|\mathcal{V}|$         & 500   & 500$\pm$0        & 2708  & 2708$\pm$0        & 3327  & 3327$\pm$0         & 1912   & 1912$\pm$0        \\
$|\mathcal{E}|$         & 769   & 757.9$\pm$2.5    & 5278  & 5171.78$\pm$7.34  & 4552  & 4127.78$\pm$10.93  & 31299  & 31082.05$\pm$22.83 \\
Avg. Degree             & 3.08  & 3.03$\pm$0.01    & 3.90  & 3.82$\pm$0.01     & 2.74  & 2.48$\pm$0.01      & 32.74  & 32.51$\pm$0.02     \\
Density                 & 0.01  & 0.01$\pm$0       & 0.00  & 0.00$\pm$0        & 0.00  & 0.00$\pm$0         & 0.02   & 0.02$\pm$0         \\
Clustering Coefficient  & 0.01  & 0.01$\pm$0       & 0.24  & 0.05$\pm$0        & 0.14  & 0.03$\pm$0         & 0.32   & 0.17$\pm$0         \\
Connected Components    & 29.00 & 31$\pm$2         & 78    & 67.91$\pm$6.70    & 438   & 635.09$\pm$11.87   & 1.00   & 1.14$\pm$0.39      \\
Giant Component Size    & 471.00& 467$\pm$3        & 2485  & 2620.80$\pm$10.80 & 2120  & 2418.12$\pm$35.69  & 1912   & 1911.72$\pm$0.77   \\
Assortativity           & -0.04 & -0.08$\pm$0.03   & -0.07 & -0.07$\pm$0       & 0.05  & -0.08$\pm$0        & -0.23  & -0.29$\pm$0        \\
PageRank Sum            & 249.5 & 249.5$\pm$0      & 1353.50 & 1353$\pm$0      & 1663  & 1663$\pm$0         & 955.50 & 955.5$\pm$0        \\
Transitivity            & 0.01  & 0.01$\pm$0       & 0.09  & 0.03$\pm$0        & 0.13  & 0.04$\pm$0         & 0.13   & 0.08$\pm$0         \\
Number of Triangles     & 7     & 5$\pm$2.3        & 1630  & 471$\pm$27        & 1167  & 304.6$\pm$19.59    & 173510 & 105534.51$\pm$1904.54 \\ \bottomrule
\end{tabular}%
}
\caption{Graph statistics for synthetic graphon data, citation networks (Cora, Citeseer), and a social network (Twitch) \citep{huang2023conformalized_gnn}. We generated 500 bootstrap samples and report the mean and standard deviation. The size of the neighborhood ($k$) used for sample generation is fixed at 20. Results for additional datasets and different graphon settings are included in Appendix~\ref{app:bootstrap-validation}.}
\label{tab:graph-stat-simple}
\end{table*}

\section{Evaluation Metrics} \label{sec:cv-eval}
If the generation of independent copies of the same graph poses a significant challenge, determining an appropriate evaluation metric in the absence of known labels poses another. 
 While contrastive learning is based on the idea of turning unsupervised learning into a supervised problem by learning to recognize positive pairs, we note that we cannot use this objective as our hyperparameter tuning criterion. First, the loss function is designed to optimize the model’s internal objective, which may not necessarily reflect meaningful patterns or structures in the data. For example, minimizing the loss in contrastive learning could lead to trivial solutions \citep{hua2021feature, tsitsulin2023unsupervised} that satisfy the objective but fail to capture important relationships in the graph, or the model may overfit to spurious correlations in the data, such as background features in images or noise in graphs \citep{chen2020SimCLR}. Second, the scale of the loss function can vary with different hyperparameters (particularly for those who directly impact the loss function, as for the composite loss used in \citet{zhang2021canonical}), complicating direct comparisons even within the same model architecture. To ensure robust evaluation, it is essential to employ a separate, universal metric that directly evaluates the learned embeddings to assess model performance.

However, due to the nonconvexity of the method, we do not expect the learned embeddings to be close, even when fitted by the same algorithm on the same dataset. Scale and location can in fact vary greatly from one run to the next.
To remedy these issues, we propose here using a procedure based on Canonical Correlation Analysis  (CCA) \citep{hotelling1936relations}. Canonical correlation analysis is a classical method for finding the correspondence between two datasets on the same samples by finding linear transformations of $X$ and $Y$ that maximizes their correlation. The CCA objective can be written as a prediction problem:
\begin{equation}\label{eq-cca-dist}
\begin{aligned}
    \hat{U}, \hat{V} \in 
    &\argmin_{U \in \mathbb{R}^{p_1 \times r},\, V \in \mathbb{R}^{p_2 \times r}} 
    \| XU - YV \|_F^2 \\
    &\text{subject to} \quad 
    U^T \Sigma_X U = I_r, \quad 
    V^T \Sigma_Y V = I_r.
\end{aligned}
\end{equation}
where ${\Sigma}_X$ and ${\Sigma}_Y$ denote the covariance matrices of $X$ and $Y$ respectively.
\begin{remark}
We emphasize that the normalization $U^T{\Sigma}_XU= V^T \Sigma_YV = I_r$ is here indispensable to ensure that the learned mappings between representations remain independent of the varying scales introduced by different GNN representations.
\end{remark}
\begin{remark}
We argue that the linear nature of mappings learned by CCA is not overly restrictive. While deep variants of CCA could be employed,  self-supervised embeddings are often used with simple linear models (e.g. the outputs are processed with a simple linear classifier to produce labels). Thus, restricting our objective to consider linear mappings seems reasonable.
\end{remark}

\subsection{CCA for Aligning Representations}

As we seek to evaluate unsupervised representations, in this section, we assume that we have generated $3 n_b$ independent versions of the dataset with the procedure described in Section~\ref{section:cv-split}. For each $i \in [2n_b]$, we learn an unsupervised representation of the nodes: $H_i = GNN_i(G_i, \theta)$, where $\theta$ indicates the tunable hyperparameters. We propose evaluating the quality of the learned representation by comparing the alignment of the embeddings learned by different models on replicas of the same dataset as per (\ref{eq-cca-dist}).


The solution to the CCA problem (\ref{eq-cca-dist}) has a closed-form expression. Let $U_0, V_0$ be the left and right singular vectors of the cross-covariance matrix:  
$$
\text{corr}(H_i, H_j) = \hat\Sigma_{H_i}^{-1/2} \hat{\Sigma}_{H_iH_j} \hat\Sigma_{H_j}^{-1/2} = U_0 \Lambda_0 V_0^\top,
$$
where $\hat{\Sigma}_{H_i}$ is the empirical covariance of embeddings from dataset $i$, and $\hat{\Sigma}_{H_iH_j}$ is the empirical cross-covariance of embeddings from datasets $i$ and $j $. The solutions to (\ref{eq-cca-dist}) are  
\begin{equation}
\hat{U}(i,j) = \hat \Sigma_{H_i}^{-1/2}U_0, \quad \hat{V}(i,j) = \hat \Sigma_{H_j}^{-1/2}V_0.
\end{equation}
and we can compute the alignment between versions of the dataset as:
\begin{align*}
\text{alignment} = \|H_i \hat{U}(i, j) -H_{j} \hat{V}(i,j) \|_F,
\end{align*}
where the alignment is evaluated and aggregated over the bootstrapped samples $i,j \in [n_b]$, $i\neq j$.

\subsection{Validation of the Evaluation Metric}\label{sec:validation-metric}

\begin{table*}
\centering
\begin{tabular}{@{}cccccccccc@{}}
\toprule
          & \multicolumn{3}{c}{\textbf{Spleen}}                 & \multicolumn{3}{c}{\textbf{TNBC}}                 & \multicolumn{3}{c}{\textbf{CRC}}                  \\ \midrule
$\lambda$ & ACC             & Mean            & SD     & AUC             & Mean           & SD    & AUC             & Mean           & SD    \\ \midrule
0.000001  & 0.4114          & 56,955          & 23,032 & 0.7566          & 4,765          & 3,617 & 0.8039          & 26,385         & 3,812 \\
0.00001   & {\ul 0.4135}          & 58,398          & 30,425 & 0.7487          & 5,217          & 4,284 & 0.8039          & 26,699         & 3,746 \\
0.0001    & \textbf{0.4146} & 40,017          & 12,422 & 0.7249          & 4,734          & 3,348 & 0.8170          & 25,972         & 5,934 \\
0.001     & 0.4128    & \textbf{21,741} & 5,732  & 0.7513          & 3,781          & 1,782 & 0.7974          & 8,844          & 1,893 \\
0.01      & 0.3691          & 42,336          & 1,970  & 0.7328          & 3,425          & 1,566 & 0.8431          & 6,425          & 1,319 \\
0.1       & 0.3986          & 50,351          & 2,550  & \textbf{0.8757} & \textbf{3,149} & 1,111 & \textbf{0.9412} & 5,940          & 1,538 \\
1         & 0.3914          & 55,264          & 2,788  & 0.8307          & 3,516          & 1,309 & {\ul 0.9346}    & \textbf{5,543} & 889   \\
10        & 0.3184          & 61,804          & 2,339  & {\ul 0.8704}    & 3,689          & 1,443 & 0.8627          & 5,951          & 1,301 \\ \bottomrule
\end{tabular}
\caption{For each dataset, the first column reports the downstream task performance, while the second and third columns present the mean and standard deviation of the evaluation metric defined in Equation~\ref{eq-cca-dist}. We adopt the architecture from \citet{zhang2021canonical} and fix all hyperparameters except for $\lambda$ in the CCA-SSG loss (Equation~\ref{eq:cca-ssg-loss}). Using Algorithm~\ref{alg:full}, the minimum average distances are achieved at $\lambda_\text{MS} = 0.001$ for the mouse spleen dataset \citep{goltsev2018deep}, $\lambda_{\text{TNBC}} = 0.1$ for Triple Negative Breast Cancer (TNBC) \citep{keren2018structured}, and $\lambda_\text{CRC} = 1.0$ for colorectal cancer (CRC) \citep{schurch2020coordinated}. Notably, strong downstream performance coincides with improved embedding alignment, as indicated by lower average distances reported in the second column for each dataset. }
\label{tab:bio-eval-metric}
\end{table*}

We evaluate the validity of our metric (\ref{eq-cca-dist}) on three biological datasets: the spleen dataset \citep{goltsev2018deep}, the MIBI-TOF breast cancer dataset \citep{keren2018structured}, and the colorectal cancer (CRC) dataset \citep{schurch2020coordinated}. Each dataset comprises multiple graphs, allowing us to assess the validity of our proposed metric (\ref{eq-cca-dist}) independently of the graph bootstrapping procedure. For instance, the spleen dataset includes samples from 3 mice, while the CRC dataset contains data from 31 patients. Table~\ref{tab:bio-eval-metric} presents both the evaluation of our metric along with the downstream task performance. In addition, visualizations provided in Appendix Figure~\ref{fig:spleen} further support the utility of our metric in guiding the hyperparameter selection (e.g., $\lambda$), effectively recovering biologically meaningful cell microenvironments. Detailed descriptions of the datasets and downstream tasks are provided in Appendix~\ref{sec:bio-apps}.

\subsection{Proposed Hyperparameter Tuning Framework}\label{sec:prop-cv}

We now describe the full procedure, which we call LOBSTUR ({\bf Lo}cal {\bf B}oot{\bf s}trap for {\bf T}uning {\bf U}nsupervised {\bf R}epresentations in GNNs).
The previous two sections have described the two core components of our procedure. We now detail an additional step to safeguard our pipeline against degeneracies.

\paragraph{Adjustment for Dimensional Collapse}
Our proposed alignment metric is grounded in a straightforward statistical method, Canonical Correlation Analysis (CCA). The strength of this method lies in its assessment of {\it correlations} between representations. However, because it accounts for different variances, this method may struggle to accurately reflect the quality of embeddings in the presence of dimensional collapse \citep{hua2021feature}. Dimensional collapse, a phenomenon common in self-supervised representation learning, occurs when the learned representations are confined to a low-dimensional manifold. For example, when training a model with an embedding dimension of $p=2$, dimensional collapse may result in embeddings that lie along a single line (reduced to a one-dimensional representation) or form a blob. In such cases, although the embeddings lack informative structure, their alignment across different samples may still be high, leading to an over-inflated metric.

The \textit{StableRank} metric \citep{tsitsulin2023unsupervised} is defined as $\sum_i \sigma_i^2/{\sigma_1^2}$, where $\sigma_i$ are the singular values of the embeddings $H \in \mathbb{R}^{n \times p}$ in descending order, and assesses the numerical rank of the embedding space. We will use this metric to filter out embeddings that are clearly suboptimal \citep{jing2022understandingdimensionalcollapsecontrastive} before applying our CCA-based metric to tune hyperparameters. An alternative choice for the threshold metric could be \textit{RankMe} proposed by \citet{garrido2023rankmeassessingdownstreamperformance}.

Our full procedure is highlighted in Algorithm~\ref{alg:full}. 

\begin{algorithm}
\caption{Hyperparameter Tuning Procedure}\label{alg:full}
\begin{algorithmic}[1]
    \STATE \textbf{Input:} An input graph $G$ and a set of hyperparameters $\Theta$ from which to choose an optimal value.
    \STATE Create $3n_b$ bootstrap samples of the graph, denoted as $\{\hat{G}_i\}_{i=1}^{3n_b}$ (Algorithm~\ref{alg:features} , \ref{alg:graph}).
    \FOR{each value $\theta \in \Theta$}
    \FOR{$i =1, \ldots, 2n_b$}
        \STATE Train an unsupervised GNN, $f_i(\cdot, \theta)$ on $\hat{G}_i$.
    \ENDFOR
    \FOR{each pair of models $f_i(\cdot, \theta)$ and $f_{i +n_b}(\cdot, \theta)$ with $i\in \{1, \cdots, n_b\}$}
        \STATE Compute the distance between embeddings from models $f_i$ and $f_{i +n_b}$ on the test graph $\hat{G}_{i +2n_b}$:
        $$d_i(\mathbf{\theta}) = \ell \big( f_i(\hat{G}_{i +2n_b}, \theta),  f_{i+n_b}(\hat{G}_{i +2n_b}, \theta) \big), $$ where $\ell(\cdot)$ is some metric, like the one we proposed in Section~\ref{sec:cv-eval}. 
    \ENDFOR
      \ENDFOR
    \STATE Choose the optimal hyperparameters:
    $
    \hat{\mathbf{\theta}} = \underset{\mathbf{\theta} \in \Theta, \textit{StableRank} \geq t}{\text{argmin }} \bar{d}(\mathbf{\theta}),\quad$ where $\bar{d}(\mathbf{\theta})$ is the average distance across $i \in [n_b]$, and $t$ is the \textit{StableRank} threshold.
\end{algorithmic}
\end{algorithm}

\begin{table*}[ht]
\centering
\resizebox{\textwidth}{!}{%
\begin{tabular}{@{}lccccccccc@{}}
\toprule
\textbf{Dataset} &
  \textbf{Default} &
  \textbf{Ours} &
  \textbf{$\alpha$-ReQ} &
  \textbf{pseudo-$\kappa$} &
  \textbf{RankME} &
  \textbf{NESum} &
  \textbf{SelfCluster} &
  \textbf{Stable Rank} &
  \textbf{Coherence} \\ 
\midrule
\multicolumn{10}{c}{\textit{Classification tasks}} \\ \midrule
Cora      & 0.36 & 0.65 & {\ul 0.66} & 0.54 & 0.63 & 0.63 & \textbf{0.69} & 0.59 & 0.47 \\
PubMed    & 0.62 & {\ul 0.81} & 0.75 & 0.75 & 0.75 & 0.75 & \textbf{0.82} & 0.75 & 0.76 \\
Citeseer  & 0.32 & \textbf{0.51} & \textbf{0.51} & \textbf{0.51} & \textbf{0.51} & \textbf{0.51} & {\ul 0.48} & \textbf{0.51} & 0.22 \\
CS        & 0.47 & {\ul 0.79} & \textbf{0.86} & 0.72 & \textbf{0.86} & \textbf{0.86} & \textbf{0.86} & \textbf{0.86} & 0.76 \\
Photo     & 0.29 & 0.73 & {\ul 0.79} & {\ul 0.79} & {\ul 0.79} & {\ul 0.79} & 0.57 & \textbf{0.81} & 0.69 \\
Computers & 0.37 & {\ul 0.57} & 0.45 & {\ul 0.57} & 0.45 & 0.39 & 0.39 & {\ul 0.57} & \textbf{0.65} \\
\midrule
\multicolumn{10}{c}{\textit{Regression tasks}} \\ \midrule
Chicago   & {\ul 0.39} & 0.34 & 0.35 & 0.35 & 0.35 & 0.35 & 0.35 & 0.29 & \textbf{0.40} \\
Anaheim   & 0.13 & \textbf{0.23} & 0.12 & {\ul 0.18} & {\ul 0.18} & 0.12 & \textbf{0.23} & {\ul 0.18} & 0.12 \\
Twitch    & 0.47 & \textbf{0.52} & 0.15 & 0.15 & 0.15 & 0.15 & 0.46 & 0.15 & {\ul 0.48} \\
Education & 0.23 & {\ul 0.26} & \textbf{0.33} & \textbf{0.33} & \textbf{0.33} & \textbf{0.33} & \textbf{0.33} & \textbf{0.33} & {\ul 0.26} \\
\midrule
Avg clf   & 0.41 & \textbf{0.68} & 0.67 & 0.65 & {\ul 0.66} & 0.65 & 0.63 & \textbf{0.68} & 0.59 \\
Avg reg   & 0.30 & \textbf{0.34} & 0.24 & 0.25 & 0.25 & 0.24 & \textbf{0.34} & 0.24 & {\ul 0.32} \\
\bottomrule
\end{tabular}%
}
\caption{Downstream task (classification or regression) performance of the best model and hyperparameters chosen by each criterion.  The best value is bolded and the second best is underlined.  We compare to the BGRL \citep{bgrl} with default hyperparameters ($\text{fmr}=0.5, \text{edr}=0.25, \lambda=10^{-2}$) in the left‑most column.}
\label{tab:exp-summary-downstream-performance}
\end{table*}

\section{Experiments} \label{section:exp}

\noindent{\bf Benchmark Datasets.}
We demonstrate the validity of our entire framework on GNN benchmark datasets such as Cora, Citeseer, and Pubmed. 
We show that hyperparameter and model selection using our suggested framework results in robust, high downstream task performance on benchmark datasets, thereby indicating embeddings of good quality. 
More specifically, we consider the task of learning unsupervised GNN embeddings using four different methods (CCA-SSG, BGRL, DGI, and GRACE, see Appendix~\ref{sec:gnn-review}), and choosing the correct set of hyperparameters in each method.
\textit{Note that we do not look at the classification accuracy ahead of time and use them for choosing the model and hyperparameters}. Instead, we only report them after choosing the model to validate the approach, reflecting a more practical scenario to apply unsupervised GNNs on real datasets. In Table~\ref{tab:exp-summary-downstream-performance}, we report the downstream task performance (classification or regression) of the model chosen by our framework (Algorithm~\ref{alg:full}) and metrics proposed in \citet{tsitsulin2023unsupervised}. Our method shows a robust performance and achieves either the best or the second best performance compared to the existing metrics for 7 out of 10 datasets, and achieving the best overall accuracy. A similar table reporting the performance by different GNN architectures \citep{bgrl,zhang2021canonical, zhu2020GRACE} is presented in Table~\ref{tab:exp-summary-downstream-performance-bgrl}, \ref{tab:exp-summary-downstream-performance-cca-ssg}, \ref{tab:exp-summary-downstream-performance-grace} in the Appendix.





\noindent{\bf Choice of the Threshold.} We turn to the problem of selecting a stable-rank threshold. We suggest using the reasonable lower bound for the latent (effective) dimension as sufficient. For Tables~\ref{tab:exp-summary-downstream-performance} and \ref{tab:exp-summary-spearman-corr-full}, we set the threshold to $t=2$. This choice ensures that the embeddings retain a minimum effective dimensionality, preventing collapse to a single line. Consequently, our alignment metric accurately measures meaningful signal alignment rather than trivial, collapsed patterns. It is important to highlight the trade-off associated with this threshold: setting a higher threshold enhances robustness but may inadvertently exclude beneficial models, while a lower threshold allows greater model diversity but risks increased variability and potential collapse of representations.

\section{Conclusion} \label{section:conclusion}
Hyperparameter tuning for unsupervised GNNs presents two primary challenges -- the complexity of generating multiple samples out of one observed graph and the difficulty in evaluating model performance without labeled data. In this paper, we propose a novel method for performing bootstrapping specifically tailored for unsupervised GNNs, facilitating both hyperparameter tuning and model selection. Although our validation is primarily empirical, we believe that this study highlights a more systematic approach for tuning graph neural networks and machine learning models in general, encouraging further exploration in this direction. Notably, our approach is applicable to graphs of moderate size (a few thousand nodes), but may not scale directly to larger graphs. A potential solution is to partition the graph and bootstrap within blocks. We present preliminary results in the appendix (see Appendix~\ref{sec:split-more}, \ref{sec:scalability}) suggesting the promise of this approach, though adapting the procedure for large-scale graphs remains an open question for future research.

\subsubsection*{Acknowledgments}
This work was completed in part with resources provided by the University of Chicago Research Computing Center, and was supported in part through the computational resources and staff contributions provided for the Mercury high performance computing cluster at The University of Chicago Booth School of Business which is supported by the Office of the Dean.


\newpage
\bibliography{main}
\bibliographystyle{apalike}

\appendix
\onecolumn
\newpage
\section{Additional definitions and proofs}

\subsection{Definitions}

Throughout this manuscript, we assume the same conventions as in the general literature on graphon estimation (see, for instance, ~\citet{gao2015rate, gaucher2021maximum}).

In particular, for a function $f: [0,1] \times [0,1] \to [0,1]$,
 the
derivative operator is defined by
\[
\nabla_{jk}f(x,y)=\frac{\partial^{j+k}}{(\partial x)^j(\partial y)^k}f(x,y),
\]
and we adopt the convention $\nabla_{00}f(x,y)=f(x,y)$.
\begin{definition}[H\"{o}lder class for Graphon functions (from \citet{gao2015rate})]\label{def:graphon_norm}
    
The H\"{o}lder norm is defined as
\begin{eqnarray*}
\left\lVert f\right\rVert _{\mathcal{H}_{\alpha}} &=& \max_{j+k\leq
{ \lfloor{\alpha}  \rfloor}}\sup
_{x,y\in\mathcal{D}}\bigl \lvert \nabla_{jk}f(x,y)\bigr \rvert
\\
&&{} +\max
_{j+k={ \lfloor{\alpha}  \rfloor}}\sup_{(x,y)\neq
(x',y')\in\mathcal{D}}\frac
{\left \lvert \nabla_{jk}f(x,y)-\nabla_{jk}f(x',y')\right \rvert
}{(\left \lvert  x-x'\right \rvert +\left \lvert  y-y'\right \rvert )^{\alpha
-{ \lfloor{\alpha}  \rfloor}}}.
\end{eqnarray*}
The H\"{o}lder class is defined by
\[
\mathcal{H}_{\alpha}(M)= \bigl\{\| f\| _{\mathcal{H}_{\alpha}}\leq
M: f(x,y)=f(y,x)\mbox{ for }x\geq y \bigr\},
\]
where $\alpha>0$ is the smoothness parameter and $M>0$ is the size of
the class, which is assumed to be a constant. 
\end{definition}

\begin{definition}[Distance Measures]
For nodes $i,j \in [n]$:
\begin{align*}
d_L(i,j) &= |U_i - U_j| \quad \text{(Latent distance)}\\
d_F(i,j) &= \|X_i - X_j\|_2 \quad \text{(Feature distance)} \\
d_G(i,j) &= \text{length of shortest path from $i$ to $j$} \quad \text{(Graph distance)}
\end{align*}
\end{definition}

Note that in the actual implementation, other graph distances are available as an option, but for the analysis purpose, we assume $D_G(\cdot, \cdot)$ is a shortest-path distance. 

\begin{definition}[k-NN Neighborhoods]
For node $i$:
\begin{align*}
\mathcal{N}_k^U(i) &= \{j: U_j \text{ is among k-nearest neighbors of } U_i\}\\
\mathcal{N}_k^X(i) &= \{j: X_j \text{ is among k-nearest neighbors of } X_i\} \\
\mathcal{N}_k^G(i) &= \{j: \text{node }j \text{ is among k-nearest neighbors of node } i\} 
\end{align*}
where $X_i = g(U_i) + \epsilon_i$ , and $U_i \sim \text{Unif}[0,1]$. The neighborhood is determined by corresponding distance. For example, the neighborhood in the latent space is determined by latent distance. 
\end{definition}

\subsection{Proof of Theorem~\ref{theorem:edge}}\label{app:proof}

\begin{proof}
Letting $\mathcal{N}_k(i)$ denote the $k^{th}$ closest neighbor of node $i$ according to the oracle graph $\mathcal{G}_{knn}$.

For any pair of nodes $(i,j)$, as we are resampling, we are  effectively replacing the underlying connection probability $\mathbb{P}[A_{ij} = 1| U_i, U_j]$ by:
$$    \hat{p}_{ij}  = \frac{1}{K} \sum_{k=1}^K A_{\mathcal{N}_k(i), j} $$
    We decompose the risk of this estimator as: 
   $$\mathbb{E}\left[\left(\mathbb{P}[A_{ij} = 1| U_i, U_j] - \hat{p}_{ij}\right )^2\right] =\text{Bias}^2 + \text{Variance} $$
   where
    \begin{equation}
        \begin{split}
         \qquad       \text{Bias} &= \mathbb{P}[A_{ij} = 1| U_i, U_j] - \E[\hat{p}_{ij}] \\
    &= \frac{1}{k} \sum_{m=1}^k \left(\mathbb{P}[A_{ij} = 1| U_i, U_j]  -  \mathbb{P}[Y_{\mathcal{N}_m(i), j}=1| U_{\mathcal{N}_m(i)}, U_j]\right) \\
                 &   \text{Variance}   =  \E\left[\left( \frac{1}{k} \sum_{m=1}^k (\mathbb{P}[Y_{\mathcal{N}_m(i), j}=1| U_{\mathcal{N}_m(i)}, U_j] -A_{\mathcal{N}_m(i), j}) \right)^2\right]
        \end{split}
    \end{equation}
  By assumption, since $W$ is assumed to be $\alpha$-H{\"o}lder, as emphasized in \citet{gao2015rate}, when $\alpha\in(0,1]$,
a function $f\in\mathcal{H}_{\alpha}(M)$ satisfies the Lipschitz condition

\begin{equation}
\bigl\lvert f(x,y)-f\bigl(x',y'\bigr)\bigr\rvert
\leq M\bigl(\bigl\lvert x-x'\bigr\rvert +\bigl\lvert
y-y'\bigr\rvert \bigr)^{\alpha}, \label{eqLip}
\end{equation}
Therefore, we have:

\begin{equation}
        \begin{split}
   |\text{Bias}| &= \left| \mathbb{P}[A_{ij} = 1| U_i, U_j] - \frac{1}{k} \sum_{m=1}^k \mathbb{P}[Y_{\mathcal{N}_m(i), j}=1| U_{\mathcal{N}_m(i)}, U_j] \right|\\
   &\leq \frac{1}{k} \sum_{m=1}^k M |U_i-U_{\mathcal{N}_m(i)}|^{\alpha}.
        \end{split}
    \end{equation}  

The quantity $| U_i- U_m|^{\alpha}$ (with $m$ a $k$-nearest neighbor of $i$) is a monotonously decreasing function of $n$, and
with high probability (over the distribution of $U_1, \cdots , U_n)$, 
we have $|U_i - U_m|_2  \leq {c_0 }\frac{k}{n}$, for all $m \in \mathcal{N}(U_i)$, and a constant $c_0$.  Therefore, as $n$ goes to infinity, $\lim_{n \to \infty} | \text{Bias} | =0. $

    Similarly, for the variance:
      \begin{equation}
        \begin{split}
   \text{Variance} &= \E\left[\left( \frac{1}{k} \sum_{m=1}^k (\mathbb{P}[Y_{\mathcal{N}_m(i), j}=1| U_{\mathcal{N}_m(i)}, U_j] -A_{\mathcal{N}_m(i), j}) \right)^2\right]\\ 
   &= \frac{1}{k^2} \sum_{m=1}^k \mathbb{P}[Y_{\mathcal{N}_m(i), j}=1| U_{\mathcal{N}_m(i)}, U_j](1-\mathbb{P}[Y_{\mathcal{N}_m(i), j}=1| U_{\mathcal{N}_m(i)}, U_j])\\ 
   &\leq \frac{1}{k}.
        \end{split}
    \end{equation}
    As $k\to \infty$, this converges to 0. 

    This shows that $\hat{p}_{ij}$ is a consistent estimator of $p_{ij}.$
    
\end{proof}

\section{Summary of Selected Unsupervised GNNs}\label{sec:gnn-review}

\input{gnn-review}

\section{Additional Literature Review}
\subsection{Cross-Validation}\label{sec:review-cv}
In the supervised learning literature, cross-validation (CV) 
 \citep{hastie01statisticallearning, Cluster2005Tibshirani} stands as a fundamental strategy for selecting hyperparameters and evaluating models. In the usual (Euclidean) setting, this technique involves partitioning the dataset into distinct subsets: a "training set" for model training and a "test set" for its evaluation. The partitioning is justified by the independence between observations, which implies that the subsamples still follow the same distribution as the original data. A commonly used method is $K$-fold cross-validation, where the dataset is divided into $K$ subsets or folds. For simplicity, we assume there are $n$ samples, and each fold has $m$ data points so that $n = K \times m$. We denote a set of index for the $k$-th fold as $I_k$. The model is trained $K$ times, each time using $K-1$ folds for training and the remaining fold for validation. Evaluation of the validation set is performed through an appropriate evaluation function $\ell(\cdot)$ measuring the discrepancy between the observations $y_i$ and their predicted values $\hat{y}_i = \hat{f}(x_i, \theta)$. This loss is usually taken to be the mean squared error(MSE) in the regression case, ($\text{MSE}_k = \frac{1}{m}\sum_{i \in I_k}(y_i-\hat{y}_i)^2$), or to be the classification accuracy in the classification setting. 
By averaging this metric over all $k$ folds, cross-validation provides a reliable estimate of the model's prediction error on unseen data. 

While the implementation and practice of cross-validation is simple and straightforward, its interpretation has only recently been investigated in work by \citet{Bates2024CV}. The authors' key finding is that the cross-validation does not estimate the prediction error for the model trained on a specific dataset but rather the ``average" prediction error across all possible training datasets from the same distribution. 
\begin{equation}\label{eq:cv}
\widehat{Err}^{(CV)}  = \frac{1}{n}\sum_{i=1}^n e_i = \frac{1}{K}\sum_{k=1}^K \frac{1}{m}\sum_{i \in I_k}\ell(\hat{f}(x_i, \hat\theta^{(-k)}),y_i).
\end{equation}
The intuition is the inner summation in Equation~\ref{eq:cv} estimates the prediction error of the model at hand, and the outer summation calculates the empirical average over all possible training sets of the same size.
In the previous equation, $\hat\theta^{(-k)}$ denotes the parameters of the model fitted on all but the $k^{th}$ fold, and $\hat{f}(x_i, \hat\theta^{(-k)})$ indicates the estimator of $y$.

\subsection{Cross-Validation for Unsupervised Learning}
Despite the popularity and simplicity of the cross-validation procedure, its application in unsupervised learning has been relatively underexplored, largely due to the absence of clear evaluation metrics. \citet{perry2009crossvalidation} addressed this gap by examining cross-validation in unsupervised settings and proposing several solutions, with a focus on methods utilizing Singular Value Decomposition (SVD). Among the strategies reviewed, two are particularly relevant for this discussion. The first is a traditional hold-out method, where a portion of the data is set aside for validation, and the second involves treating random elements of the dataset as "missing values." For a detailed explanation of these methods, refer to \citet{perry2009crossvalidation}, Chapter 5. However, it is important to note that these methods were originally designed for conventional, independent, tabular data for unsupervised tasks. In this study, we build on Perry's framework, focusing on its connection to graph neural networks (GNNs) and extending its use to evaluate unsupervised learning methods in the context of GNNs in Section~\ref{sec:cv-eval}.

For the hold-out method, we randomly partition the data $Z \in \mathbb{R}^{n \times p}$ into $\begin{pmatrix} Z_1 \\ Z_2 \end{pmatrix}$, where $Z_1 \in \mathbb{R}^{n_1 \times p}$ is a training set, $Z_2 \in \mathbb{R}^{n_2 \times p}$ is a test set, and $n_1+n_2 = n$. We want to approximate the test data by projecting it onto the principal spaces of the training data. To do so, one can calculate the k-dimensional reduced SVD of $Z_1$, where $\hat{Z}_1(k) = \sqrt(n)\hat{U}_1 \hat{D}_1(k) \hat{V}_1$. Project the test set onto the principal space of $Z_1$. 
$$\hat{Z}_2(k) = Z_2 {\hat{Z}_1(k)}^\top(\hat{Z}_1(k) \hat{Z}_1(k)^\top)^\dagger \hat{Z}_1(k) = Z_2 \hat{V}_1 \hat{V}_1^\top.$$ $X^\dagger$ denotes the pseudo-inverse of $X$. The performance ban be evaluated using $\ell_2$ loss, $\|Z_2-\hat{Z}_2(k)\|_F^2$. Although this method cannot be used in practice because the loss is a decreasing function with $k$, the idea of using projection to compute the projection error for unsupervised tasks was insightful.

The second is called either missing value strategy or Wold hold-outs.  
Instead of simply splitting the data, one could randomly select the indices $I \in \mathcal{I}$, which denote the missing elements. Then, $Z_I = \big\{\begin{array}{lr}
        Z_i & \quad i \in I\\
        * & \text{o.w}
        \end{array}$; similarly, $Z_{\bar{I}} = \big\{\begin{array}{lr}
        Z_i & \quad i \notin I\\
        * & \text{o.w}
        \end{array}$. 
Apply k-rank missing value SVD algorithm to find the decomposition of $Z_{\bar{I}}(k) = U_k D_k V_k^\top$. There are many options\citep{} including the one proposed by \citet{perry2009crossvalidation}. The performance can again be evaluated using $\|U_kD_kV_k^\top - Z_I\|_{F, I}^2$ 

The last method is basically to convert the unsupervised task into the supervised task, and called Gabriele hold-outs. Given the data, we could randomly permute the row and column so that we have the following decomposition $P^\top ZQ = \begin{pmatrix}Z_{11} & Z_{12}\\ Z_{21} & Z_{22}\end{pmatrix}$, where $P$ and $Q$ are the permutation matrices. 
        
There is continuing work on applying this hold-out approach (especially Gabriele's hold-out on clustering analysis \cite{fu2017estimating}.

\subsection{Cross Validation for Network Analysis}

There have been relatively few studies \citep{li2020network, hoff2007modeling, Chen2018Network} on the cross-validation of network data. In \citet{li2020network}, the key assumption for the entire analysis is that the edge is the realization of independent Bernoulli random variables, and the probability of connection $M$, which is realized by the observed adjacency matrix $A$, is approximately of low rank. The edge cross-validation proposed in this study is different from traditional node-splitting methods in that the random dropping applies to the connected pair of nodes. The model by \citet{Chen2018Network} is particularly designed for determining the number of communities within the network data, as well as choosing between the regular stochastic block model and the degree-corrected stochastic block model(DCSBM). The core idea is a block-wise node-pair splitting, which is then combined with an integrated step of community recovery using sub-blocks of the adjacency matrix. 

\citet{leiner2024graph} introduces another cross-validation method for graphs but approaches the problem from a different angle. The study applies data thinning to data following convolution-closed distributions by \citet{neufeld2023data}. This procedure creates data folds that maintain the same distribution as the original data, are independent of each other, and sum to the original random variable. 
A canonical example of it is a normal variable. Given data $X \sim(\mu, \sigma^2)$, with unknown parameter of interest $\mu$. Through data thinning algorithm, we could thin $X$ into  $X^{(1)} \sim N(\epsilon\mu, \epsilon^2\sigma^2)$ and $X^{(2)} \sim N((1-\epsilon)\mu, (1-\epsilon)^2\sigma^2)$, where these two thinned variables are independent to each other. \citet{leiner2024graph} is an extension of this concept to graph data, applying data thinning to node features while treating the adjacency matrix as fixed. 

However, all these statistical methods heavily rely on the certain generation mechanism of underlying networks, such as the stochastic block model \citep{Chen2018Network} or low-rank structure of expected value of adjacency matrix \citep{li2020network}. The assumptions of the aforementioned approaches on which part of the graph is a random component are also different. \citet{leiner2024graph} treats the graph structural component $(V, E)$ as non-random and the node feature as random; however, \cite{li2020network} treats edge as then random realization based on statistical graph generation model, such as stochastic block model.

\subsection{Bootstrap}\label{sec:split-more}

The bootstrap \citep{efron1979bootstrap} has been widely used as a non-parametric method for estimating the distribution of a statistic through resampling with replacement. This method is useful because it does not rely on assumptions about the underlying distribution, making it applicable across various fields where such assumptions are challenging. 
The validity of the bootstrap is supported by its consistency \citep{horowitz2019bootstrap} under mild assumptions, where the bootstrap distribution converges to the true sampling distribution as the sample size increases. However, the validity of the bootstrap relies on having access to independent samples, an assumption violated in the graph case. We thus consider two distinct scenarios, depending on the nature of the graph at hand:
\begin{itemize}
    \item \textit{For graphs with short-range dependencies}, such as for instance, spatial graphs: we propose to apply a graph-based \textbf{block bootstrap} method, inspired by its use in time series and spatial statistics \citep{Dimitris1994Stationary,castillo2019nonparametric}. The block bootstrap is based on the assumption that the dependency structure is well contained within the small neighborhood that we could assume independence among these neighborhoods. We extend the application of the block bootstrap to the graph case here by splitting the graph into smaller (non-overlapping) neighborhoods of size $B$, and creating new graphs based on replacing each of these neighborhoods by sampling with replacement from the total set of possible neighborhoods (see Algorithm~\ref{alg:block-bootstrap}). Similar to the spatial setting \citep{castillo2019nonparametric}, the size of the blocks is crucial to the success of the procedure. To guide the choice of the neighborhood, we propose using descriptive graph statistics (see next section) to generate graphs with similar characteristics.
    \item \textit{For graphs with long-range dependencies}, For non-spatial and homophilic graphs, we propose to use an extension of \textbf{network bootstrap} by \citet{levin2021bootstrapping}. In this work, \citet{levin2021bootstrapping} consider random dot product graphs (RDPG) where the  edge connectivity is determined by the inner product of the latent positions $H$ of two nodes: for each edge $A_{ij}$ between node $i$ and $j$,  $A_{ij} \sim \text{Bernouilli}(H_i^TH_j)$. The crux of this method is that by converting an observed network into its latent positions, we can leverage the independence among its latent variables. In our setting, we propose to extend this setting to larger classes of graphs by learning node representations $H_i = GNN(X, A)$ of the graph (see Algorithm~\ref{alg:bootstrap-network}). 
\end{itemize}

\begin{algorithm}
\caption{Resample Graphs through Block Bootstrap}
    \begin{algorithmic}[1]
    \STATE \textbf{Input:} Spatial coordinates $x\_coord$, $y\_coord$, $grid\_size$, $n\_samples$.
    \STATE \textbf{Output:} Block bootstrapped graphs $samples$.
    \FOR{$i = 1$ to $n\_samples$}
        \STATE \textbf{Step 1: Shuffle Data Points}
        \STATE Create a grid over the spatial domain using coordinates $x\_coord$ and $y\_coord$.
        \STATE Shuffle the grids to create new patched data, $shuffled\_data$.
        \STATE \textbf{Step 2: Convert Shuffled Data to Graphs}
        \STATE Convert $shuffled\_data$ into a graph by the method of choice (e.g. k-NN or radius graph)     
        \STATE Store the graph $samples[i] = G$
    \ENDFOR
    \end{algorithmic}
    \label{alg:block-bootstrap}
\end{algorithm}

\begin{algorithm}
\caption{Resample Graphs through Network Bootstrap}
\begin{algorithmic}[1]
    \STATE \textbf{Input:} Graph $G$, embedding dimension $d$, $n\_samples$, neighborhood size $k$
    \STATE \textbf{Output:} Bootstrapped graph samples $samples$
        \STATE Generate spectral embedding $H$ of adjacency matrix using top $d$ eigenvectors
        \FOR{$i = 1$ to $n\_samples$}
            \STATE Sample indices $bootstrap\_idx$ from $H$ with replacement
            \STATE Generate new graph $\hat{A}$ from bootstrapped latent positions
            \STATE Initialize node features $x$ as zeros in the new graph $\hat{G}$
            \FOR{each node $i$ in $\hat{G}$}
                \STATE Calculate distances from node $i$ to all other nodes in $H$
                \STATE Sort distances and find nearest neighbors (based on neighborhood size $k$)
                \STATE Randomly select a neighbor and assign its features to node $i$
            \ENDFOR
            \STATE Store the generated graph $\hat{G}$ in the sample set $samples$
        \ENDFOR
\end{algorithmic}
\label{alg:bootstrap-network}
\end{algorithm}

\paragraph{Bayesian Bootstrap} \citet{rubin1981bayesian} introduces the Bayesian bootstrap (BB) as a nonparametric alternative to traditional Bayesian inference, sidestepping the need for explicit likelihood functions. Unlike the frequentist bootstrap, which resamples data \textit{with replacement}, the Bayesian bootstrap assigns Dirichlet-distributed random weights to observed data points, generating a posterior distribution for parameters of interest. Specifically, for a dataset $X = \{x_1, x_2, ..., x_n\} $, instead of sampling with replacement as in the frequentist bootstrap, the Bayesian bootstrap draws a random probability vector $p = (p_1, p_2, ..., p_n) $ from a $Dir(1,1,…,1)$ distribution, ensuring that
$\sum_{i=1}^{n} p_i = 1 \quad \text{and} \quad p_i > 0$.
This randomized weighting serves as a Bayesian nonparametric prior, effectively treating the empirical distribution of the data as the prior distribution.

\citet{efron2012bayesian} explores the relationship between Bayesian inference and the parametric bootstrap, demonstrating how frequentist resampling techniques can be adapted to estimate posterior distributions. The key insight of this work is that the parametric bootstrap, traditionally used to approximate sampling distributions, can serve as an efficient computational tool for Bayesian inference when paired with importance sampling. \citet{efron2012bayesian} shows that bootstrap reweighting can be used to transform frequentist confidence intervals into Bayesian credible intervals. This approach provides a bridge between the two paradigms, enabling frequentist methods to yield posterior distributions without relying on Markov Chain Monte Carlo (MCMC) techniques. 

The Bayesian bootstrap provides a perspective for interpreting the proposed {nonparametric graph rewiring}, particularly when edge resampling is guided by shared neighborhood structure. Just as the BB assigns Dirichlet-distributed weights to data points to construct a posterior distribution, the graph rewiring process can be seen as assigning probabilistic weights to edges based on local graph structure, thereby producing alternative realizations of the same graph. In this context, the neighborhood-weighted resampling in LOBSTUR aligns with Bayesian importance sampling, where the rewired edges represent a form of pseudo-posterior distribution over network structures.

\paragraph{Extended VGAE Approach} 
Inspired by \citet{kipf2016vgae}, we tried using the Variational Graph Autoencoder(VGAE) as a new graph sampler. The extension was needed as the original method only reconstructed the adjacency matrix. The proposed loss function includes a feature reconstruction component alongside the edge reconstruction and KL divergence losses. With the edge decoder designed by the original work, the feature decoder generates reconstructed node features, and the reconstruction loss for features is based on the sum of squared differences between the original and reconstructed features. 

The total loss used for training consists of three parts: the KL divergence loss regularizing the latent variables, the edge reconstruction loss, and the feature reconstruction loss, scaled by a regularization parameter $\lambda$. The overall objective is:
$$\text{Total Loss} = \frac{\text{KL}}{n} + \text{loss}_A + \lambda \times \text{loss}_X$$
In our implementation, the parameter $\lambda$ controls the weight of the feature reconstruction in the loss. This allows the model to focus primarily on learning the graph structure while still incorporating node feature information.

\subsubsection{Experiments: Block Bootstrap}

\begin{figure}[H]
    \centering
    \includegraphics[width=0.9\linewidth]{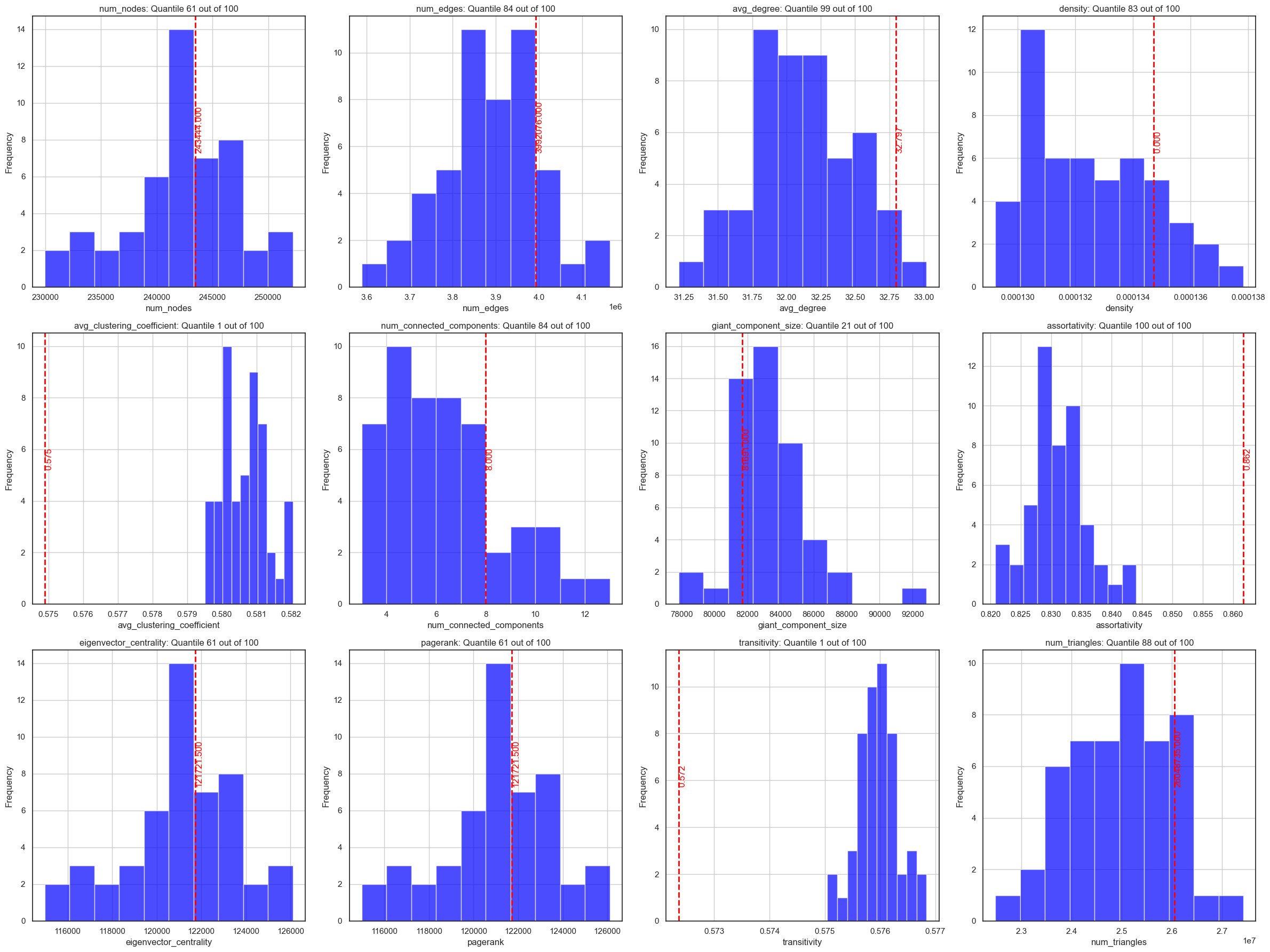}
    \caption{Block Bootstrap for Mouse Spleen data. Distribution of graph statistics of bootstrapped graphs. The principle is to see if the graph statistics of the original graph is within the extremity of the distribution of generated samples. The red dotted line indicates the statistics computed on the original graph. Most of the graph statistics do not lie at the extremity of the distribution of graph statistics by bootstrapped samples.}
    \label{fig:graph-satstitics}
\end{figure}

\begin{figure}[H]
    \centering
    \includegraphics[width=0.9\linewidth]{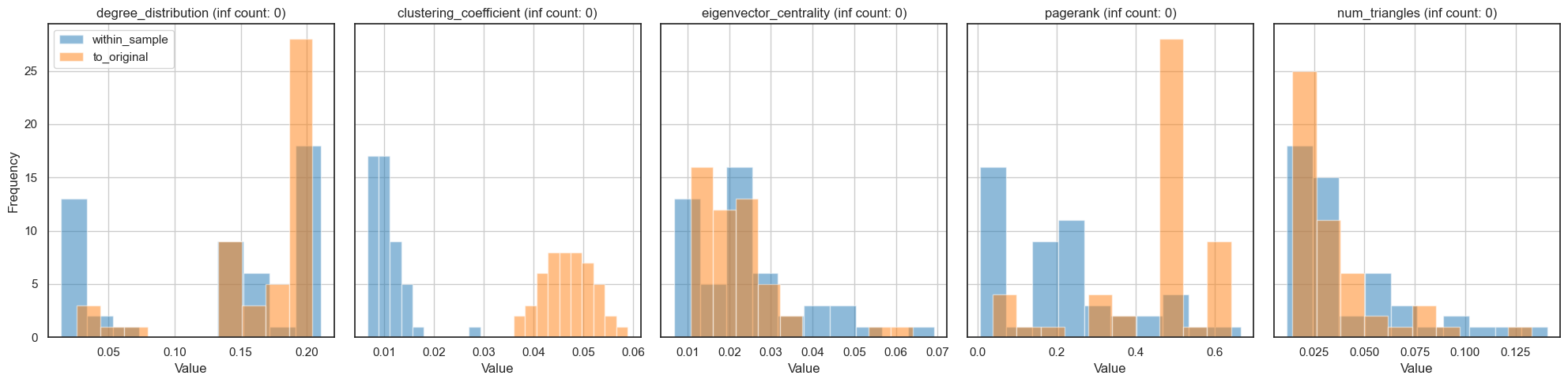}
    \caption{Block Bootstrap for Mouse Spleen data. Distribution of node-level statistics of bootstrapped graphs. The orange-colored distribution represents the JS divergence between the bootstrapped samples and the original graph, and the blue-colored distribution represents among bootstrapped samples divergence. The more the two distributions overlap, the bootstrapped samples `mimic' the original graph well in terms of node-level statistics.}
    \label{fig:node-statistics}
\end{figure}

\subsection{Evaluating Embedding Qualities}
In Section~\ref{sec:cv-eval}, we propose a stability metric. There are few works proposing metrics to evaluate the quality of unsupervised embeddings, although they are not intended for hyperparameter tuning. 

\paragraph{Alignment-based metrics}
Our first family of metrics focuses on measuring how well two embeddings align with each other. Suppose we have two embeddings, $H_i$ and $H_j$, produced by the same learning procedure but on different graph folds. We propose two discretized versions of (\ref{eq-cca-dist}), measuring how much two embeddings align with each other. 

\begin{enumerate}
    \item \textbf{Label Matching}: The first thing we can think of is to make the label from the embedding from each fold, which follows the ``converting to the supervised task'' convention. 
    \begin{enumerate}
        \item Determine the clusters on embeddings using simple clustering algorithm such K-Nearest Neighbor or Gaussian Mixture Model(GMM)
        \item Use widely used clustering evaluation metrics, such as Adjusted Rand Index(ARI) by \citet{hubert1985comparing} or Normalized Mutual Index(NMI), to see the labels from $H_i$ and $H_j$ agree to each other.
    \end{enumerate}
    \item \textbf{Neighborhood Matching}: If the model is able to extract enough of the latent structure of data, the model trained on the different folds of a graph should be similar. With this reasoning, we can evaluate the model by how much of the neighborhoods in the embedding agree with each other. To avoid the usage of data twice, we will evaluate the neighborhood from $H_{i}$ and $H_{j}$ and report the ratio of overlapping neighbors. To construct the neighborhood in the embedding space, we will use the simple k-Nearest algorithm with varying sizes of k. 
For each node on output embeddings, $H_i$ and $H_j$, we first find the m-nearest neighbors. Then for node-level neighbor-kept ratio is defined as  $N_i(m) = \text{\# of overlapped neighbors}/m$, where $m \leq k$ is the neighbor size. Graph-level ratio can be calculated by simply averaging over the nodes, $N(m) = \sum_i N_i(m)$.


\end{enumerate}

\paragraph{Direct Embedding Quality Metrics}
Beyond measuring alignment between two embeddings, one can also evaluate an embedding’s internal quality or degree of collapse. These methods offer a complementary view: even if two embeddings align with each other, they could both be suffering from dimension collapse or poor distribution of singular vectors.

\begin{enumerate}
    \item \textbf{RankMe}: \citet{garrido2023rankmeassessingdownstreamperformance} proposes \textit{RankMe} a metric to measure the effective dimension of embeddings to quantify the embedding collapse in self-supervised learning. To overcome the numerical instability of the exact rank computation, for example, due to round-off error, they propose an alternative to use Shannon entropy of normalized singular values. The formula was originally proposed by \citet{roy2007effective} and then applied to dimension collapse context by \citet{garrido2023rankmeassessingdownstreamperformance}. Formally, 
$$\text{RankMe}(H) = \exp\left( - \sum_{k=1}^{\text{min}(n,p)} p_k \log p_k\right), \text{ with } p_k = \frac{\sigma_k(H)}{\|\sigma(H)\|_1}+\epsilon.$$
    \item \textbf{Metrics proposed in \citet{tsitsulin2023unsupervised}}: \citet{tsitsulin2023unsupervised} further extended the approaches and proposed four different metrics to evaluate the embedding quality in terms of embedding collapse and stability perspective. The key differences between their experiment setting and ours are, first, \citet{tsitsulin2023unsupervised} only consider the graph structure, not the node features, and second, they do not change the model parameters but change the level of perturbation on the structure (edge dropping or node masking).
    Let $H \in \mathbb{R}^{n \times p}$ be an embedding obtained from the trained unsupervised model of choice.
\begin{enumerate}
    \item \textbf{Coherence}: The coherence metric measures how concentrated the rows of the singular vector matrix $U$ are. A low coherence indicates that the energy is spread more uniformly across all rows (good for compressed sensing), while a high coherence suggests that the energy is concentrated in a few rows, which can indicate a poorly distributed set of singular vectors.
    $$ \text{Coherence}(H) = \frac{\max_i \|U_i\|_2^2 \cdot n}{p},$$ where $U \in \mathbb{R}^{n \times p}$ is reduced left singular matrix of $H \in \mathbb{R}^{n \times p}$. 
    \item \textbf{Stable Rank}: It is the quantity called a `numerical rank'' (or effective rank) in numerical analysis. 
    $$\text{Stable Rank}(H) = \frac{\|H\|_F^2}{\|H\|_2^2 },$$ where $\|H\|_F$ denotes the Frobenius norm, $\|H\|_F^2 = \sum \sigma_i^2$, and $\|H\|_2 = \sigma_1$, where $\sigma_1 \geq \sigma_2 \geq \cdots \geq \sigma_n$ denotes the singular values of $H$.
    \item \textbf{Pseudo-condition number}: Let SVD of the embedding $H$ be $H = U\Sigma V^\top$. 
    $$\kappa_p(H) =\|H\|_p\|H^\dagger\|_p\overset{p=2}{=} \frac{\sigma_1}{\sigma_n}$$
    \item \textbf{SelfCluster}: The idea is to estimate how much the embeddings are clustered in the embedding space compared to random distribution on a sphere. Let $\tilde{H} \in \mathbb{R}^{n \times p}$ be the normalized embeddings. $$\text{SelfCluster}(H) = \frac{\|\tilde{H}\tilde{H}^\top\|_F - n - n\times(n-1)/d}{n^2-n - n\times(n-1)/d}$$
\end{enumerate}
\end{enumerate}

\section{Scalability}\label{sec:scalability}
While our framework performs well on graphs of moderate size (up to 19k nodes, e.g., the Pubmed citation network), scalability remains a challenge. The bootstrapping procedure and CCA-based evaluation introduce significant additional computation, which can limit applicability to larger graphs. In particular, when applying our method to the OGBN-Arxiv dataset (over 170k nodes), we encountered substantial runtime challenges that made the process very time-consuming.

The main limitation, however, stems from the need to train multiple graph neural networks (GNNs) during the bootstrapping process. This requirement significantly increases the computational cost, but it is essential to ensure robust hyperparameter selection, especially in high-precision applications such as finance or biomedical domains, where reliability and unbiased evaluation are critical.

To address scalability challenges, we have begun exploring two strategies (1) block bootstrapping where the graph is partitioned into smaller subgraphs and bootstrapping is applied within blocks; and (2) approximate rewiring schemes to reduce computational overhead during resampling. Preliminary results for block bootstrapping, presented in \ref{sec:split-more}, suggest that this direction holds promise. 

\subsection{Alternative Algorithm for Scalability}

\begin{algorithm}[H]
\caption{Approximate Edge Rewiring via $A^2$}\label{alg:approximate-bootstrap}
\begin{algorithmic}[1]
    \STATE \textbf{Input:} Graph $G = (\mathcal{V}, \mathcal{E})$ with $n = |\mathcal{V}|$ nodes, flattened list of edge stems $L = \{ u \mid u \in E[:,0] \} \cup \{ v \mid v \in E[:,1] \}$,
    where $E \in \mathbb{R}^{|\mathcal{E}| \times 2}$, and the squared adjacency matrix $A^2$ representing 2-hop connectivity strengths between nodes.
    \STATE Initialize an empty graph $G' = (\mathcal{V}, \mathcal{E}')$ with $n$ nodes.
    \STATE Compute the sparse adjacency matrix $A$ of $G$ and symmetrize it to ensure it is undirected.
    \STATE Compute the matrix product $A^2 = A \times A$, remove self-loops by setting the diagonal of $A^2$ to zero, and eliminate any zero entries.
    \WHILE{$\texttt{len}(L) > 0$}
        \STATE Sample a source node $u$ uniformly at random from the list $L$, and remove it from $L$.
        \STATE Retrieve the set of candidate nodes $v$ for $u$, where each candidate $v$ satisfies $A^2_{uv} > 0$ and $v \neq u$, and where $v \in L$.
        \STATE If no such candidate exists, discard $u$ and continue to the next iteration.
        \STATE Otherwise, sample a target node $v$ from the set of candidates according to the normalized weights given by $A^2_{uv}$.
        \STATE Remove the sampled node $v$ from the list $L$.
        \STATE Add an undirected edge between $u$ and $v$ in the graph $G'$.
    \ENDWHILE
    \STATE \textbf{Output:} Bootstrapped graph $G'$ with resampled edge structure.
\end{algorithmic}
\end{algorithm}

The original edge rewiring algorithm (Algorithm~\ref{alg:graph} explores a node's local 1-hop neighborhood at each iteration. For a randomly selected node $u$, it first identifies its $k$-nearest neighbors based on some graph-based distance, then for each $k$-nearest neighbor $m$, it retrieves all direct neighbors $\mathcal{N}_A(m)$ in the original graph. The node $u$ samples a new connection $v$ from this dynamically built candidate set, with probabilities weighted by the frequency of appearance across different $m$. This ensures that edge resampling captures local neighborhood information around each node. However, this procedure incurs high computational cost because it needs to explore multiple neighborhoods at every rewiring step. 

The approximate algorithm (Algorithm~\ref{alg:approximate-bootstrap}) instead precomputes the 2-hop neighborhood connectivity of the graph by squaring the adjacency matrix, yielding $A^2$. Here, $A \in \mathbb{R}^{n \times n}$ is the (symmetric) adjacency matrix of the graph, and $A^2_{ij}$ counts the number of distinct 2-hop paths between nodes $i$ and $j$. In this setting, each node $u$ directly samples a target node $v$ from the 2-hop neighbors based on their weighted connection strength given by $A^2_{uv}$. The sampling probability is proportional to the number of 2-hop paths between $u$ and $v$, i.e.,
$$P(v | u) = \frac{A^2_{uv}}{\sum_{v' \in \mathcal{C}(u)} A^2_{uv'}},$$
where $\mathcal{C}(u)$ is the set of candidates for node $u$ with positive 2-hop connectivity and available stems. If no candidates are found, the algorithm discards $u$ and continues.

The relationship between Algorithm~\ref{alg:graph} and the approximate method (Algorithm~\ref{alg:approximate-bootstrap}) depends on the degree of each node and the choice of $k$ for the $k$-nearest neighbor graph. Specifically, whether the candidate set in the original method is larger or smaller than the set of direct neighbors depends on the comparison between a node's degree and the size of $k$. If a node $u$ has a low degree, meaning it is connected to only a few nodes in the original graph, then the $k$-nearest neighbor (k-NN) graph will forcefully connect it to $k$ other nodes based on feature similarity or graph distance, even if $u$ does not have that many direct connections. In this case, the $k$-NN set can be larger than the direct 1-hop neighbor set. The original algorithm supplements the missing local structure by adding neighbors based on external feature similarity rather than existing edges. Consequently, when $\deg(u) < k$, the original rewiring may result in a broader candidate set than the direct neighborhood.

On the other hand, if a node $u$ has a high degree, meaning it is already connected to many nodes in the adjacency graph, then the $k$-nearest neighbor graph selects only a subset of its many neighbors. Here, $k$-NN acts as a filter, choosing the most ``important'' or closest $k$ neighbors, possibly ignoring others. In this case, because $k$ is smaller than the degree, the $k$-NN candidate set becomes smaller than the full direct neighborhood. When $\deg(u) \geq k$, the original algorithm is thus more restrictive compared to simply traversing all direct neighbors.

Therefore, the original algorithm is not always narrower or broader by default; it depends on the relative size of a node’s degree and $k$. This behavior is different from the approximate method using $A^2$, where no such filtering exists. The approximate method (Algorithm~\ref{alg:approximate-bootstrap}) uses all nodes that are reachable in exactly two hops, without considering feature space distances or $k$-nearest neighbor constraints. As a result, the approximate method includes any node with a 2-hop path from a node $u$, potentially adding candidates that would never have been explored in the original method, especially when the node's degree is small and the $k$-NN graph must reach out to faraway nodes.

\newpage
\section{GNN Experiment Details}

We use benchmark datasets for node classification, including Cora, Pubmed, and Citeseer, and test our framework on node regression datasets from \cite{huang2023conformalized_gnn}. We summarize the datasets used to demonstrate the entire hyperparameter tuning procedure in Table~\ref{tab:data-summary}.

We consider various benchmark datasets for node classification tasks, including Cora, Pubmed, Citeseer. Additionally, we have tested our framework on a few datasets for the node regression by \cite{huang2023conformalized_gnn}. To demonstrate our full framework for hyperparameter tuning, we used the following datasets, and their details are summarized in Table~\ref{tab:data-summary}. 

\begin{table}[ht]
\centering
\renewcommand{\arraystretch}{1.2} 
\resizebox{\linewidth}{!}{%
\begin{tabular}{crrrcc}
\toprule
         Dataset &  Num Nodes &  Num Edges &  Num Classes &                                        Description &                             Source \\
\midrule
            Cora &       2708 &       5429 &            7 &         Citation network & PyTorch Geometric\\
        Citeseer &       3327 &       4732 &            6 &         Citation network & PyTorch Geometri \\
          Pubmed &      19717 &      44338 &            3 &         Citation network & PyTorch Geometric \\
    Amazon Photo &       7650 &     119081 &            8 & Product co-purchasing network &    PyTorch Geometric \\
Amazon Computers &      13752 &     245861 &           10 & Product co-purchasing network  &    PyTorch Geometric \\
     Coauthor CS &      18333 &      81894 &           15 & Coauthorship network&  PyTorch Geometric \\
Anaheim & 914 & 3881 & -- & Graph of transportation networks & Conformalized GNN \citep{huang2023conformalized_gnn} \\
ChicagoSketch & 2176 & 15104 & -- & Urban traffic network (sketch) & Conformalized GNN \citep{huang2023conformalized_gnn} \\
County Education & 3234 & 12717 & -- & County-level education metrics (2012) & Conformalized GNN \citep{huang2023conformalized_gnn}\\
Twitch PTBR & 1912 & 3170 & -- & Brazilian Twitch interactions & Conformalized GNN \citep{huang2023conformalized_gnn} \\
\bottomrule
\end{tabular}%
}
\caption{Summary of benchmark datasets used for the experiments, including both classification and regression datasets.}
\label{tab:data-summary}
\end{table}




The followings are tested combinations of hyperparameters, including different types of unsupervised GNN models. 

\begin{itemize}
    \item model: \{CCA-SSG, DGI, BGRL, GRACE\}
    \item feature masking rate (FMR): $\{ 0.05, 0.25, 0.5, 0.75\}$
    \item edge dropping rate (EDR): $\{ 0.05, 0.25, 0.5, 0.75\}$
    \item $\lambda$ (CCA-SSG, BGRL) or $\tau$ (GRACE): $\{10^{-5}, 10^{-4}, 10^{-3}, 10^{-2}, 10^{-1}, 1.0, 10.0 \}$ 
    \item number of layers: 2
    \item hidden dimension: $256$
    \item output dimension ($p$): 8
    \item learning rate: $10^{-3}$
    \item epochs: 500
    \item number of simulations for each dataset ($n_b$): 20 
\end{itemize}

\subsection{Computer Resources Used}\label{computer-resource}
The experiments in this study were conducted using a combination of personal and institutional computational resources. Preliminary analyses and prototyping were performed on a MacBook Pro with an Intel Core i7 processor and 16GB of RAM. For larger-scale experiments, including graph bootstrapping and downstream evaluations, we used high-performance computing resources provided by the institution’s research cluster, which includes access to multi-core CPUs and GPU-enabled nodes. While execution time varied by dataset and task, typical runs for clustering and evaluation completed within a few hours. Detailed resource specifications and runtime profiles are available upon request to support reproducibility.

\section{Additional Tables and Figures}\label{app-additional-figure}

\begin{figure}[H]
    \centering
    \includegraphics[width = \textwidth]{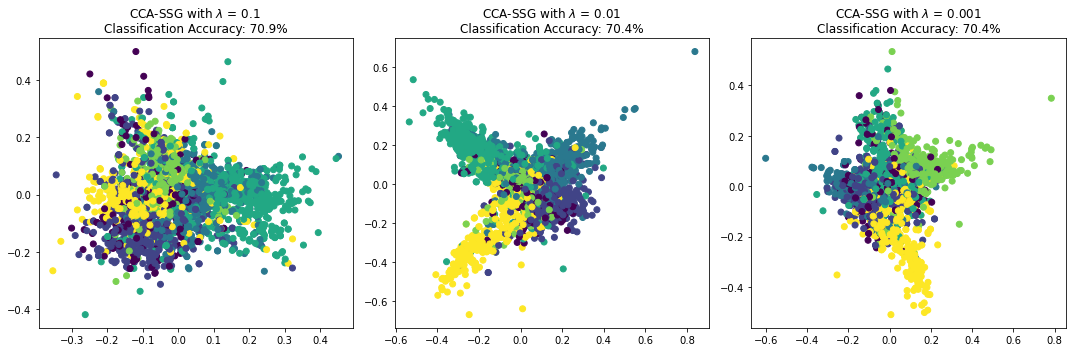}
    \caption{Citeseer: Model trained by different hyperparameters. 2D Visualization through PCA. The learned representations vary by the choice of hyperparameters.}
    \label{fig:citeseer-embeddings}
\end{figure}

\subsection{Validation of Metrics}

\paragraph{Synthetic Datasets.}
The motivation for using synthetic data is that we know the exact data-generating process (DGP), enabling us to replicate the dataset and focus on validating our metric. By controlling the DGP, we remove confounding factors related to real-world data and can better isolate and evaluate the performance of algorithms and metrics.

In this synthetic dataset generation, we create spatially structured data using a simple Gaussian blob. We first define $n$ cluster centers and standard deviations to simulate spatial groupings in a 2D space, which belong to distinct clusters. For each point, we generate a 32-dimensional feature vector, with features generated from Laplacian eigenmap by \cite{belkin2001laplacian}. 
The final dataset includes 2D spatial coordinates, cluster labels, and 32-dimensional feature vectors. We generated 15 copies of graphs following the same (and known) data-generating process. We run the procedure (Algorithm~\ref{alg:full}) and compute the metrics' average and prediction accuracy (Figure~\ref{fig:exp-synthetic}). Our proposed metric matches the clustering alignments (NMI, ARI) and shows a strong negative correlation with accuracy.

\begin{figure}[H]
    \centering
    \begin{subfigure}{0.3\textwidth}
        \includegraphics[width = \linewidth]{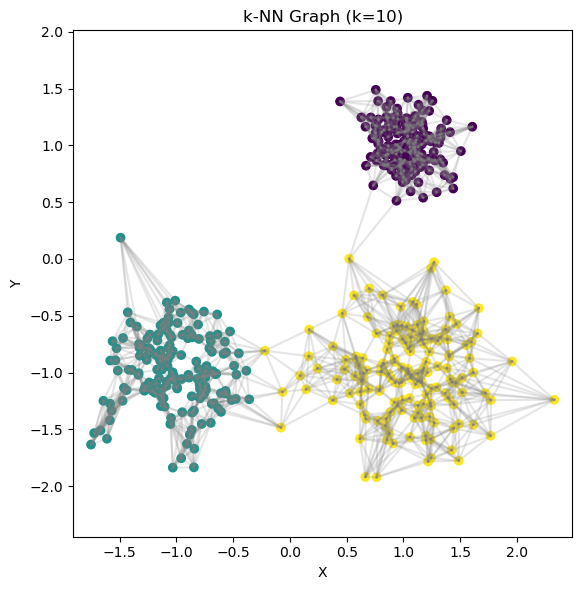}
        \caption{}
    \end{subfigure}
    \begin{subfigure}{0.62\textwidth}
        \includegraphics[width = \linewidth]{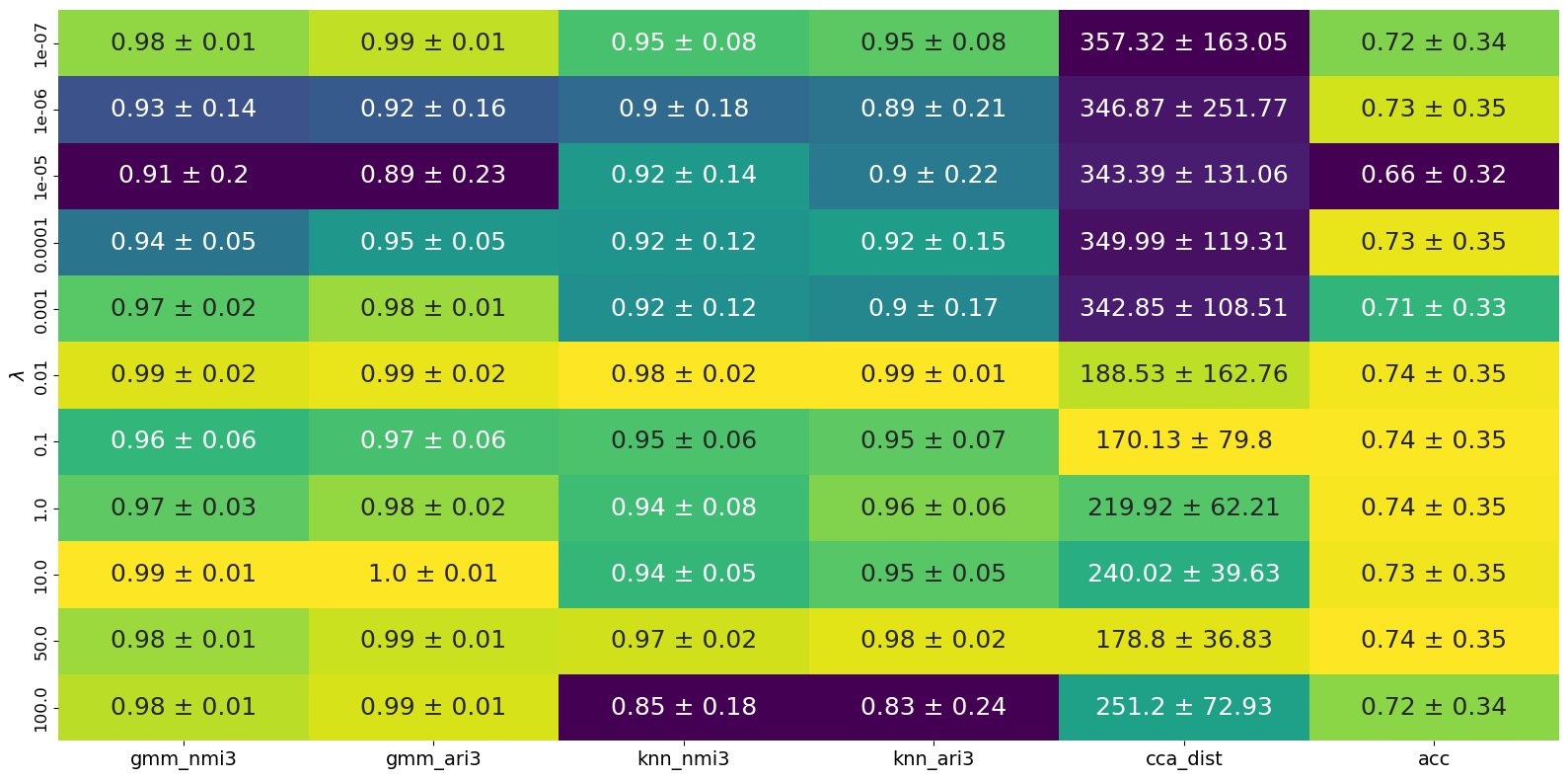}
    \caption{} 
    \end{subfigure}
    \caption{Summary of synthetic dataset and experiment results. The proposed metric and prediction accuracy show a strong negative Spearman rank correlation (-0.71).}
    \label{fig:exp-synthetic}
\end{figure}

\subsubsection{Application to Spatial Single-Cell Datasets}\label{sec:bio-apps}

There is a growing demand for robust computational tools that can extract biologically meaningful representations across heterogeneous samples. In such applications, it is crucial to obtain consistent and high-quality embeddings that generalize across samples while preserving fine-grained spatial structure. Our proposed metric is particularly well-suited for this goal, as it evaluates the stability and informativeness of unsupervised embeddings without requiring labeled data. When annotations are available, we further demonstrate that our method aligns closely with manual labels, exhibiting strong spatial continuity and biological interpretability across a range of datasets.

\paragraph{Mouse Spleen (CODEX)}

We apply our procedure to a high-resolution spatial proteomics dataset of the mouse spleen from \citet{goltsev2018deep}. This dataset, generated using CO-Detection by Indexing (CODEX), provides single-cell spatial and phenotypic profiles of immune cells across intact spleen tissue. With over 30 measured protein markers, it enables precise mapping of cell types, functional states, and spatial interactions at sub-tissue resolution. The dataset preserves key anatomical compartments—including T cell zones, B cell follicles, and red and white pulp—and highlights how spatial arrangement corresponds to immune function, such as structured lymphocyte zones and compartmentalized myeloid populations. We also have an access to the expert annotated lables, which we report the accuracy against it in Table~\ref{tab:bio-eval-metric}. Figure~\ref{fig:spleen} also refelects varying quality of learned embeddings by the choice of $\lambda$.

\begin{figure}[H]
    \centering
    \includegraphics[width=1\linewidth]{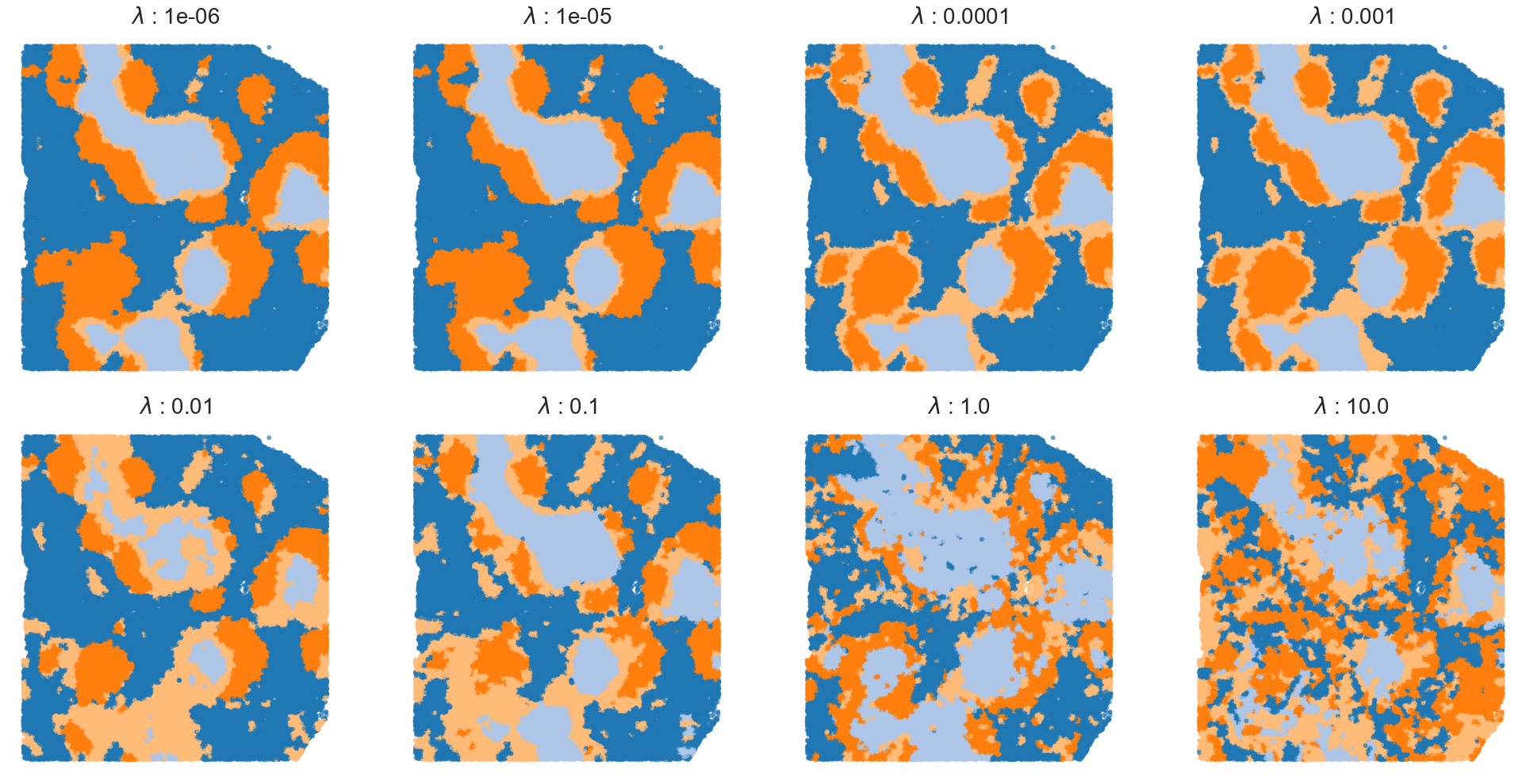}
   \caption{Visualizations of mouse spleen CODEX data based on the output of CCA-SSG model with different $\lambda$ settings. We can observe that depending on the choice of $\lambda$, the quality of expression varies a lot. When $\lambda$ becomes too large, the learned representation fail to recover the underlying cell environments. See Section~\ref{sec:validation-metric} for the setup.}
   \label{fig:spleen}
\end{figure}

\paragraph{Triple Negative Breast Cancer Dataset}
The dataset from \citet{keren2018structured} comprises MIBI-TOF imaging data from 41 TNBC patients, capturing the spatial expression of 36 proteins across tumor, immune, and regulatory markers at subcellular resolution. Tumors are classified into three immune architectures—cold, mixed, and compartmentalized—based on spatial patterns of immune infiltration, cell type organization, and marker expression. Compartmentalized tumors are linked to the best survival outcomes. Mixed tumors, featuring intermingled tumor and immune cells with high CD8+ T cell and checkpoint marker expression, may benefit from immunotherapy. Cold tumors show sparse immune presence and poor prognosis. Among these, the mixed and compartmentalized tumor microenvironments (TMEs) represent favorable immune architectures that the authors aim to recover, as they are identified through a combination of cell type composition, spatial organization, and marker expression profiles. 
We predict such group (mixed vs. comparatmentalized) based on the learned embeddings. The AUC for the prediction is reported in Table~\ref{tab:bio-eval-metric}.

\paragraph{Colorectal Cancer Dataset}

The colorectal cancer (CRC) dataset from \citet{schurch2020coordinated} includes 140 tissue regions from 35 advanced-stage CRC patients, profiled using FFPE-CODEX imaging with 56 protein markers to identify diverse cell populations within the tumor microenvironment (TME). The study identified nine distinct cellular neighborhoods (CNs) through unsupervised clustering of spatial co-occurrence patterns, revealing how the spatial organization of immune and stromal cells shapes immune responses. Two major immune architectures emerged (1) Crohn’s-like reaction (CLR), associated with structured immune infiltration and favorable outcomes, and  (2) diffuse inflammatory infiltration (DII), marked by disorganized immune presence and poor prognosis. These TMEs, distinguished by differences in cell types, spatial arrangements, and marker expression, represent the patterns the authors aim to recover, as they reflect clinically relevant immune organization associated with patient survival. The AUC for the predicting such group (CLR vs. DII) is reported in Table~\ref{tab:bio-eval-metric}.

\subsection{Validation of Bootstrap Samples} \label{app:bootstrap-validation}

\begin{table}[H]
\centering
\resizebox{0.75\textwidth}{!}{%
\begin{tabular}{@{}cccc@{}}
\toprule
\textbf{}                  & \textbf{True} & \textbf{Solution 1} & \textbf{Solution 2 (graph k-NN)} \\ \midrule
Number of Nodes            & 2708          & 2708±0              & 2708±0                           \\
Number of Edges            & 5278          & 5200.54±9           & 5171.78±7.34                     \\
Average Degree             & 3.8980        & 3.84±0.01           & 3.82±0.01                        \\
Density                    & 0.0014        & 0±0                 & 0±0                              \\
Avg Clustering Coefficient & 0.2407        & 0.01±0              & 0.05±0                           \\
Avg Connected Component    & 78            & 13.16±3.26          & 67.91±6.7                        \\
Giant Component Size       & 2485          & 2684.28±6.51        & 2620.38±10.78                    \\
Assortativity              & -0.0659       & -0.06±0             & -0.07±0                          \\
PageRank                   & 1353.5        & 1353.5±0            & 1353.5±0                         \\
Transitivity               & 0.0935        & 0.01±0              & 0.03±0                           \\
Number of Triangles        & 1630          & 133.96±11.62        & 471.48±27.11                     \\ \bottomrule
\end{tabular}%
}
\caption{Graph statistics for Cora illustrating the two solutions suggested in Section~\ref{sec:split-noisy}. We see the clear deviation on graph statistics, especially the average connected component and the number of triangles when we follow the \textit{Solution 1}.}
\label{tab:cora-exp-noisty-setting}
\end{table}

\begin{figure}[H]
    \centering
    \includegraphics[width=\linewidth]{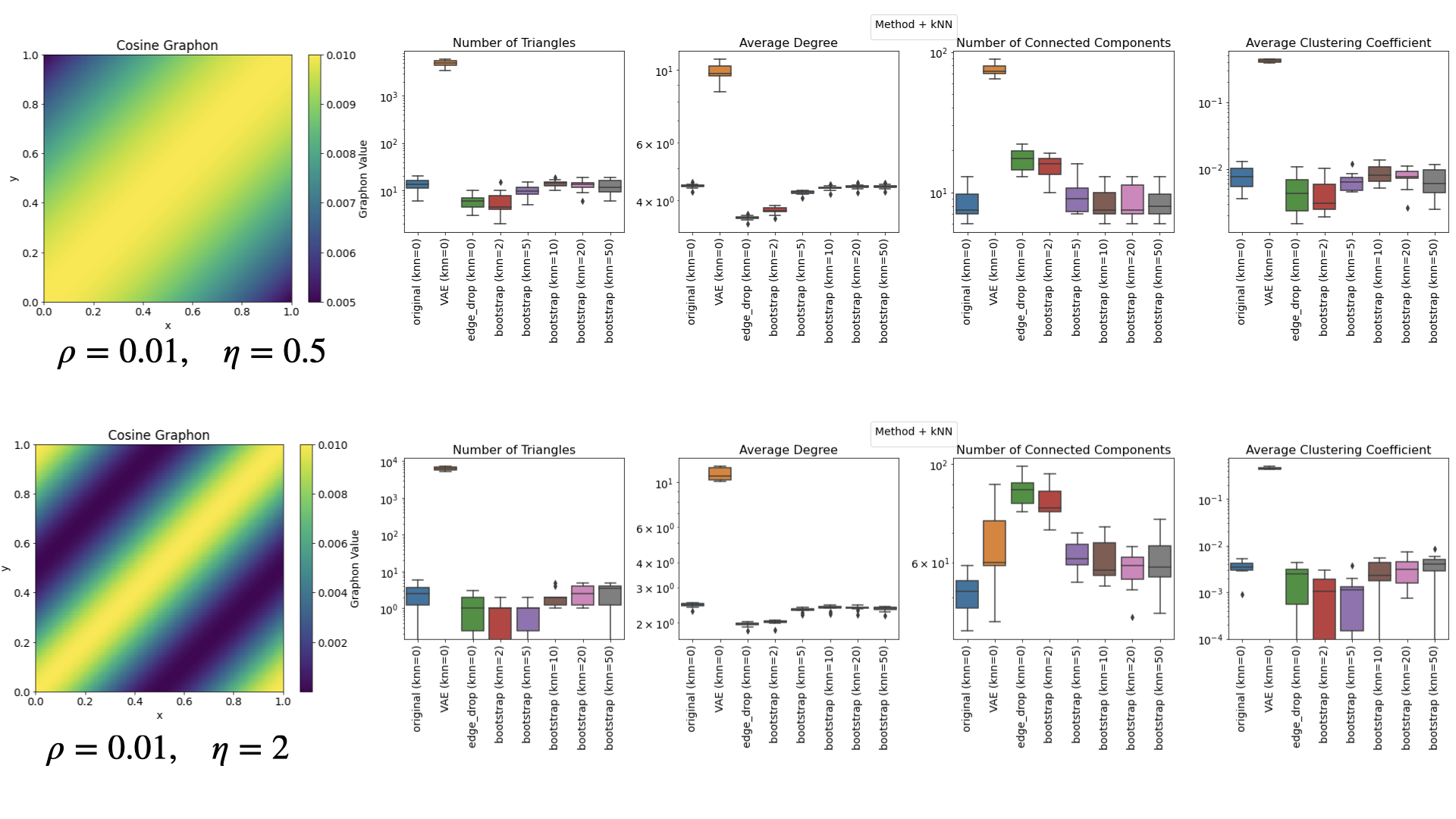}
    \caption{Visualization of the statistics obtained by different methods. The left most plot in each row corresponds to a visualization of the graphon function $W(x,y) = \rho * (1 + \cos(\eta  \pi \cdot (x - y))) / 2$, for $\rho=0.01$ and different values of $\eta$. Each row presents a visualization of 10 instances of a resampling of a graphon generated according to $W$ using different methods.}
    \label{fig:graphon_comparison}
\end{figure}

\begin{table}[H]
\centering
\renewcommand{\arraystretch}{1.2} 
\resizebox{\textwidth}{!}{%
\begin{tabular}{@{}cccccccc@{}}
\toprule
                                      & \textbf{}                  & \textbf{True} & \textbf{Edge Drop} & \textbf{Node Drop} & \textbf{Ours} & \textbf{NB} & \textbf{VAE} \\ \midrule
\multirow{11}{*}{\textbf{Scenario 1}} & Assortativity                & -0.0345 & 0±0         & 0±0          & 0±0         & 0±0          & -0.44±0.05      \\
                                      & Avg ClusterCoefficient       & 0       & 0±0         & 0±0          & 0±0         & 0±0          & 0.12±0.01       \\
                                      & Avg Degree                   & 0.12    & 0.12±0      & 0.1±0.01     & 0.08±0      & 0.04±0.02    & 2.63±0.47       \\
                                      & Density                      & 0.0002  & 0±0         & 0±0          & 0±0         & 0±0          & 0.01±0          \\
                                      & Giant Component Size         & 3       & 2.93±0.26   & 2.53±0.5     & 2.79±0.4    & 4.12±1.36    & 90.08±8.22      \\
                                      & Num Connected Components     & 470     & 471.13±0.87 & 380.72±2.44  & 480±1.15    & 492.81±3.52  & 410.89±8.21     \\
                                      & Num Edges                    & 30      & 28.87±0.87  & 19.28±2.44   & 20±1.15     & 8.82±5.12    & 657.73±118.12   \\
                                      & Num Nodes                    & 500     & 500±0       & 400±0        & 500±0       & 500±0        & 500±0           \\
                                      & Num Triangles                & 0       & 0±0         & 0±0          & 0±0         & 1.48±2.22    & 3060.65±1108.48 \\
                                      & PageRank                     & 249.5   & 249.5±0     & 199.5±0      & 249.5±0     & 249.5±0      & 249.5±0         \\
                                      & Transitivity                 & 0       & 0±0         & 0±0          & 0±0         & 0.33±0.31    & 0.47±0.05       \\ \midrule
\multirow{11}{*}{\textbf{Scenario 2}} & Assortativity              & -0.0227       & -0.02±0.01         & -0.02±0.04         & -0.02±0.03    & 0.01±0.04   & 0.04±0.04    \\
                                      & Avg ClusterCoefficient       & 0.0064  & 0.01±0      & 0.01±0       & 0.01±0      & 0.02±0.01    & 0.16±0.01       \\
                                      & Avg Degree                   & 3.224   & 3.1±0.02    & 2.58±0.07    & 3.2±0.01    & 3.68±0.29    & 3.13±0.09       \\
                                      & Density                      & 0.0065  & 0.01±0      & 0.01±0       & 0.01±0      & 0.01±0       & 0.01±0          \\
                                      & Giant Component Size         & 479     & 475.02±2.21 & 360.97±6.3   & 481.21±2.64 & 407.42±11.16 & 70.91±7.09      \\
                                      & Num Connected Components     & 18      & 21.41±1.81  & 33.23±4.71   & 17.89±1.25  & 91.93±10.74  & 391.57±4.03     \\
                                      & Num Edges                    & 806     & 774.11±4.42 & 515.92±13.17 & 801.17±1.81 & 921±71.47    & 782.12±21.59    \\
                                      & Num Nodes                    & 500     & 500±0       & 400±0        & 500±0       & 500±0        & 500±0           \\
                                      & Num Triangles                & 7       & 6.16±0.89   & 3.59±1.35    & 5.42±2.36   & 68.31±21.3   & 4469.92±189.99  \\
                                      & PageRank                     & 249.5   & 249.5±0     & 199.5±0      & 249.5±0     & 249.5±0      & 249.5±0         \\
                                      & Transitivity                 & 0.0083  & 0.01±0      & 0.01±0       & 0.01±0      & 0.04±0.01    & 0.69±0.01       \\ \midrule
\multirow{11}{*}{\textbf{Scenario 3}} & Assortativity              & -0.0227       & -0.02±0.01         & -0.02±0.04         & -0.01±0.03    & 0.01±0.04   & -0.58±0.06   \\
                                      & Avg ClusterCoefficient       & 0.0064  & 0.01±0      & 0.01±0       & 0.01±0      & 0.02±0.01    & 0.11±0.02       \\
                                      & Avg Degree                   & 3.224   & 3.1±0.02    & 2.58±0.07    & 3.21±0.01   & 3.68±0.29    & 1.28±0.18       \\
                                      & Density                      & 0.0065  & 0.01±0      & 0.01±0       & 0.01±0      & 0.01±0       & 0±0             \\
                                      & Giant Component Size         & 479     & 475.02±2.21 & 360.97±6.3   & 481.2±2.77  & 407.42±11.16 & 49.52±10.72     \\
                                      & Num Connected Components     & 18      & 21.41±1.81  & 33.23±4.71   & 17.84±1.28  & 91.93±10.74  & 421.24±11.14    \\
                                      & Num Edges                    & 806     & 774.11±4.42 & 515.92±13.17 & 801.7±1.61  & 921±71.47    & 321.11±45.67    \\
                                      & Num Nodes                    & 500     & 500±0       & 400±0        & 500±0       & 500±0        & 500±0           \\
                                      & Num Triangles                & 7       & 6.16±0.89   & 3.59±1.35    & 5.74±2.34   & 68.31±21.3   & 798.48±165.04   \\
                                      & PageRank                     & 249.5   & 249.5±0     & 199.5±0      & 249.5±0     & 249.5±0      & 249.5±0         \\
                                      & Transitivity                 & 0.0083  & 0.01±0      & 0.01±0       & 0.01±0      & 0.04±0.01    & 0.46±0.06       \\ \midrule
\multirow{11}{*}{\textbf{Scenario 4}} & Assortativity                & -0.0833 & 0±0         & 0±0          & 0±0         & 0±0          & -0.02±0.11      \\
                                      & Avg ClusteringCoefficient & 0       & 0±0         & 0±0          & 0±0         & 0±0.01       & 0.23±0.01       \\
                                      & Avg Degree                  & 0.104   & 0.1±0       & 0.08±0.01    & 0.07±0.01   & 0.04±0.02    & 7.39±0.37       \\
                                      & Density                      & 0.0002  & 0±0         & 0±0          & 0±0         & 0±0          & 0.01±0          \\
                                      & Giant Component Size       & 3       & 3±0.06      & 2.77±0.42    & 2.58±0.49   & 4.28±1.43    & 81.19±7.89      \\
                                      & Num Connected Components & 474           & 475.1±0.82         & 383.45±2.4         & 482.85±1.29   & 492.15±3.75 & 350.19±9.04  \\
                                      & Num Edges                   & 26      & 24.9±0.82   & 16.55±2.4    & 17.15±1.29  & 9.68±5.66    & 1847.28±91.84   \\
                                      & Num Nodes                   & 500     & 500±0       & 400±0        & 500±0       & 500±0        & 500±0           \\
                                      & Num Triangles               & 0       & 0±0         & 0±0          & 0±0         & 1.75±2.65    & 19727.09±959.62 \\
                                      & PageRank                     & 249.5   & 249.5±0     & 199.5±0      & 249.5±0     & 249.5±0      & 249.5±0         \\
                                      & Transitivity                 & 0       & 0±0         & 0±0          & 0±0         & 0.34±0.32    & 0.75±0.02       \\ \bottomrule
\end{tabular}%
}
\caption{Comparison of all sampling methods on graphon model as posited in Equation~\ref{eq:graphex}. The ground truth graphon is generated with n = 500, p = 150, k = 20. For each methods, 500 (bootstrap) samples are generated. For edge and node drop, we randomly remove 20\% of edges or nodes (and corresponding edges). Our proposed nonparametric bootstrap consistently achieves significant similarities with the ground truth graph under different scenarios. }
\label{tab:graphon-compare-all}
\end{table}

We analyze four scenarios (each in Table~\ref{tab:synthetic-graph-stat-s1}, \ref{tab:synthetic-graph-stat-s2}, \ref{tab:synthetic-graph-stat-s3}, and \ref{tab:synthetic-graph-stat-s4}) of recovering the underlying dependency by our proposed nonparametric bootstrap method either through graph-knn or feature-knn. The graph is generated by the model posited in Equation~\ref{eq:graphex} with varying graphon function $W$ and feature generator $g$.

\begin{table}[H]
\centering
\renewcommand{\arraystretch}{1.4} 
\resizebox{\textwidth}{!}{%
\begin{tabular}{@{}ccccccccc@{}}
\toprule
 &
  \textbf{} &
  \textbf{} &
  \multicolumn{3}{c}{\textbf{graph-kNN}} &
  \multicolumn{3}{c}{\textbf{feature-kNN}} \\ \midrule
 &
  \textbf{} &
  \textbf{Original} &
  \textbf{k = 5} &
  \textbf{k=20} &
  \textbf{k=50} &
  \textbf{k = 5} &
  \textbf{k=20} &
  \textbf{k=50} \\ \midrule
\multirow{11}{*}{\textbf{n= 100}} & Assortativity            &     \cellcolor[HTML]{A6A6A6}    &   \cellcolor[HTML]{A6A6A6}         &   \cellcolor[HTML]{A6A6A6}         & \cellcolor[HTML]{A6A6A6}           & 0±0        & 0±0       & 0±0        \\
                                  & Avg ClusterCoefficient   & 0       & 0±0        & 0±0        & 0±0        & 0±0        & 0±0       & 0±0        \\
                                  & Avg Degree               & 0.02    & 0±0        & 0±0        & 0±0        & 0±0        & 0±0       & 0±0        \\
                                  & Density                  & 0.0002  & 0±0        & 0±0        & 0±0        & 0±0        & 0±0       & 0±0        \\
                                  & Giant Component Size     & 2       & 1±0        & 1±0        & 1±0        & 1±0        & 1±0       & 1±0        \\
                                  & Num Connected Components & 99      & 100±0      & 100±0      & 100±0      & 100±0      & 100±0     & 100±0      \\
                                  & Num Edges                & 1       & 0±0        & 0±0        & 0±0        & 0±0        & 0±0       & 0±0        \\
                                  & Num Nodes                & 100     & 100±0      & 100±0      & 100±0      & 100±0      & 100±0     & 100±0      \\
                                  & Num Triangles            & 0       & 0±0        & 0±0        & 0±0        & 0±0        & 0±0       & 0±0        \\
                                  & PageRank                 & 49.5    & 49.5±0     & 49.5±0     & 49.5±0     & 49.5±0     & 49.5±0    & 49.5±0     \\
                                  & Transitivity             & 0       & 0±0        & 0±0        & 0±0        & 0±0        & 0±0       & 0±0        \\ \midrule
\multirow{11}{*}{\textbf{n=500}}  & Assortativity            & -0.0345 & 0±0        & 0±0        & 0±0        & 0±0        & 0±0       & 0±0        \\
                                  & Avg ClusterCoefficient   & 0       & 0±0        & 0±0        & 0±0        & 0±0        & 0±0       & 0±0        \\
                                  & Avg Degree               & 0.12    & 0±0        & 0.01±0     & 0.02±0     & 0.04±0.01  & 0.08±0    & 0.11±0     \\
                                  & Density                  & 0.0002  & 0±0        & 0±0        & 0±0        & 0±0        & 0±0       & 0±0        \\
                                  & Giant Component Size     & 3       & 2.04±0.82  & 2.32±0.47  & 2.34±0.47  & 2±0        & 2.79±0.4  & 2.94±0.24  \\
 &
  Num Connected Components &
  470 &
  498.96±0.82 &
  497.07±0.82 &
  495.4±0.96 &
  489.19±1.61 &
  480±1.15 &
  472.94±0.86 \\
                                  & Num Edges                & 30      & 1.04±0.82  & 2.93±0.82  & 4.6±0.96   & 10.81±1.61 & 20±1.15   & 27.06±0.86 \\
                                  & Num Nodes                & 500     & 500±0      & 500±0      & 500±0      & 500±0      & 500±0     & 500±0      \\
                                  & Num Triangles            & 0       & 0±0        & 0±0        & 0±0        & 0±0        & 0±0       & 0±0        \\
                                  & PageRank                 & 249.5   & 249.5±0    & 249.5±0    & 249.5±0    & 249.5±0    & 249.5±0   & 249.5±0    \\
                                  & Transitivity             & 0       & 0±0        & 0±0        & 0±0        & 0±0        & 0±0       & 0±0        \\ \midrule
\multirow{11}{*}{\textbf{n=1000}} &
  Assortativity &
  -0.112 &
  -0.39±0.15 &
  -0.36±0.14 &
  -0.33±0.13 &
  0±0 &
  -0.01±0.11 &
  -0.02±0.08 \\
                                  & Avg ClusterCoefficient   & 0       & 0±0        & 0±0        & 0±0        & 0±0        & 0±0       & 0±0        \\
                                  & Avg Degree               & 0.182   & 0.03±0.01  & 0.03±0.01  & 0.04±0.01  & 0.09±0.01  & 0.14±0    & 0.17±0     \\
                                  & Density                  & 0.0002  & 0±0        & 0±0        & 0±0        & 0±0        & 0±0       & 0±0        \\
 &
  Giant Component Size &
  4 &
  3.45±0.5 &
  3.42±0.49 &
  3.46±0.51 &
  3.37±0.49 &
  4.06±0.6 &
  4.23±0.48 \\
 &
  Num Connected Components &
  909 &
  984.21±2.87 &
  983.32±2.82 &
  980.71±2.92 &
  955.1±2.99 &
  928.42±2.25 &
  916.2±1.51 \\
                                  & Num Edges                & 91      & 15.79±2.87 & 16.68±2.82 & 19.29±2.92 & 44.9±2.99  & 71.6±2.24 & 83.81±1.51 \\
                                  & Num Nodes                & 1000    & 1000±0     & 1000±0     & 1000±0     & 1000±0     & 1000±0    & 1000±0     \\
                                  & Num Triangles            & 0       & 0±0        & 0±0        & 0±0        & 0±0        & 0.02±0.13 & 0±0.06     \\
                                  & PageRank                 & 499.5   & 499.5±0    & 499.5±0    & 499.5±0    & 499.5±0    & 499.5±0   & 499.5±0    \\
                                  & Transitivity             & 0       & 0±0        & 0±0        & 0±0        & 0±0        & 0.01±0.04 & 0±0.01     \\ \bottomrule
\end{tabular}%
}
\caption{Scenario 1: the graph structure is highly localized ($W(u,v) = \mathds{1}\{|u - v| < 0.01\}$), leading to disconnected components and the failure of graph-based kNN, while features ($\mathcal{N}(5u, 0.01)$) exhibit a strong correlation with the latent variable $u$, enabling feature-based kNN success. The greyed-out cells indicate values that are unavailable.}
\label{tab:synthetic-graph-stat-s1}
\end{table}

\begin{table}[H]
\centering
\renewcommand{\arraystretch}{1.4} 
\resizebox{\textwidth}{!}{%
\begin{tabular}{@{}ccccccccc@{}}
\toprule
 & \textbf{}                & \textbf{} & \multicolumn{3}{c}{\textbf{graph-kNN}}   & \multicolumn{3}{c}{\textbf{feature-kNN}}  \\ \midrule
 &
  \textbf{} &
  \textbf{Original} &
  \textbf{k = 5} &
  \textbf{k=20} &
  \textbf{k=50} &
  \textbf{k = 5} &
  \textbf{k=20} &
  \textbf{k=50} \\ \midrule
\multirow{11}{*}{\textbf{n= 100}} &
  Assortativity &
  0.2884 &
  -0.24±0.16 &
  -0.12±0.18 &
  -0.13±0.17 &
  -0.01±0.19 &
  -0.11±0.17 &
  -0.1±0.17 \\
 & Avg ClusterCoefficient   & 0         & 0±0          & 0±0         & 0±0         & 0±0         & 0±0          & 0±0          \\
 & Avg Degree               & 0.66      & 0.37±0.03    & 0.49±0.03   & 0.6±0.02    & 0.53±0.03   & 0.63±0.02    & 0.65±0.01    \\
 & Density                  & 0.0067    & 0±0          & 0±0         & 0.01±0      & 0.01±0      & 0.01±0       & 0.01±0       \\
 & Giant Component Size     & 8         & 6.15±1.21    & 6.36±1.4    & 5.97±1.23   & 5.22±1.26   & 5.86±1.25    & 6.09±1.39    \\
 & Num Connected Components & 67        & 81.9±1.59    & 75.54±1.45  & 70.24±1.24  & 73.58±1.56  & 68.5±0.79    & 67.6±0.56    \\
 & Num Edges                & 33        & 18.27±1.57   & 24.54±1.44  & 29.78±1.24  & 26.44±1.55  & 31.52±0.79   & 32.41±0.55   \\
 & Num Nodes                & 100       & 100±0        & 100±0       & 100±0       & 100±0       & 100±0        & 100±0        \\
 & Num Triangles            & 0         & 0±0          & 0.02±0.15   & 0.02±0.13   & 0.01±0.12   & 0.01±0.12    & 0.01±0.12    \\
 & PageRank                 & 49.5      & 49.5±0       & 49.5±0      & 49.5±0      & 49.5±0      & 49.5±0       & 49.5±0       \\
 & Transitivity             & 0         & 0±0          & 0.01±0.03   & 0±0.03      & 0±0.03      & 0±0.02       & 0±0.02       \\ \midrule
\multirow{11}{*}{\textbf{n=500}} &
  Assortativity &
  -0.0359 &
  -0.12±0.03 &
  -0.08±0.03 &
  -0.08±0.03 &
  -0.04±0.03 &
  -0.02±0.03 &
  -0.01±0.04 \\
 & Avg ClusterCoefficient   & 0.0052    & 0±0          & 0.01±0      & 0.01±0      & 0.01±0      & 0.01±0       & 0.01±0       \\
 & Avg Degree               & 3.076     & 2.91±0.02    & 3.03±0.01   & 3.02±0.01   & 2.96±0.02   & 3.06±0.01    & 3.07±0       \\
 & Density                  & 0.0062    & 0.01±0       & 0.01±0      & 0.01±0      & 0.01±0      & 0.01±0       & 0.01±0       \\
 &
  Giant Component Size &
  471 &
  463.59±3.66 &
  466.91±3.48 &
  468.15±3.02 &
  466.85±3.11 &
  469.38±2.64 &
  469.54±2.58 \\
 & Num Connected Components & 29        & 34.37±2.24   & 31.73±1.94  & 31.15±1.79  & 32.07±1.91  & 29.84±1.27   & 29.66±1.23   \\
 & Num Edges                & 769       & 727.53±5.05  & 757.88±2.53 & 756.11±2.92 & 740.32±4.78 & 764.09±1.75  & 767.16±1.04  \\
 & Num Nodes                & 500       & 500±0        & 500±0       & 500±0       & 500±0       & 500±0        & 500±0        \\
 & Num Triangles            & 7         & 4.08±2.01    & 5.03±2.33   & 5.09±2.33   & 4.54±2.05   & 5.37±2.38    & 5.48±2.33    \\
 & PageRank                 & 249.5     & 249.5±0      & 249.5±0     & 249.5±0     & 249.5±0     & 249.5±0      & 249.5±0      \\
 & Transitivity             & 0.0087    & 0.01±0       & 0.01±0      & 0.01±0      & 0.01±0      & 0.01±0       & 0.01±0       \\ \midrule
\multirow{11}{*}{\textbf{n=1000}} &
  Assortativity &
  0.0122 &
  -0.05±0.02 &
  -0.03±0.02 &
  -0.04±0.02 &
  -0.01±0.02 &
  0±0.02 &
  0±0.02 \\
 & Avg ClusterCoefficient   & 0.0078    & 0.01±0       & 0.01±0      & 0.01±0      & 0.01±0      & 0.01±0       & 0.01±0       \\
 & Avg Degree               & 6.594     & 6.37±0.02    & 6.56±0.01   & 6.59±0      & 6.5±0.01    & 6.58±0       & 6.59±0       \\
 & Density                  & 0.0066    & 0.01±0       & 0.01±0      & 0.01±0      & 0.01±0      & 0.01±0       & 0.01±0       \\
 & Giant Component Size     & 999       & 998.74±0.5   & 998.98±0.13 & 999±0.06    & 998.94±0.26 & 998.99±0.08  & 998.99±0.15  \\
 & Num Connected Components & 2         & 2.25±0.5     & 2.02±0.13   & 2±0.06      & 2.06±0.25   & 2.01±0.08    & 2.01±0.1     \\
 & Num Edges                & 3297      & 3185.15±9.77 & 3281.22±3.6 & 3293.46±1.6 & 3248.4±6.06 & 3289.46±2.37 & 3294.19±1.49 \\
 & Num Nodes                & 1000      & 1000±0       & 1000±0      & 1000±0      & 1000±0      & 1000±0       & 1000±0       \\
 & Num Triangles            & 50        & 40.6±6.86    & 47.34±6.66  & 48.53±6.7   & 46.97±6.8   & 50.03±7.75   & 50.54±7.07   \\
 & PageRank                 & 499.5     & 499.5±0      & 499.5±0     & 499.5±0     & 499.5±0     & 499.5±0      & 499.5±0      \\
 & Transitivity             & 0.0069    & 0.01±0       & 0.01±0      & 0.01±0      & 0.01±0      & 0.01±0       & 0.01±0       \\ \bottomrule
\end{tabular}%
}
\caption{Scenario 2 has a well-structured graph ($W(u,v) = 1 - |u - v|$), ensuring graph kNN success, but highly oscillatory features $(\sin(10u) + \mathcal{N}(0, 0.1)$) disrupt feature-based kNN. }
\label{tab:synthetic-graph-stat-s2}
\end{table}

\begin{table}[H]
\centering
\renewcommand{\arraystretch}{1.4} 
\resizebox{\textwidth}{!}{%
\begin{tabular}{@{}ccccccccc@{}}
\toprule
 &
  \textbf{} &
  \textbf{} &
  \multicolumn{3}{c}{\textbf{graph-kNN}} &
  \multicolumn{3}{c}{\textbf{feature-kNN}} \\ \midrule
 &
  \textbf{} &
  \textbf{Original} &
  \textbf{k = 5} &
  \textbf{k=20} &
  \textbf{k=50} &
  \textbf{k = 5} &
  \textbf{k=20} &
  \textbf{k=50} \\ \midrule
\multirow{11}{*}{\textbf{n= 100}} & Assortativity & -0.0344 & -0.15±0.14 & -0.08±0.15 & -0.06±0.14 & -0.05±0.15 & -0.09±0.16 & -0.08±0.15 \\
 &
  Avg ClusterCoefficient &
  0 &
  0±0 &
  0±0.01 &
  0±0.01 &
  0±0.01 &
  0±0.01 &
  0±0.01 \\
 &
  Avg Degree &
  0.78 &
  0.56±0.03 &
  0.66±0.03 &
  0.74±0.02 &
  0.65±0.03 &
  0.75±0.01 &
  0.77±0.01 \\
 &
  Density &
  0.0079 &
  0.01±0 &
  0.01±0 &
  0.01±0 &
  0.01±0 &
  0.01±0 &
  0.01±0 \\
 &
  Giant Component Size &
  13 &
  10.37±2.03 &
  10.83±3.2 &
  13.93±3.78 &
  10.35±3.12 &
  13.05±3.86 &
  13.88±3.95 \\
 &
  Num Connected Components &
  62 &
  72.71±1.53 &
  67.23±1.49 &
  63.09±1.09 &
  67.71±1.65 &
  62.69±0.85 &
  61.78±0.7 \\
 &
  Num Edges &
  39 &
  27.91±1.45 &
  32.95±1.43 &
  37.19±1 &
  32.46±1.63 &
  37.52±0.73 &
  38.46±0.56 \\
 &
  Num Nodes &
  100 &
  100±0 &
  100±0 &
  100±0 &
  100±0 &
  100±0 &
  100±0 \\
 &
  Num Triangles &
  0 &
  0.01±0.09 &
  0.1±0.31 &
  0.12±0.34 &
  0.09±0.31 &
  0.1±0.3 &
  0.11±0.32 \\
 &
  PageRank &
  49.5 &
  49.5±0 &
  49.5±0 &
  49.5±0 &
  49.5±0 &
  49.5±0 &
  49.5±0 \\
 &
  Transitivity &
  0 &
  0±0.01 &
  0.01±0.04 &
  0.01±0.03 &
  0.01±0.04 &
  0.01±0.03 &
  0.01±0.03 \\ \midrule
\multirow{11}{*}{\textbf{n=500}}  & Assortativity & -0.009  & -0.09±0.03 & -0.09±0.03 & -0.11±0.03 & -0.05±0.03 & -0.02±0.04 & -0.02±0.03 \\
 &
  Avg ClusterCoefficient &
  0.00575685 &
  0±0 &
  0.01±0 &
  0.01±0 &
  0.01±0 &
  0.01±0 &
  0.01±0 \\
 &
  Avg Degree &
  3.228 &
  3.07±0.02 &
  3.19±0.01 &
  3.19±0.01 &
  3.15±0.01 &
  3.21±0.01 &
  3.22±0 \\
 &
  Density &
  0.0065 &
  0.01±0 &
  0.01±0 &
  0.01±0 &
  0.01±0 &
  0.01±0 &
  0.01±0 \\
 &
  Giant Component Size &
  479 &
  473.05±3 &
  474.67±2.97 &
  476.17±2.19 &
  476.13±2.16 &
  476.77±2.06 &
  476.69±2.09 \\
 &
  Num Connected Components &
  22 &
  26.44±2 &
  24.65±1.57 &
  23.92±1.34 &
  23.9±1.29 &
  23.14±0.99 &
  23.14±1.01 \\
 &
  Num Edges &
  807 &
  766.71±5.09 &
  796.85±2.43 &
  796.3±2.54 &
  787.46±3.64 &
  802.61±1.64 &
  805.05±1.07 \\
 &
  Num Nodes &
  500 &
  500±0 &
  500±0 &
  500±0 &
  500±0 &
  500±0 &
  500±0 \\
 &
  Num Triangles &
  8 &
  4.08±2.03 &
  5.18±2.31 &
  5.23±2.25 &
  4.94±2.15 &
  5.97±2.44 &
  6.02±2.59 \\
 &
  PageRank &
  249.5 &
  249.5±0 &
  249.5±0 &
  249.5±0 &
  249.5±0 &
  249.5±0 &
  249.5±0 \\
 &
  Transitivity &
  0.0094 &
  0.01±0 &
  0.01±0 &
  0.01±0 &
  0.01±0 &
  0.01±0 &
  0.01±0 \\ \midrule
\multirow{11}{*}{\textbf{n=1000}} &
  Assortativity &
  0.0014 &
  -0.06±0.02 &
  -0.03±0.02 &
  -0.06±0.02 &
  0±0.02 &
  0.01±0.02 &
  0.01±0.02 \\
 &
  Avg ClusterCoefficient &
  0.0077 &
  0.01±0 &
  0.01±0 &
  0.01±0 &
  0.01±0 &
  0.01±0 &
  0.01±0 \\
 &
  Avg Degree &
  6.584 &
  6.36±0.02 &
  6.55±0.01 &
  6.58±0 &
  6.53±0.01 &
  6.57±0 &
  6.58±0 \\
 &
  Density &
  0.0066 &
  0.01±0 &
  0.01±0 &
  0.01±0 &
  0.01±0 &
  0.01±0 &
  0.01±0 \\
 &
  Giant Component Size &
  997 &
  996.68±0.63 &
  996.94±0.31 &
  996.99±0.15 &
  996.97±0.19 &
  996.96±0.28 &
  996.98±0.2 \\
 &
  Num Connected Components &
  4 &
  4.3±0.57 &
  4.04±0.19 &
  4.01±0.1 &
  4.03±0.19 &
  4.02±0.15 &
  4.01±0.11 \\
 &
  Num Edges &
  3292 &
  3179.28±10.22 &
  3276.36±3.69 &
  3288.4±1.53 &
  3262.62±4.73 &
  3285.76±2.2 &
  3289.42±1.38 \\
 &
  Num Nodes &
  1000 &
  1000±0 &
  1000±0 &
  1000±0 &
  1000±0 &
  1000±0 &
  1000±0 \\
 &
  Num Triangles &
  60 &
  47.84±7.03 &
  49.11±7.17 &
  50.94±7.08 &
  50.74±6.88 &
  55.12±6.99 &
  55.54±7.53 \\
 &
  PageRank &
  499.5 &
  499.5±0 &
  499.5±0 &
  499.5±0 &
  499.5±0 &
  499.5±0 &
  499.5±0 \\
 &
  Transitivity &
  0.0082 &
  0.01±0 &
  0.01±0 &
  0.01±0 &
  0.01±0 &
  0.01±0 &
  0.01±0 \\ \bottomrule
\end{tabular}%
}
\caption{Scenario 3 maintains a structured graph ($W(u,v)= 1 - |u - v|$) and smooth feature variation ($\mathcal{N}(5u, 0.01)$), leading to the success of both methods.}
\label{tab:synthetic-graph-stat-s3}
\end{table}

\begin{table}[H]
\centering
\renewcommand{\arraystretch}{1.4} 
\resizebox{\textwidth}{!}{%
\begin{tabular}{@{}ccccccccc@{}}
\toprule
 &
  \textbf{} &
  \textbf{} &
  \multicolumn{3}{c}{\textbf{graph-kNN}} &
  \multicolumn{3}{c}{\textbf{feature-kNN}} \\ \midrule
 &
  \textbf{} &
  \textbf{Original} &
  \textbf{k = 5} &
  \textbf{k=20} &
  \textbf{k=50} &
  \textbf{k = 5} &
  \textbf{k=20} &
  \textbf{k=50} \\ \midrule
\multirow{11}{*}{\textbf{n= 100}} &
  Assortativity &\cellcolor[HTML]{A6A6A6}
   &\cellcolor[HTML]{A6A6A6}
   &\cellcolor[HTML]{A6A6A6}
   &\cellcolor[HTML]{A6A6A6}
   &
  0±0 &
  0±0 &
  0±0 \\
 &
  Avg ClusterCoefficient &
  0 &
  0±0 &
  0±0 &
  0±0 &
  0±0 &
  0±0 &
  0±0 \\
 &
  Avg Degree &
  0.02 &
  0±0 &
  0±0 &
  0±0 &
  0±0 &
  0±0 &
  0±0 \\
 &
  Density &
  0.00020202 &
  0±0 &
  0±0 &
  0±0 &
  0±0 &
  0±0 &
  0±0 \\
 &
  Giant Component Size &
  2 &
  1±0 &
  1±0 &
  1±0 &
  1±0 &
  1±0 &
  1±0 \\
 &
  Num Connected Components &
  99 &
  100±0 &
  100±0 &
  100±0 &
  100±0 &
  100±0 &
  100±0 \\
 &
  Num Edges &
  1 &
  0±0 &
  0±0 &
  0±0 &
  0±0 &
  0±0 &
  0±0 \\
 &
  Num Nodes &
  100 &
  100±0 &
  100±0 &
  100±0 &
  100±0 &
  100±0 &
  100±0 \\
 &
  Num Triangles &
  0 &
  0±0 &
  0±0 &
  0±0 &
  0±0 &
  0±0 &
  0±0 \\
 &
  PageRank &
  49.5 &
  49.5±0 &
  49.5±0 &
  49.5±0 &
  49.5±0 &
  49.5±0 &
  49.5±0 \\
 &
  Transitivity &
  0 &
  0±0 &
  0±0 &
  0±0 &
  0±0 &
  0±0 &
  0±0 \\ \midrule
\multirow{11}{*}{\textbf{n=500}} &
  Assortativity &
  -0.2 &
  0±0 &
  -0.34±0.19 &
  0±0 &
  0±0 &
  0±0 &
  -0.15±0.1 \\
 &
  Avg ClusterCoefficient &
  0 &
  0±0 &
  0±0 &
  0±0 &
  0±0 &
  0±0 &
  0±0 \\
 &
  Avg Degree &
  0.12 &
  0.02±0.01 &
  0.03±0.01 &
  0.04±0.01 &
  0.04±0.01 &
  0.08±0.01 &
  0.11±0 \\
 &
  Density &
  0.00024048 &
  0±0 &
  0±0 &
  0±0 &
  0±0 &
  0±0 &
  0±0 \\
 &
  Giant Component Size &
  3 &
  2.89±0.31 &
  3.11±0.32 &
  3.07±0.26 &
  2.8±0.4 &
  3±0.04 &
  3.16±0.37 \\
 &
  Num Connected Components &
  470 &
  494.14±1.63 &
  491.73±1.63 &
  489.81±1.48 &
  489.6±1.66 &
  480.53±1.37 &
  472.85±1.03 \\
 &
  Num Edges &
  30 &
  5.86±1.63 &
  8.27±1.63 &
  10.19±1.48 &
  10.4±1.66 &
  19.47±1.37 &
  27.15±1.03 \\
 &
  Num Nodes &
  500 &
  500±0 &
  500±0 &
  500±0 &
  500±0 &
  500±0 &
  500±0 \\
 &
  Num Triangles &
  0 &
  0±0 &
  0±0 &
  0±0 &
  0±0 &
  0±0 &
  0±0 \\
 &
  PageRank &
  249.5 &
  249.5±0 &
  249.5±0 &
  249.5±0 &
  249.5±0 &
  249.5±0 &
  249.5±0 \\
 &
  Transitivity &
  0 &
  0±0 &
  0±0 &
  0±0 &
  0±0 &
  0±0 &
  0±0 \\ \midrule
\multirow{11}{*}{\textbf{n=1000}} &
  Assortativity &
  0.14150943 &
  -0.15±0.16 &
  -0.08±0.15 &
  -0.06±0.13 &
  -0.09±0.08 &
  0±0.1 &
  0±0.1 \\
 &
  Avg ClusterCoefficient &
  0 &
  0±0 &
  0±0 &
  0±0 &
  0±0 &
  0±0 &
  0±0 \\
 &
  Avg Degree &
  0.208 &
  0.05±0.01 &
  0.05±0.01 &
  0.06±0.01 &
  0.09±0.01 &
  0.17±0 &
  0.19±0 \\
 &
  Density &
  0.00020821 &
  0±0 &
  0±0 &
  0±0 &
  0±0 &
  0±0 &
  0±0 \\
 &
  Giant Component Size &
  7 &
  6.31±0.8 &
  6.18±0.86 &
  6.21±0.86 &
  3.39±0.6 &
  4.58±0.75 &
  5.22±0.98 \\
 &
  Num Connected Components &
  896 &
  976.5±2.95 &
  975.11±2.75 &
  970.02±2.58 &
  953.18±3.15 &
  916.91±2.48 &
  904±1.75 \\
 &
  Num Edges &
  104 &
  23.71±2.94 &
  25.15±2.76 &
  30.18±2.55 &
  46.82±3.15 &
  83.11±2.48 &
  96.01±1.74 \\
 &
  Num Nodes &
  1000 &
  1000±0 &
  1000±0 &
  1000±0 &
  1000±0 &
  1000±0 &
  1000±0 \\
 &
  Num Triangles &
  0 &
  0±0 &
  0±0 &
  0±0 &
  0±0 &
  0.02±0.13 &
  0.01±0.09 \\
 &
  PageRank &
  499.5 &
  499.5±0 &
  499.5±0 &
  499.5±0 &
  499.5±0 &
  499.5±0 &
  499.5±0 \\
 &
  Transitivity &
  0 &
  0±0 &
  0±0 &
  0±0 &
  0±0 &
  0±0.03 &
  0±0.01 \\ \bottomrule
\end{tabular}%
}
\caption{Scenario 4  combines a fragmented graph ($W(u,v) = \mathds{1}\{|u - v| < 0.01\}$) with oscillatory features ($\sin(10u) + \mathcal{N}(0, 0.1)$), making the problem hard for both graph- and feature-based kNN. The greyed-out cells indicate values that are unavailable.}
\label{tab:synthetic-graph-stat-s4}
\end{table}

\begin{table}[H]
\centering
\renewcommand{\arraystretch}{1.2}
\resizebox{\textwidth}{!}{%
\begin{tabular}{@{}cccccccccc@{}}
\toprule
\textbf{} &
  \textbf{Statistic} &
  \textbf{Original} &
  \textbf{k=3} &
  \textbf{k=5} &
  \textbf{k=7} &
  \textbf{k=10} &
  \textbf{k=15} &
  \textbf{k=20} &
  \textbf{k=50} \\ \midrule
\multirow{11}{*}{Cora} &
  Number of Nodes &
  2708 &
  2708±0 &
  2708±0 &
  2708±0 &
  2708±0 &
  2708±0 &
  2708±0 &
  2708±0 \\
 &
  Number of Edges &
  5278 &
  4793.52±14.73 &
  4962.71±13.23 &
  5035.19±12.52 &
  5087.69±10.33 &
  5154.92±9.11 &
  5171.78±7.34 &
  5196.15±7.42 \\
 &
  Average Degree &
  3.90 &
  3.54±0.01 &
  3.67±0.01 &
  3.72±0.01 &
  3.76±0.01 &
  3.81±0.01 &
  3.82±0.01 &
  3.84±0.01 \\
 &
  Density &
  0.00 &
  0±0 &
  0±0 &
  0±0 &
  0±0 &
  0±0 &
  0±0 &
  0±0 \\
 &
  Avg Clustering Coefficient &
  0.24 &
  0.1±0.01 &
  0.09±0 &
  0.08±0 &
  0.07±0 &
  0.06±0 &
  0.05±0 &
  0.03±0 \\
 &
  Avg Connected Component &
  78 &
  136.32±6.99 &
  114.71±6.76 &
  97.29±7.3 &
  97.52±7 &
  62.44±6.95 &
  67.91±6.7 &
  42.54±6.08 \\
 &
  Giant Component Size &
  2485 &
  2479.92±19.9 &
  2530.81±21.35 &
  2571.64±18.39 &
  2580.51±14.08 &
  2625.44±11.64 &
  2620.38±10.78 &
  2652.33±9.41 \\
 &
  Assortativity &
  -0.07 &
  -0.09±0 &
  -0.09±0 &
  -0.08±0 &
  -0.08±0 &
  -0.08±0 &
  -0.07±0 &
  -0.08±0 \\
 &
  PageRank &
  1353 &
  1353.5±0 &
  1353.5±0 &
  1353.5±0 &
  1353.5±0 &
  1353.5±0 &
  1353.5±0 &
  1353.5±0 \\
 &
  Transitivity &
  0.09 &
  0.04±0 &
  0.04±0 &
  0.04±0 &
  0.03±0 &
  0.03±0 &
  0.03±0 &
  0.02±0 \\
 &
  Number of Triangles &
  1630 &
  716.03±32.22 &
  664.18±29.42 &
  623.48±29.33 &
  575.85±29.4 &
  512.85±27.32 &
  471.48±27.11 &
  319.15±21.65 \\ \midrule
\multirow{11}{*}{Citeseer} &
  Number of Nodes &
  3327 &
  3327±0 &
  3327±0 &
  3327±0 &
  3327±0 &
  3327±0 &
  3327±0 &
  3327±0 \\
 &
  Number of Edges &
  4552 &
  3846.33±15.1 &
  3951.45±13.58 &
  3997.51±13.36 &
  4031.65±12.6 &
  4059.15±11.71 &
  4127.78±10.93 &
  4150.58±11.39 \\
 &
  Average Degree &
  2.74 &
  2.31±0.01 &
  2.38±0.01 &
  2.4±0.01 &
  2.42±0.01 &
  2.44±0.01 &
  2.48±0.01 &
  2.5±0.01 \\
 &
  Density &
  0.00 &
  0±0 &
  0±0 &
  0±0 &
  0±0 &
  0±0 &
  0±0 &
  0±0 \\
 &
  Avg Clustering Coefficient &
  0.14 &
  0.05±0 &
  0.05±0 &
  0.04±0 &
  0.04±0 &
  0.04±0 &
  0.03±0 &
  0.03±0 \\
 &
  Avg Connected Component &
  438 &
  868.13±12.49 &
  848.01±10.31 &
  840.08±11.44 &
  830.07±10.87 &
  795.53±10.75 &
  635.09±11.87 &
  570.92±13.67 \\
 &
  Giant Component Size &
  2120 &
  1934.25±42.19 &
  2010.02±32.77 &
  2043.56±31.79 &
  2063.28±31.07 &
  2098.62±36.32 &
  2418.12±35.69 &
  2585.67±23.96 \\
 &
  Assortativity &
  0.05 &
  -0.02±0.01 &
  -0.01±0.01 &
  -0.01±0.01 &
  0±0.01 &
  0.01±0.01 &
  -0.08±0 &
  -0.1±0 \\
 &
  PageRank &
  1663 &
  1663±0 &
  1663±0 &
  1663±0 &
  1663±0 &
  1663±0 &
  1663±0 &
  1663±0 \\
 &
  Transitivity &
  0.13 &
  0.07±0 &
  0.07±0 &
  0.06±0 &
  0.05±0 &
  0.05±0 &
  0.04±0 &
  0.03±0 \\
 &
  Number of Triangles &
  1167 &
  574.04±29.17 &
  550.58±28.97 &
  520.69±27.43 &
  462.24±27.24 &
  431.58±25.85 &
  304.6±19.59 &
  227.62±16.37 \\ \midrule
\multirow{11}{*}{ChicagoSketch} &
  Number of Nodes &
  2176 &
  2176±0 &
  2176±0 &
  2176±0 &
  2176±0 &
  2176±0 &
  2176±0 &
  2176±0 \\
 &
  Number of Edges &
  15104 &
  14505.33±22.7 &
  14811.91±17.53 &
  14914.4±14.51 &
  14956.5±12.89 &
  14999.55±11.49 &
  15014.23±10.28 &
  15066.34±5.57 \\
 &
  Average Degree &
  13.88 &
  13.33±0.02 &
  13.61±0.02 &
  13.71±0.01 &
  13.75±0.01 &
  13.79±0.01 &
  13.8±0.01 &
  13.85±0.01 \\
 &
  Density &
  0.01 &
  0.01±0 &
  0.01±0 &
  0.01±0 &
  0.01±0 &
  0.01±0 &
  0.01±0 &
  0.01±0 \\
 &
  Avg Clustering Coefficient &
  0.57 &
  0.19±0 &
  0.17±0 &
  0.16±0 &
  0.14±0 &
  0.13±0 &
  0.12±0 &
  0.08±0 \\
 &
  Avg Connected Component &
  1 &
  1±0 &
  1±0 &
  1±0 &
  1±0 &
  1±0 &
  1±0 &
  1±0 \\
 &
  Giant Component Size &
  2176 &
  2176±0 &
  2176±0 &
  2176±0 &
  2176±0 &
  2176±0 &
  2176±0 &
  2176±0 \\
 &
  Assortativity &
  0.65 &
  0.31±0.01 &
  0.3±0.01 &
  0.3±0.01 &
  0.31±0.01 &
  0.31±0.01 &
  0.29±0.01 &
  0.13±0.01 \\
 &
  PageRank &
  1087 &
  1087.5±0 &
  1087.5±0 &
  1087.5±0 &
  1087.5±0 &
  1087.5±0 &
  1087.5±0 &
  1087.5±0 \\
 &
  Transitivity &
  0.56 &
  0.19±0 &
  0.17±0 &
  0.16±0 &
  0.14±0 &
  0.13±0 &
  0.12±0 &
  0.08±0 \\
 &
  Number of Triangles &
  38240 &
  11919.95±119.57 &
  11097.95±105.45 &
  10402.11±108.83 &
  9577.65±106.85 &
  8637.74±102.45 &
  8006.58±95.13 &
  5592.24±76.04 \\ \midrule
\multirow{11}{*}{Twitch\_PTBR} &
  Number of Nodes &
  1912 &
  1912±0 &
  1912±0 &
  1912±0 &
  1912±0 &
  1912±0 &
  1912±0 &
  1912±0 \\
 &
  Number of Edges &
  31299 &
  30803.01±27.13 &
  30985.09±25.09 &
  31049.84±24.48 &
  31076.91±23.3 &
  31084.55±23.33 &
  31082.05±22.83 &
  31058.76±25.23 \\
 &
  Average Degree &
  32.74 &
  32.22±0.03 &
  32.41±0.03 &
  32.48±0.03 &
  32.51±0.02 &
  32.52±0.02 &
  32.51±0.02 &
  32.49±0.03 \\
 &
  Density &
  0.02 &
  0.02±0 &
  0.02±0 &
  0.02±0 &
  0.02±0 &
  0.02±0 &
  0.02±0 &
  0.02±0 \\
 &
  Avg Clustering Coefficient &
  0.32 &
  0.17±0 &
  0.17±0 &
  0.17±0 &
  0.17±0 &
  0.17±0 &
  0.17±0 &
  0.17±0 \\
 &
  Avg Connected Component &
  1.00 &
  1.33±0.55 &
  1.17±0.4 &
  1.2±0.44 &
  1.26±0.5 &
  1.19±0.44 &
  1.14±0.39 &
  1.07±0.25 \\
 &
  Giant Component Size &
  1912 &
  1911.31±1.19 &
  1911.64±0.84 &
  1911.59±0.9 &
  1911.47±1.01 &
  1911.61±0.88 &
  1911.72±0.77 &
  1911.87±0.5 \\
 &
  Assortativity &
  -0.23 &
  -0.31±0 &
  -0.3±0 &
  -0.3±0 &
  -0.3±0 &
  -0.29±0 &
  -0.29±0 &
  -0.28±0 \\
 &
  PageRank &
  955 &
  955.5±0 &
  955.5±0 &
  955.5±0 &
  955.5±0 &
  955.5±0 &
  955.5±0 &
  955.5±0 \\
 &
  Transitivity &
  0.13 &
  0.08±0 &
  0.08±0 &
  0.08±0 &
  0.08±0 &
  0.08±0 &
  0.08±0 &
  0.08±0 \\
 &
  Number of Triangles &
  173510 &
  103368.74±1614.98 &
  102572.23±1741.49 &
  103379.78±1759.64 &
  104007.92±1678.98 &
  105301.1±1862.01 &
  105534.51±1904.54 &
  106230.23±1900.76 \\ \midrule
\multirow{11}{*}{Education} &
  Number of Nodes &
  3234 &
  3234±0 &
  3234±0 &
  3234±0 &
  3234±0 &
  3234±0 &
  3234±0 &
  3234±0 \\
 &
  Number of Edges &
  12717 &
  12449.4±13.31 &
  12567.43±9.6 &
  12593.18±8.39 &
  12606.58±8.34 &
  12616.48±7.7 &
  12615.12±7.79 &
  12608.38±8.42 \\
 &
  Average Degree &
  7.86 &
  7.7±0.01 &
  7.77±0.01 &
  7.79±0.01 &
  7.8±0.01 &
  7.8±0 &
  7.8±0 &
  7.8±0.01 \\
 &
  Density &
  0.00 &
  0±0 &
  0±0 &
  0±0 &
  0±0 &
  0±0 &
  0±0 &
  0±0 \\
 &
  Avg Clustering Coefficient &
  0.43 &
  0.19±0 &
  0.17±0 &
  0.15±0 &
  0.14±0 &
  0.12±0 &
  0.12±0 &
  0.12±0 \\
 &
  Avg Connected Component &
  17.00 &
  2.24±0.5 &
  1.2±0.44 &
  1.13±0.37 &
  1.14±0.36 &
  1.27±0.55 &
  1.64±0.79 &
  4.14±1.56 \\
 &
  Giant Component Size &
  3109 &
  3155.56±2.16 &
  3228.04±20.42 &
  3233.87±0.37 &
  3233.86±0.36 &
  3233.73±0.55 &
  3233.36±0.79 &
  3230.86±1.56 \\
 &
  Assortativity &
  0.14 &
  0.02±0.01 &
  0.02±0.01 &
  0.02±0.01 &
  0.02±0.01 &
  0.02±0.01 &
  0.03±0.01 &
  0.04±0.01 \\
 &
  PageRank &
  1616.50 &
  1616.5±0 &
  1616.5±0 &
  1616.5±0 &
  1616.5±0 &
  1616.5±0 &
  1616.5±0 &
  1616.5±0 \\
 &
  Transitivity &
  0.39 &
  0.19±0 &
  0.16±0 &
  0.15±0 &
  0.13±0 &
  0.12±0 &
  0.12±0 &
  0.11±0 \\
 &
  Number of Triangles &
  6490 &
  5402.03±65.45 &
  4712.47±56.84 &
  4356.54±57.62 &
  3980.89±58.92 &
  3569.63±54.37 &
  3486.89±53.63 &
  3346.53±48.62 \\ \midrule
\multirow{11}{*}{Anaheim} &
  Number of Nodes &
  914 &
  914±0 &
  914±0 &
  914±0 &
  914±0 &
  914±0 &
  914±0 &
  914±0 \\
 &
  Number of Edges &
  3881 &
  3557.23±17.18 &
  3744.01±10.86 &
  3804.45±8.77 &
  3845.48±5.53 &
  3855.53±4.2 &
  3858.74±4.21 &
  3855.93±4.51 \\
 &
  Average Degree &
  8.49 &
  7.78±0.04 &
  8.19±0.02 &
  8.32±0.02 &
  8.41±0.01 &
  8.44±0.01 &
  8.44±0.01 &
  8.44±0.01 \\
 &
  Density &
  0.01 &
  0.01±0 &
  0.01±0 &
  0.01±0 &
  0.01±0 &
  0.01±0 &
  0.01±0 &
  0.01±0 \\
 &
  Avg Clustering Coefficient &
  0.55 &
  0.19±0.01 &
  0.16±0.01 &
  0.14±0 &
  0.12±0 &
  0.11±0 &
  0.09±0 &
  0.05±0 \\
 &
  Avg Connected Component &
  1.00 &
  1.17±0.42 &
  1.1±0.33 &
  1.09±0.29 &
  1.05±0.22 &
  1.02±0.15 &
  1.01±0.1 &
  1.01±0.08 \\
 &
  Giant Component Size &
  914 &
  913.82±0.44 &
  913.89±0.35 &
  913.9±0.36 &
  913.95±0.25 &
  913.98±0.15 &
  913.99±0.1 &
  913.99±0.08 \\
 &
  Assortativity &
  0.71 &
  0.51±0.01 &
  0.45±0.01 &
  0.41±0.01 &
  0.39±0.01 &
  0.39±0.01 &
  0.39±0.01 &
  0.15±0.01 \\
 &
  PageRank &
  456 &
  456.5±0 &
  456.5±0 &
  456.5±0 &
  456.5±0 &
  456.5±0 &
  456.5±0 &
  456.5±0 \\
 &
  Transitivity &
  0.60 &
  0.2±0 &
  0.16±0 &
  0.14±0 &
  0.12±0 &
  0.11±0 &
  0.1±0 &
  0.06±0 \\
 &
  Number of Triangles &
  7162 &
  1968.36±46.38 &
  1779.38±43.05 &
  1575.97±39.58 &
  1438.32±38.23 &
  1270.53±35.34 &
  1163.42±32.96 &
  768.93±28.73 \\ \bottomrule
\end{tabular}%
}
\caption{Graph statistics of bootstrapped samples generated by Algorithm~\ref{alg:graph} with varying neighborhood size ($k$). }
\label{tab:graph-stat-real-data}
\end{table}

\begin{table}[H]
\renewcommand{\arraystretch}{1.2} %
\centering
\resizebox{\textwidth}{!}{%
\begin{tabular}{@{}cccccccc@{}}
\toprule
\textbf{} &
  \textbf{} &
  \textbf{Original} &
  \textbf{Edge Drop} &
  \textbf{Node Drop} &
  \textbf{Ours} &
  \textbf{Network Bootstrap} &
  \textbf{VAE} \\ \midrule
 &
  Assortativity &
  -0.07 &
  -0.07 ± 0 &
  -0.07 ± 0.01 &
  \textbf{-0.07±0} &
  {\ul -0.03 ± 0.03} &
  -0.43 ± 0 \\
 &
  Avg Clustering Coefficient &
  0.24 &
  0.22 ± 0 &
  0.17 ± 0.01 &
  \textbf{0.05±0} &
  {\ul 0.02 ± 0.01} &
  0.5 ± 0 \\
 &
  Avg Degree &
  3.90 &
  3.74 ± 0.01 &
  2.63 ± 0.07 &
  \textbf{3.82±0.01} &
  {\ul 2.11 ± 0.18} &
  10.04 ± 0.05 \\
 &
  Density &
  0.00 &
  0 ± 0 &
  0 ± 0 &
  \textbf{0±0} &
  {\ul 0 ± 0} &
  0 ± 0 \\
 &
  Giant Component Size &
  2485 &
  2457.77 ± 9.09 &
  1819.98 ± 30.29 &
  \textbf{2620.38±10.78} &
  1090.22 ± 41.05 &
  {\ul 1931.32 ± 9.06} \\
 &
  Num Connected Components &
  78 &
  100.12 ± 4.7 &
  587.49 ± 16.74 &
  \textbf{67.91±6.7} &
  1595.93 ± 38.78 &
  {\ul 777.68 ± 9.06} \\
 &
  Num Edges &
  5278 &
  5066.1 ± 11.78 &
  3382.72 ± 87.72 &
  \textbf{5171.78±7.34} &
  {\ul 2857.55 ± 244.83} &
  13598.31 ± 67.26 \\
 &
  Num Nodes &
  2708 &
  2708 ± 0 &
  2569.6 ± 9.32 &
  2708±0 &
  2708 ± 0 &
  2708 ± 0 \\
 &
  Num Triangles &
  1630.00 &
  1441.03 ± 20.11 &
  835.13 ± 63.84 &
  {\ul 471.48±27.11} &
  \textbf{801.19 ± 338.24} &
  78345.91 ± 643.03 \\
 &
  Pagerank &
  1353.50 &
  1353.5 ± 0 &
  1294.43 ± 4.23 &
  1353.5±0 &
  1353.5 ± 0 &
  1353.5 ± 0 \\
\multirow{-11}{*}{Cora} &
  Transitivity &
  0.09 &
  0.09 ± 0 &
  0.09 ± 0.01 &
  0.03±0 &
  {\ul 0.07 ± 0.01} &
  \textbf{0.12 ± 0} \\ \midrule
 &
  Assortativity &
  -0.23 &
  -0.23±0 &
  -0.23±0.01 &
  \textbf{-0.29±0} &
  0±0 &
  {\ul -0.45±0} \\
 &
  Avg Clustering Coefficient &
  0.32 &
  0.25±0 &
  0.26±0.01 &
  {\ul 0.17±0} &
  0±0 &
  \textbf{0.41±0} \\
 &
  Avg Degree &
  32.74 &
  26.19±0 &
  21.87±0.84 &
  \textbf{32.51±0.02} &
  0±0 &
  {\ul 19.89±0.04} \\
 &
  Density &
  0.02 &
  0.01±0 &
  0.01±0 &
  \textbf{0.02±0} &
  0±0 &
  {\ul 0.01±0} \\
 &
  Giant Component Size &
  1912 &
  1883.96±4.76 &
  1506±5.9 &
  \textbf{1911.72±0.77} &
  1±0 &
  {\ul 893.77±7.03} \\
 &
  Num Connected Components &
  1.00 &
  28.12±4.49 &
  325.55±8.23 &
  \textbf{1.14±0.39} &
  1912±0 &
  1016.37±6.87 \\
 &
  Num Edges &
  31299 &
  25039±0 &
  20029.75±776.17 &
  \textbf{31082.05±22.83} &
  0±0 &
  {\ul 19010.87±40.09} \\
 &
  Num Nodes &
  1912 &
  1912±0 &
  1831.36±7.18 &
  1912±0 &
  1912±0 &
  1912±0 \\
 &
  Num Triangles &
  173510 &
  88756.59±980.58 &
  89048.18±9052.57 &
  \textbf{105534.51±1904.54} &
  {\ul 0±0} &
  332319.47±623.05 \\
 &
  Pagerank &
  955.50 &
  955.5±0 &
  921.8±3.18 &
  955.5±0 &
  955.5±0 &
  955.5±0 \\
\multirow{-11}{*}{TwitchPTBR} &
  Transitivity &
  0.13 &
  0.1±0 &
  0.13±0 &
  \textbf{0.08±0} &
  0±0 &
  {\ul 0.31±0} \\  \midrule
 &
  Assortativity &
  0.65 &
  456.5 ± 0 &
  441.21 ± 2.15 &
  \textbf{0.29±0.01} &
  456.5 ± 0 &
  456.5 ± 0 \\
 &
  Avg Clustering Coefficient &
  0.57 &
  0 ± 0 &
  0 ± 0 &
  \textbf{0.12±0} &
  0 ± 0 &
  0 ± 0 \\
 &
  Avg Degree &
  13.88 &
  0.48 ± 0 &
  0.6 ± 0.01 &
  \textbf{13.8±0.01} &
  0 ± 0 &
  {\ul 0.59 ± 0.01} \\
 &
  Density &
  0.01 &
  0 ± 0 &
  0 ± 0 &
  \textbf{0.01±0} &
  {\ul 0 ± 0} &
  {\ul 0 ± 0} \\
 &
  Giant Component Size &
  2176 &
  2175.97 ± 0.18 &
  1739.89 ± 0.44 &
  \textbf{2176±0} &
  1 ± 0 &
  {\ul 703.59 ± 6.77} \\
 &
  Num Connected Components &
  1.00 &
  1.03 ± 0.16 &
  349.83 ± 7.54 &
  \textbf{1±0} &
  2176 ± 0 &
  {\ul 1473.41 ± 6.77} \\
 &
  Num Edges &
  15104 &
  12083 ± 0 &
  9658.73 ± 46.97 &
  \textbf{15014.23±10.28} &
  0 ± 0 &
  {\ul 9156.6 ± 85.77} \\
 &
  Num Nodes &
  2176 &
  2176 ± 0 &
  2088.77 ± 7.54 &
  2176±0 &
  2176 ± 0 &
  2176 ± 0 \\
 &
  Num Triangles &
  38240 &
  19576.4 ± 101.57 &
  19554 ± 325.13 &
  \textbf{8006.58±95.13} &
  0 ± 0 &
  {\ul 83059.43 ± 1685.2} \\
 &
  Pagerank &
  1087.50 &
  1087.5 ± 0 &
  1051.19 ± 3.29 &
  1087.5±0 &
  1087.5 ± 0 &
  1087.5 ± 0 \\
\multirow{-11}{*}{ChicagoSketch} &
  Transitivity &
  0.56 &
  0.45 ± 0 &
  0.56 ± 0 &
  {\ul 0.12±0} &
  0 ± 0 &
  \textbf{0.51 ± 0} \\ \midrule
 &
  Assortativity &
  0.14 &
  0.11 ± 0.01 &
  0.33 ± 0.02 &
  \textbf{0.03±0.01} &
  {\ul 0 ± 0} &
  -0.24 ± 0 \\
 &
  Avg Clustering Coefficient &
  0.43 &
  0.34 ± 0 &
  0.36 ± 0 &
  {\ul 0.12±0} &
  0 ± 0 &
  \textbf{0.37 ± 0} \\
 &
  Avg Degree &
  7.86 &
  6.29 ± 0 &
  5.58 ± 0.02 &
  \textbf{7.8±0} &
  {\ul 0 ± 0} &
  16.63 ± 0.07 \\
 &
  Density &
  0.00 &
  0 ± 0 &
  0 ± 0 &
  \textbf{0±0} &
  0 ± 0 &
  {\ul 0.01 ± 0} \\
 &
  Giant Component Size &
  3109 &
  3103.63 ± 2.38 &
  2478.84 ± 8.39 &
  \textbf{3233.36±0.79} &
  2.34 ± 0.94 &
  {\ul 1826.67 ± 7.86} \\
 &
  Num Connected Components &
  17 &
  25.11 ± 2.6 &
  539.85 ± 8.69 &
  \textbf{1.64±0.79} &
  3231.81 ± 1.71 &
  {\ul 1408.33 ± 7.86} \\
 &
  Num Edges &
  12717 &
  10173 ± 0 &
  8654.72 ± 30.07 &
  \textbf{12615.12±7.79} &
  {\ul 2.94 ± 2.91} &
  26887.73 ± 118.84 \\
 &
  Num Nodes &
  3234 &
  3234 ± 0 &
  3104.69 ± 8.44 &
  3234±0 &
  3234 ± 0 &
  3234 ± 0 \\
 &
  Num Triangles &
  6490 &
  3321.39 ± 32.58 &
  3321.02 ± 37.46 &
  \textbf{3486.89±53.63} &
  {\ul 0.96 ± 2.18} &
  257887.78 ± 2061.14 \\
 &
  Pagerank &
  1616.50 &
  1616.5 ± 0 &
  1562.58 ± 3.73 &
  1616.5±0 &
  1616.5 ± 0 &
  1616.5 ± 0 \\
\multirow{-11}{*}{Eduation} &
  Transitivity &
  0.39 &
  0.32 ± 0 &
  0.39 ± 0 &
  0.12±0 &
  \textbf{0.36 ± 0.48} &
  {\ul 0.35 ± 0} \\ \midrule
 &
  Assortativity &
  0.71 &
  0.6 ± 0.01 &
  0.69 ± 0.02 &
  \textbf{0.39±0.01} &
  \cellcolor[HTML]{A6A6A6} &
  -0.15 ± 0.02 \\
 &
 Avg Clustering Coefficient &
  0.55 &
  0.44 ± 0.01 &
  0.46 ± 0.01 &
  {\ul 0.09±0} &
  \cellcolor[HTML]{A6A6A6} &
  \textbf{0.22 ± 0} \\
 &
Avg Degree &
  8.49 &
  6.79 ± 0 &
  5.66 ± 0.08 &
  \textbf{8.44±0.01} &
  \cellcolor[HTML]{A6A6A6} &
  5.06 ± 0.07 \\
 &
  Density &
  0.01 &
  0 ± 0 &
  0 ± 0 &
  \textbf{0.01±0} &
  \cellcolor[HTML]{A6A6A6} &
  0 ± 0 \\
 &
Giant Component Size &
  914 &
  0.01 ± 0 &
  0.01 ± 0 &
  \textbf{913.99±0.1} &
  \cellcolor[HTML]{A6A6A6} &
  0.01 ± 0 \\
 &
Num Connected Components &
  1.00 &
  0 ± 0 &
  0 ± 0 &
  \textbf{1.01±0.1} &
  \cellcolor[HTML]{A6A6A6} &
  0 ± 0 \\
 &
Num Edges &
  3881 &
  911.49 ± 2.43 &
  727.17 ± 3.63 &
  \textbf{3858.74±4.21} &
  \cellcolor[HTML]{A6A6A6} &
  121.87 ± 13.67 \\
 &
Num Nodes &
  914 &
  2.8 ± 1.45 &
  149.32 ± 5.24 &
  \textbf{914±0} &
  \cellcolor[HTML]{A6A6A6} &
  645.81 ± 5.49 \\
 &
Num Triangles &
  7162 &
  3104 ± 0 &
  2483.09 ± 37.24 &
  \textbf{1163.42±32.96} &
  \cellcolor[HTML]{A6A6A6} &
  2311.09 ± 32.23 \\
 &
Pagerank&
  456.50 &
  914 ± 0 &
  877.31 ± 4.94 &
  \textbf{456.5±0} &
  \cellcolor[HTML]{A6A6A6} &
  914 ± 0 \\
\multirow{-11}{*}{Anaheim} &
  transitivity &
  0.60 &
  3661.41 ± 51.07 &
  3667.47 ± 152.92 &
  \textbf{0.1±0} &
  \cellcolor[HTML]{A6A6A6} &
  14770.14 ± 386.77 \\ \bottomrule
\end{tabular}%
}
\caption{Graph statistics of bootstrapped samples generated by Algorithm~\ref{alg:graph} with $k = 20 $, simple generation with Edge or Node Drop (with $p = 0.2$), and network bootstrap (NB) \citep{levin2021bootstrapping} and the extended VAE approach in Appendix-\ref{sec:split-more}. The best values among our method, NB and VAE are bolded, and the second best values are underlined. The simple split methods largely distort the original connectivity structure, reflected in the Giant Component Size or the Number of Connected Components. Our proposed local bootstrap consistently mimics the original graph in terms of the reported graph statistics. The greyed-out cells indicate values that are unavailable due to instability encountered during training.}
\label{tab:graph-stat-real-data-compare-all}
\end{table}

\begin{remark}
When the graph structure is sufficiently informative for the underlying latent variable distribution (see the differences in Table~\ref{tab:synthetic-graph-stat-s1} and Table~\ref{tab:synthetic-graph-stat-s3} for the graph-kNN case), the bootstrapped graphs show noticeable robustness to the choice of $k$ (the number of nearest neighbors). For example, with 500 nodes and k=20, the bootstrapped graph retains about 98\% of the original edges and recovers approximately 70\% of its triangles (see Table~\ref{tab:synthetic-graph-stat-s2} for detailed statistics). 
Increasing k produces progressively denser graphs with more edges and higher average degrees, but it also tends to lower the clustering coefficient and reduce the number of triangles. Conversely, using a very small k leads to sparser graphs that may preserve more local structure—reflected by higher clustering coefficients and relatively more triangles per edge—but can underrepresent global connectivity, often resulting in many small components. On real datasets, setting $k=20$ typically recovers an edge count close to that of the original graph, though it still underestimates the original triangle count (see Table~\ref{tab:graph-stat-real-data}). Nonetheless, our proposed method produces graphs whose triangle counts more closely match the original than those generated by other methods (see Table~\ref{tab:graph-stat-real-data-compare-all}).
\end{remark}

\subsection{Validation of the Entire Framework (Algorithm~\ref{alg:full})}\label{sec:validation-all-resuts}

\begin{figure}[H]
    \centering
   \begin{subfigure}{\textwidth}
        \includegraphics[width = \linewidth]{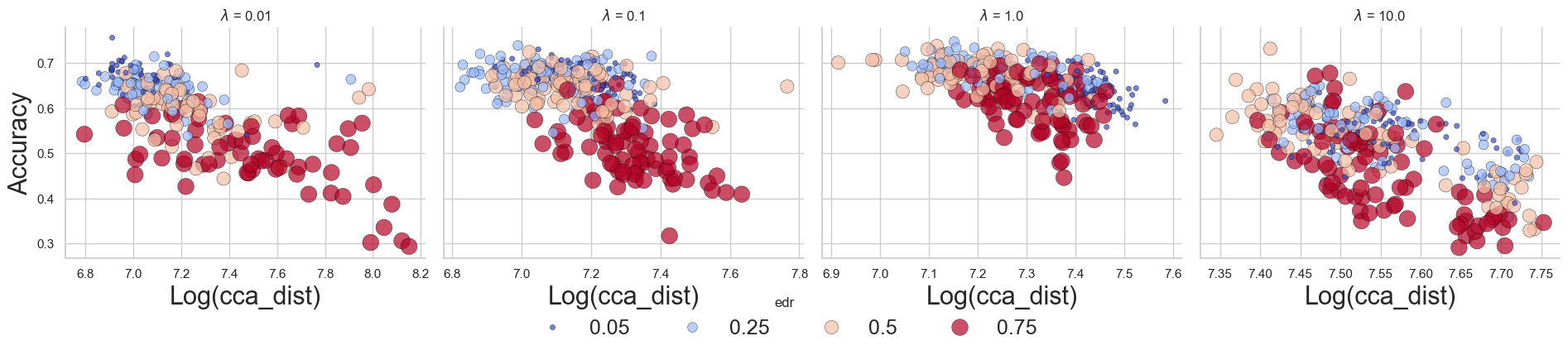}
        \caption{Our CCA-based metrics discussed in Section~\ref{sec:cv-eval}.}
   \end{subfigure}
   \begin{subfigure}{\textwidth}
        \includegraphics[width = \linewidth]{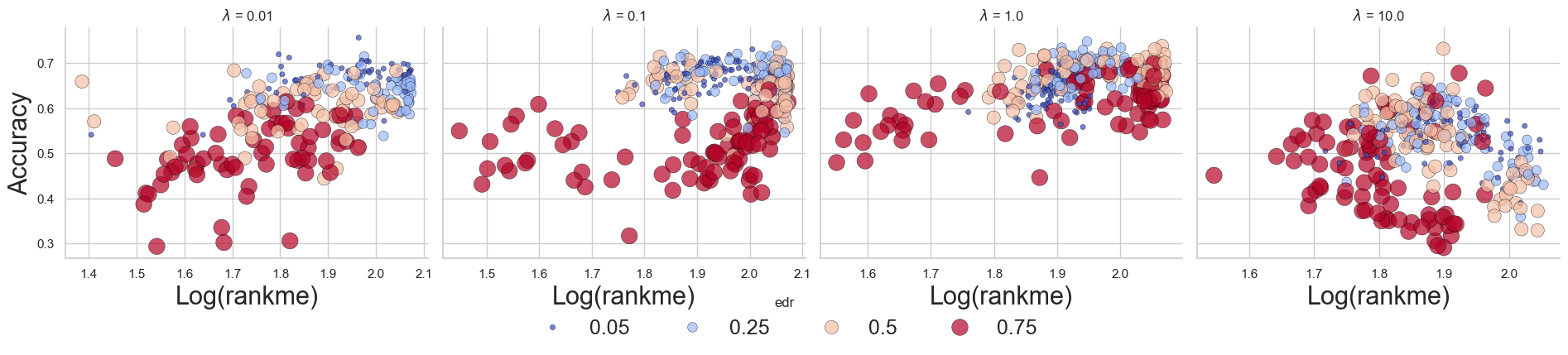}
        \caption{RankMe metric by \citet{garrido2023rankmeassessingdownstreamperformance}}
   \end{subfigure}
   \begin{subfigure}{\textwidth}
        \includegraphics[width = \linewidth]{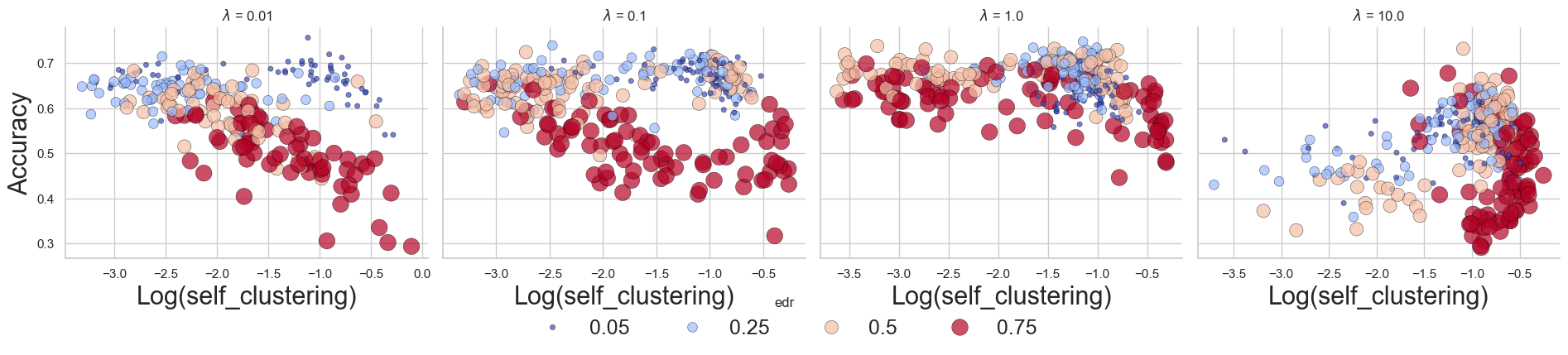}
        \caption{SelfCluster by \citet{tsitsulin2023unsupervised}}
   \end{subfigure}
   \begin{subfigure}{\textwidth}
        \includegraphics[width = \linewidth]{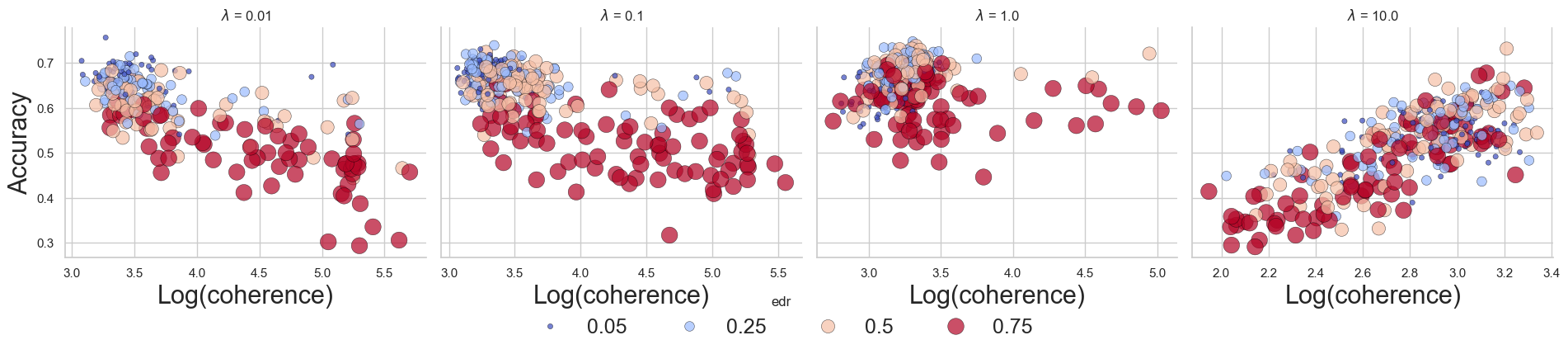}
        \caption{Coherence by \citet{tsitsulin2023unsupervised}}
   \end{subfigure}
   \caption{Visualizations of metrics value and classification accuracy on Cora. The CCA-SSG \citep{zhang2021canonical} model is trained on set of hyperparameter combinations including $\lambda$, edge drop rate (EDR), and feature masking rate (FMR). Each point denotes the each combination of hyperparameters. The color denotes the edge dropping rate with blue dots referring small value ($edr=0.05$) whereas the red dot referring to the large value ($edr=0.75$). In this dataset, a clear negative correlation is noted for our metrics across all combinations of hyperparameters. }
   \label{fig:cora-metrics}
\end{figure}

\begin{table}[H]
\centering
\renewcommand{\arraystretch}{1.4} %
\resizebox{\textwidth}{!}{%
\begin{tabular}{@{}cccccccccc@{}}
\toprule
\textbf{Dataset} &
  \textbf{Default} &
  \textbf{Ours} &
  \textbf{$\alpha$-ReQ} &
  \textbf{pseudo-$\kappa$} &
  \textbf{RankME} &
  \textbf{NESum} &
  \textbf{SelfCluster} &
  \textbf{Stable Rank} &
  \textbf{Coherence} \\ \midrule
Cora      & 0.36          & 0.4           & 0.4           & \textbf{0.57} & {\ul 0.55} & 0.4           & 0.4           & \textbf{0.57} & 0.54          \\
PubMed    & 0.62          & 0.68          & 0.67          & \textbf{0.74} & 0.67       & 0.67          & {\ul 0.69}    & \textbf{0.74} & 0.74          \\
Citeseer &
  0.32 &
  \textbf{0.42} &
  \textbf{0.42} &
  \textbf{0.42} &
  \textbf{0.42} &
  \textbf{0.42} &
  \textbf{0.42} &
  \textbf{0.42} &
  \textbf{0.42} \\
CS        & 0.47          & 0.71          & \textbf{0.82} & {\ul 0.75}    & {\ul 0.75} & \textbf{0.82} & \textbf{0.82} & {\ul 0.75}    & \textbf{0.82} \\
Chicago   & \textbf{0.39} & 0.34          & {\ul 0.35}    & {\ul 0.35}    & {\ul 0.35} & {\ul 0.35}    & {\ul 0.35}    & 0.29          & 0.4           \\
Anaheim   & 0.13          & \textbf{0.23} & 0.12          & {\ul 0.18}    & {\ul 0.18} & 0.12          & \textbf{0.23} & {\ul 0.18}    & 0.12          \\
Education & {\ul 0.23}    & \textbf{0.26} & 0.18          & 0.17          & 0.18       & 0.16          & 0.18          & \textbf{0.26} & 0.21          \\ \midrule
Avg\_clf  & 0.44          & 0.55          & 0.58          & {\ul 0.62}    & 0.60       & 0.58          & 0.58          & {\ul 0.62}    & \textbf{0.63} \\
Avg\_reg  & {\ul 0.25}    & \textbf{0.28} & 0.22          & 0.23          & 0.24       & 0.21          & {\ul 0.25}    & 0.24          & 0.24          \\ \bottomrule
\end{tabular}%
}
\caption{Downstream task (classification or regression) performance of the BGRL with hyperparameters chosen by each criteria. We compare to the BGRL \citep{bgrl} with the default parameters in the left-most column, $\text{fmr} = 0.5 , \text{edr} = 0.25, \lambda = 10^{-2}$. The best value is bolded and the second best is underlined.}
\label{tab:exp-summary-downstream-performance-bgrl}
\end{table}

\begin{table}[H]
\centering
\renewcommand{\arraystretch}{1.4} %
\resizebox{\textwidth}{!}{%
\begin{tabular}{@{}cccccccccc@{}}
\toprule
\textbf{Dataset} &
  \textbf{Default} &
  \textbf{Ours} &
  \textbf{$\alpha$-ReQ} &
  \textbf{pseudo-$\kappa$} &
  \textbf{RankME} &
  \textbf{NESum} &
  \textbf{SelfCluster} &
  \textbf{Stable Rank} &
  \textbf{Coherence} \\ \midrule
Cora      & 0.35       & 0.65          & 0.66          & {\ul 0.67}    & 0.63          & 0.63          & \textbf{0.69} & 0.59          & 0.47       \\
PubMed    & 0.49       & {\ul 0.81}    & \textbf{0.82} & 0.56          & 0.56          & 0.56          & \textbf{0.82} & \textbf{0.82} & 0.76       \\
Citeseer  & 0.38       & {\ul 0.51}    & \textbf{0.53} & {\ul 0.51}    & \textbf{0.53} & \textbf{0.53} & \textbf{0.53} & {\ul 0.51}    & 0.22       \\
CS        & 0.65       & {\ul 0.79}    & \textbf{0.86} & \textbf{0.86} & \textbf{0.86} & \textbf{0.86} & \textbf{0.86} & \textbf{0.86} & 0.76       \\
Photo     & 0.32       & \textbf{0.73} & 0.58          & 0.53          & 0.53          & \textbf{0.73} & \textbf{0.73} & \textbf{0.73} & {\ul 0.69} \\
Computers & 0.42       & 0.57          & \textbf{0.66} & 0.57          & \textbf{0.66} & \textbf{0.66} & \textbf{0.66} & 0.57          & {\ul 0.65} \\
Anaheim &
  {\ul 0.37} &
  \textbf{0.38} &
  \textbf{0.38} &
  \textbf{0.38} &
  \textbf{0.38} &
  \textbf{0.38} &
  \textbf{0.38} &
  \textbf{0.38} &
  \textbf{0.38} \\
Twitch    & 0.47       & \textbf{0.52} & 0.15          & 0.15          & 0.15          & 0.15          & 0.46          & 0.15          & {\ul 0.48} \\
Education & 0.29       & 0.26          & 0.33          & 0.33          & 0.33          & 0.33          & 0.33          & 0.33          & 0.26       \\ \midrule
Avg\_clf  & 0.44       & 0.68          & {\ul 0.69}    & 0.62          & 0.63          & 0.66          & \textbf{0.72} & 0.68          & 0.59       \\
Avg\_reg  & {\ul 0.38} & \textbf{0.39} & 0.29          & 0.29          & 0.29          & 0.29          & \textbf{0.39} & 0.29          & 0.37       \\ \bottomrule
\end{tabular}%
}
\caption{Downstream task (classification or regression) performance of the CCA-SSG with hyperparameters chosen by each criteria. We compare to the CCA-SSG \citep{zhang2021canonical} with the default parameters in the left-most column, $\text{fmr} = 0.5 , \text{edr} = 0.25, \lambda = 10^{-4}$. The best value is bolded and the second best is underlined.}
\label{tab:exp-summary-downstream-performance-cca-ssg}
\end{table}

\begin{table}[H]
\centering
\renewcommand{\arraystretch}{1.4} %
\resizebox{\textwidth}{!}{%
\begin{tabular}{@{}cccccccccc@{}}
\toprule
\textbf{Dataset} &
  \textbf{Default} &
  \textbf{Ours} &
  \textbf{$\alpha$-ReQ} &
  \textbf{pseudo-$\kappa$} &
  \textbf{RankME} &
  \textbf{NESum} &
  \textbf{SelfCluster} &
  \textbf{Stable Rank} &
  \textbf{Coherence} \\ \midrule
Cora &
  {\ul 0.64} &
  \textbf{0.68} &
  0.54 &
  0.54 &
  0.54 &
  0.54 &
  0.5 &
  0.54 &
  0.42 \\
PubMed &
  \cellcolor[HTML]{A6A6A6}&
  {\ul 0.78} &
  0.75 &
  0.75 &
  0.75 &
  0.75 &
  0.75 &
  0.75 &
  \textbf{0.79} \\
Citeseer &
  {\ul 0.53} &
  \textbf{0.55} &
  0.51 &
  0.51 &
  0.51 &
  0.51 &
  0.48 &
  0.51 &
  0.44 \\
CS &
  \cellcolor[HTML]{A6A6A6} &
  \textbf{0.72} &
  \textbf{0.72} &
  \textbf{0.72} &
  \textbf{0.72} &
  \textbf{0.72} &
  \textbf{0.72} &
  \textbf{0.72} &
  0.61 \\
Photo &
  0.71 &
  \textbf{0.83} &
  0.79 &
  0.79 &
  0.79 &
  0.79 &
  0.57 &
  {\ul 0.81} &
  0.41 \\
Computers &
  \cellcolor[HTML]{A6A6A6} &
  \textbf{0.45} &
  \textbf{0.45} &
  \textbf{0.45} &
  \textbf{0.45} &
  {\ul 0.39} &
  {\ul 0.39} &
  {\ul 0.39} &
  {\ul 0.39} \\ \midrule
Avg\_clf &
  0.63 &
  \textbf{0.67} &
  0.63 &
  {\ul 0.63} &
  {\ul 0.63} &
  0.62 &
  0.57 &
  0.62 &
  0.50 \\ \bottomrule
\end{tabular}%
}
\caption{Downstream task (classification or regression) performance of the GRACE with hyperparameters chosen by each criteria. We compare to the GRACE \citep{zhu2020GRACE} with the default parameters in the left-most column, $\text{fmr} = 0.5 , \text{edr} = 0.25, \tau = 1$. The best value is bolded and the second best is underlined. The greyed-out cells indicate values that are unavailable due to instability encountered during training.}
\label{tab:exp-summary-downstream-performance-grace}
\end{table}

\input{stable_rank_ver_corr_table}

\end{document}

%% file: math_commands.tex
\usepackage{amsmath,amsfonts,bm}

\def\eqref#1{equation~\ref{#1}}

\def\1{\bm{1}}

\DeclareMathAlphabet{\mathsfit}{\encodingdefault}{\sfdefault}{m}{sl}
\SetMathAlphabet{\mathsfit}{bold}{\encodingdefault}{\sfdefault}{bx}{n}

\newcommand{\E}{\mathbb{E}}

\newcommand{\R}{\mathbb{R}}

\DeclareMathOperator*{\argmin}{arg\,min}

%% file: gnn-review.tex
\noindent\textbf{CCA-SSG}: CCA-SSG \citep{zhang2021canonical} is inspired by statistical canonical correlation analysis(CCA) that constructs the loss on the feature-level rather than instance-level discrimination, which is typical in contrastive methods. They augment the original graph in a random fashion by dropping edges or masking the node features to make a pair of graphs for learning. 
\begin{equation}\label{eq:cca-ssg-loss}
\mathcal{L} = \underbrace{\|\tilde{Z}_A-\tilde{Z}_B\|^2}_{\text{invariance term}} + \lambda \underbrace{\|\tilde{Z}^\top_A\tilde{Z}_A-I\|_F^2 + \|\tilde{Z}^\top_B\tilde{Z}_B -I\|_F^2}_{\text{decorrelation term}}
\end{equation}

Although their model structure is relatively simple and does not require a parametrized mutual information estimator or additional projection network, they still have the issue of choosing hyperparameters(e.g. $\lambda$) which has a non-negligible impact on the model performance. \\
\par

\noindent\textbf{GRACE:} Contrastive learning or self-supervised method has gotten increasing attention as they do not require label availability as supervised GNN does. Deep Graph Contrastive Representation Learning(GRACE) \citep{zhu2020GRACE} is one of the popular graph constrastive learning methods.
\begin{enumerate}
    \item For each iteration, GRACE generates two graph views, $\tilde{G_1}, \tilde{G_2}$, by either randomly removing edges or randomly masking node features. 
    \item Let $U = f(\tilde{X_1}, \tilde{A_1}), V = f(\tilde{X_2}, \tilde{A_2})$ be the embedded representation of two graph views, and their corresponding node features and adjacency matrices. 
    \item Positive samples: For any node $v_i$, its corresponding representation in another view $u_i$ is treated as natural positive pair. 
    \item Negative samples: For given node $v_i$, any nodes in another view $u_{k \neq i}$ are treated as negative pair. 
    \item Node-wise objective: 
    $$ \ell(u_i, v_i) = \text{log} \frac{ e^{\theta(u_i,v_i)/ \tau}} { \underbrace{e^{\theta(u_i,v_i)/\tau}}_\text{the positive pair} + \underbrace{\sum_{k=1}^N \mathbbm{1}_{k \neq i} e^{\theta (u_i,v_k)/\tau}}_\text{inter-view negative pairs} + \underbrace{\sum_{k=1}^N\mathbbm{1}_{k \neq i} e^{\theta(u_i,u_k)/\tau}}_\text{intra-view negative pairs} } $$
    \item Overall loss function: $\ell = \frac{1}{2N}\sum_{i=1}^N \big[\ell(u_i, v_i) + \ell(v_i, u_i) \big]$
    \item Optimization: apply stochastic gradient descent.
\end{enumerate}


\noindent\textbf{DGI:} Deep Graph Infomax \citep{stokes2020deep} is another option for the unsupervised graph representation learning. DGI optimizes the mutual information between the local patch representation of the graph and the overall high-level summaries.
$$
\mathcal{L} = \frac{1}{N+M}\big( \sum_{i=1}^N\mathbb{E}_{(X,A)}[log \mathcal{D}(\vec{h}_i, \vec{s})] + \sum_{j=1}^M\mathbb{E}_{(X,A)}[log(1-\mathcal{D}(\vec{\tilde{h}}_i, \vec{\tilde{s}}]
\big)
$$

\noindent\textbf{BGRL:} Large-Scale Representation Learning on Graphs via Bootstrapping(BGRL) 
 \citep{thakoor2023largescale}
similar to CCA-SSG, BGRL uses node and feature masking to augment the original graph. At the core of BGRL is a bootstrapping mechanism that updates the target representations gradually, borrowing ideas from consistency regularization and contrastive learning. Unlike contrastive learning methods that require negative samples, BGRL avoids the computational overhead associated with negative sampling by using a bootstrapping approach. This involves maintaining two networks: an online network that is updated using gradients and a target network that is slowly updated with the parameters of the online network. This setup encourages the embeddings to become more stable and consistent over iterations.
\begin{enumerate}
    \item Update the online encoder: 
    $$\ell(\theta, \phi) = -\frac{2}{N}\sum_{i=0}^{N-1}\frac{\tilde{Z}_{(1,i)}\tilde{H}^\top_{(2,i)}}{\|\tilde{Z}_{(1,i)}\| \| \tilde{H}^\top_{(2,i)}\| }
    $$
    \item Update the target encoder: $\theta \leftarrow \tau\phi+(1-\tau)\theta$

\end{enumerate}

\textbf{GCA}: Graph Contrastive Learning with Augmentations (GCA) 
 \citep{you2021graphcontrastivelearningaugmentations} introduces a contrastive learning framework designed specifically for graph data. GCA applies data augmentation techniques on both the node features and graph structure, creating different views of the same node. The central idea is to maximize the agreement between the representations of the same node in different augmented views, while ensuring that the representations of different nodes remain distinguishable.

The contrastive loss is designed to encourage the representations of different views, $a$ and $b$ of the same node $i$, with temperature scaling $\tau$.  
$$\mathcal{L}_{\text{GCA}} = \frac{1}{N} \sum_{i=1}^N -\log \frac{\exp(\text{sim}(\mathbf{z}_i^a, \mathbf{z}_i^b) / \tau)}{\sum_{j=1}^N \exp(\text{sim}(\mathbf{z}_i^a, \mathbf{z}_j^b) / \tau)}$$  where $\text{sim}(z_i, z_j) = z_i^Tz_j/(\|z_i\|\cdot \|z_j\|)$ is a cosine similarity.

\noindent\textbf{VGAE:} Variational Graph Autoencoder (VGAE) \citep{kipf2016vgae} is a framework designed for learning graph embeddings through variational inference. It is a probabilistic approach that leverages both graph structure and node features to infer latent node representations. VGAE aims to model the underlying distribution of the graph data, capturing the uncertainty in the embeddings by using a variational autoencoder architecture. This setup allows VGAE to generate robust embeddings that generalize well to unseen nodes or links. The model consists of an encoder that approximates the posterior distribution over latent variables and a decoder that reconstructs the graph from these variables.

The loss function comprises two components: a reconstruction loss that encourages the model to accurately predict the adjacency matrix, and a regularization term in the form of the KL-divergence, which ensures the latent variables follow the prior distribution. 

\begin{enumerate}
    \item Update the encoder by maximizing the evidence lower bound (ELBO): 
    $$\mathcal{L} = \mathbb{E}_{q(Z|X,A)}[\log p(A|Z)] - KL(q(Z|X,A)||p(Z))$$
    \item The prior over the latent variables $Z$ is typically set to a standard Gaussian: $p(Z) = \mathcal{N}(0, I)$.
\end{enumerate}

%% file: stable_rank_ver_corr_table.tex

\begin{table}[H]
\centering
\renewcommand{\arraystretch}{1.4} %
\resizebox{\textwidth}{!}{%
\begin{tabular}{@{}ccccccccc@{}}
\toprule
          & Ours             & \multicolumn{3}{c}{Literature}            & \multicolumn{4}{c}{\citet{tsitsulin2023unsupervised}}   \\ \midrule
 &
  CCA dist ($\downarrow$) &
  $\alpha$-ReQ ($\downarrow$) &
  pseudo-$\kappa$ ($\uparrow$) &
  RankMe ($\uparrow$) &
  NEsum ($\uparrow$) &
  SelfCluster ($\downarrow$) &
  Stable Rank ($\uparrow$) &
  Coherence ($\downarrow$) \\ \midrule
Cora      & \textbf{-0.6596} & -0.2414          & 0.2036  & 0.1998       & 0.2738          & 0.017            & 0.1146          & {\ul 0.2914}  \\
PubMed    & \textbf{-0.7702} & -0.2379          & 0.1422  & 0.2048       & {\ul 0.4854}    & -0.0103          & 0.0532          & 0.127         \\
Citeseer  & -0.2014          & \textbf{-0.3183} & 0.2842  & 0.2781       & 0.2518          & {\ul -0.3156}    & 0.0609          & -0.295        \\
CS        & 0.1875           & -0.5459          & 0.5356  & {\ul 0.5577} & \textbf{0.5608} & -0.4376          & 0.1899          & -0.3776       \\
Photo     & -0.3797          & -0.2886          & 0.274   & 0.3155       & 0.2698          & \textbf{-0.5009} & {\ul 0.4722}    & -0.3782       \\
Computers & -0.2433          & -0.1943          & 0.0777  & 0.2498       & {\ul 0.2875}    & -0.0714          & \textbf{0.3322} & -0.283        \\
Chicago   & \textbf{-0.1225} & 0.6618           & -0.6887 & -0.6667      & -0.3284         & 0.3701           & -0.451          & 0.1446        \\
Anaheim   & \textbf{-0.1528} & 0.352            & -0.3273 & -0.3091      & -0.1757         & 0.2337           & -0.0864         & {\ul -0.0543} \\
Twitch    & \textbf{-0.5858} & 0.6246           & -0.4868 & -0.5414      & -0.4852         & 0.3034           & -0.2921         & 0.2839        \\
Education & \textbf{-0.3464} & 0.4286           & -0.4315 & -0.3548      & -0.0065         & 0.0171           & 0.0917          & {\ul -0.1247} \\ \midrule
Avg\_clf  & {\ul -0.3445}    & -0.3044          & 0.2529  & 0.3010       & \textbf{0.3549} & -0.2198          & 0.2038          & -0.1526       \\
Avg\_reg  & \textbf{-0.3019} & 0.5168           & -0.4836 & -0.4680      & -0.2490         & 0.2311           & -0.1845         & 0.0624        \\ \bottomrule
\end{tabular}%
}
\caption{Spearman correlation between each metric and the node classification accuracy ($R^2$ if the predicted value is continuous) across different models and sets of hyperparameters. The higher the absolute value is, the better with the sign aligning with the arrow. We note that many other metrics lose their intended direction for some dataset. For example, the high \textit{StableRank} should indicate the better performance but the real relationship turns out to be reversed for Chicago, Aneheim and Twitch dataset. }
\label{tab:exp-summary-spearman-corr-full}
\end{table}